\documentclass{article}

\usepackage{microtype}
\usepackage{graphicx}
\usepackage{subfigure}
\usepackage{booktabs} % for professional tables
\usepackage{multirow}

\usepackage{hyperref}

\usepackage[accepted]{icml2023}

\usepackage{amsmath}
\usepackage{amssymb}
\usepackage{mathtools}
\usepackage{amsthm}

\usepackage[capitalize,noabbrev]{cleveref}

\theoremstyle{plain}
\newtheorem{theorem}{Theorem}[section]
\newtheorem{proposition}[theorem]{Proposition}
\newtheorem{lemma}[theorem]{Lemma}

\theoremstyle{definition}

\theoremstyle{remark}

\usepackage{amsfonts}       % blackboard math symbols
\usepackage{nicefrac}       % compact symbols for 1/2, etc.

\usepackage{amsmath,amssymb,amsthm,color,graphicx,mathrsfs,booktabs,bm,comment,bbm,wrapfig,blkarray}

\let\oldleq\leq % Store current version of \leq in \oldleq
\let\oldgeq\geq % Store current version of \geq in \oldgeq
\usepackage{mathabx}
\let\leq\oldleq % Restore old version of \leq from \oldleq
\let\geq\oldgeq % Restore old version of \geq from \oldg teq

\usepackage{caption}

\newcommand{\holders}{H\"{o}lder's }

\DeclareMathOperator*{\argmin}{arg\,min}
\DeclareMathOperator*{\Image}{Im}
\DeclareMathOperator*{\Kernel}{Ker}
\newcommand{\R}{\mathbb{R}}
\newcommand{\sgn}{{\rm sgn}}
\newcommand{\cut}{{\rm Cut}}

\newcommand{\subjectto}{\mathrm{s.t.}\ }

\newcommand{\bfa}{\mathbf{a}}
\newcommand{\bfb}{\mathbf{b}}
\newcommand{\bfx}{\mathbf{x}}
\newcommand{\bfi}{\mathbf{i}}
\newcommand{\bfr}{\mathbf{r}}
\newcommand{\bfy}{\mathbf{y}}
\newcommand{\bfv}{\mathbf{v}}
\newcommand{\bfe}{\mathbf{e}}
\newcommand{\bfw}{\mathbf{w}}
\newcommand{\bfu}{\mathbf{u}}
\newcommand{\bfz}{\mathbf{z}}
\newcommand{\bfalpha}{\bm{\alpha}}
\newcommand{\bfbeta}{\bm{\beta}}

\newcommand{\bfzeta}{\bm{\zeta}}
\newcommand{\bfxi}{\bm{\xi}}

\newcommand{\mcalV}{\mathcal{V}}
\newcommand{\mcalF}{\mathcal{F}}
\newcommand{\mcalS}{\mathcal{S}}

\begin{document}

\twocolumn[
\icmltitle{Multi-class Graph Clustering via
Approximated Effective \texorpdfstring{$p$}{p}-Resistance}




\begin{icmlauthorlist}
\icmlauthor{Shota Saito}{ucl}
\icmlauthor{Mark Herbster}{ucl}
\end{icmlauthorlist}

\icmlaffiliation{ucl}{Department of Computer Science, University College London, London, United Kingdom}
\icmlcorrespondingauthor{Shota Saito}{ssaito@cs.ucl.ac.uk}

\icmlkeywords{Machine Learning, $p$-Resistance, Graph $p$-Laplacian, Clustering}

\vskip 0.3in
]

\printAffiliationsAndNotice{}  % leave blank if no need to mention equal contribution

\begin{abstract}
This paper develops an approximation to the (effective) $p$-resistance and applies it to multi-class clustering. 
Spectral methods based on the graph Laplacian and its generalization to the graph $p$-Laplacian have been a backbone of non-euclidean clustering techniques.  
The advantage of the $p$-Laplacian is that the parameter $p$ induces a controllable bias on cluster structure.  
The drawback of $p$-Laplacian eigenvector based methods is that the third and higher eigenvectors are difficult to compute.  
Thus, instead, we are motivated to use the $p$-resistance induced by the $p$-Laplacian for clustering.  
For $p$-resistance, small $p$ biases towards clusters with high internal connectivity while large $p$ biases towards clusters of small ``extent,'' that is a preference for smaller shortest-path distances between vertices in the cluster. 
However, the $p$-resistance is expensive to compute.  
We overcome this by developing an approximation to the $p$-resistance. 
We prove upper and lower bounds on this approximation and observe that it is exact when the graph is a tree. 
We also provide theoretical justification for the use of $p$-resistance for clustering.
Finally, we provide experiments comparing our approximated $p$-resistance clustering to other $p$-Laplacian based methods.

\end{abstract}

\section{Introduction}
\label{sec:introduction}

Graphs are widely used data structures representing a pairwise relationship~\cite{normcut}.
In machine learning, various graph methods are considered, such as clustering and semi-supervised learning~\cite{tuto,zhu2003semi}.
Common to these methods, graph 2-seminorm, 2-seminorm induced from the graph Laplacian, is actively used. 
Its generalization to the graph $p$-seminorm is known to exhibit performance improvement~\cite{pgraph,slepcev2019analysis}.

This paper considers multi-class clustering over a graph using the graph $p$-seminorm.
For this purpose, spectral clustering is the most popular.  
In the 2-seminorm based (i.e., standard) spectral clustering, we use the first $k$ eigenvectors of the graph Laplacian for $k$-class clustering~\cite{tuto}.
This use of the first $k$ eigenvectors is theoretically supported~\cite{lee2014multiway}.
Using the $p$-seminorm, this graph Laplacian is extended to the graph $p$-Laplacian~\cite{pgraph}.
Similar to the standard case, using the first $k$ eigenvectors of this graph $p$-Laplacian for $k$-class clustering is also theoretically supported~\cite{tudisco2018nodal}.
However, there is not yet known an exact identification for the third or higher eigenpairs of $p$-Laplacian~\cite{lindqvist2008nonlinear}, and hence in practice, it is difficult to obtain them.
Due to this limitation, the existing methods using $p$-Laplacian propose an ad-hoc resolution of this limitation for multi-class clustering~\cite{pgraph,ding2019multiway,luo2010eigenvectors}.
On the other hand, this limitation makes the $p$-Laplacian difficult to use in practice to leverage the full potential of graph $p$-seminorm for multi-class clustering purposes.

Thus, in order to aim to exploit the graph $p$-seminorm more for multi-class clustering, we explore an alternative way to spectral clustering; in this paper, we propose multi-class clustering via approximated effective $p$-resistance. 
The $p$-resistance is also induced by the graph $p$-seminorm.
The use of $p$-resistance\footnote{In the following, we abbreviate effective $p$-resistance as $p$-resistance.
} for clustering is motivated in the following way.
Looking back to the 2-seminorm case, the 2-resistance is considered in the context of the graph analog to electric circuit~\cite{doyle1984random}. 
The 2-resistance is defined as an inverse of the constrained optimization problem using the graph 2-seminorm.
This 2-resistance is known to be a metric over a graph~\cite{klein1993resistance}. 
Moreover, 2-resistance is characterized by a semi-supervised learning problem of the graph 2-seminorm regularization~\cite{alamgir2011phase}.
Given these properties, the 2-resistance is used for the multi-class graph clustering~\cite{yen2005clustering,alev2017graph}.
However, in the large graph setting, the 2-resistance converges to a meaningless limit function~\cite{nadler2009semi,luxburg2010getting}.
Using the graph $p$-seminorm, the 2-resistance is generalized to the $p$-resistance~\cite{herbster2009predicting}, which overcomes this problem~\cite{slepcev2019analysis}.
The $1/(p-1)$-th power of the $p$-resistance is also shown to be a metric~\cite{herbster2010triangle,kalman2021flow}.
Furthermore, since different $p$ of $p$-resistance captures a different characteristic of a graph~\cite{alamgir2011phase}, we expect that the parameter $p$ serves as a tuning parameter for the clustering result.
Thus, the natural idea for the multi-class clustering is to use the $1/(p-1)$-th power of $p$-resistance. 

While the discussion above motivates us to use the $1/(p-1)$-th power of the $p$-resistance to multi-class clustering, there remain two issues; 
i) computational cost of $p$-resistances for many pairs
ii) lack of theoretical justification for using $p$-resistance for clustering other than the metric property. 
In this paper, we address these in the following way. 
For i), it is computationally expensive to compute $p$-resistances for many pairs. 
The reason is that we need to solve the constrained optimization problem for many pairs.
Looking back at the 2-resistance, we can compute the 2-resistance efficiently in the following way. Recall that we can compute 2-resistance as
\begin{align}
    r_{G,2}(i,j) = \|L^{+}\bfe_{i} - L^{+}\bfe_{j}\|_{G,2}^{2},
\end{align}
where $r_{G,p}(i,j)$ is $p$-resistance for a graph $G$, $i$ and $j$ are vertices, $L^{+}$ is a pseudoinverse of the graph Laplacian $L$ for $G$, $\bfe_{i}$ is the $i$-th coordinate vector of $\R^{n}$, and  $\|\cdot\|_{G,2}$ is a 2-seminorm induced from the graph Laplacian $L$.
By this representation, once we compute $L^{+}$, we can ``reuse'' $L^{+}$ to compute 2-resistance for different pairs. 
This reuse makes the computation of 2-resistances for many pairs faster than naively solving the optimization problem for each pair.
However, we do not know such representation for $p$-resistance. 
Thus, to obtain $p$-resistance for many pairs, we need to solve many constrained optimization problems.
The significant result of this work is that in Thm.~\ref{thm:presistance}, we give a theoretical guarantee for the approximation of $p$-resistance as  
\begin{align}
\label{eq:informalrgp}
    r_{G,p} (i,j) \approx \|L^{+}\bfe_{i} - L^{+}\bfe_{j}\|_{G,q}^{p},
\end{align}
where $q$ satisfies $1/p+1/q=1$, and $\|\cdot\|_{G,q}$ is a graph $q$-seminorm whose formal definition is given later.
We also show that for a tree, the approximation of Eq.~\eqref{eq:informalrgp} becomes exact (Thm.~\ref{thm:tree}).
By this approximation, we can compute the approximated $p$-resistance efficiently, similar to the $p$$=$$2$ case.
For ii),
we do not have a theoretical justification for using $p$-resistance for clustering other than the metric property.
While the $p$-resistance has the metric property, this property itself does not support the clustering quality.
For spectral clustering and 2-resistance, we have theoretical justifications for clustering.
For spectral clustering, using the first $k$ eigenvectors of the graph $p$-Laplacian is theoretically justified~\cite{lee2014multiway,tudisco2018nodal}. 
The 2-resistance has a theoretical connection to a semi-supervised learning problem of graph 2-seminorm regularization~\cite{alamgir2011phase}.
For $p$-resistance, we show that $p$-resistance is characterized by the semi-supervised learning problem of $p$-seminorm regularization.
This resolves the open problem stated in~\cite{alamgir2011phase}. 
This gives a theoretical foundation for using $p$-resistance for clustering from a view of the semi-supervised learning problem.
Addressing the two issues above, as a multi-class clustering algorithm, we propose to apply the $k$-medoids algorithm to the distance matrix obtained from the approximated $p$-resistance.
With these two results, our algorithm can be said to be more theoretically supported than existing multi-class spectral clusterings via graph $p$-Laplacian.
Our experiment demonstrates that our algorithm outperforms the existing multi-class clustering using graph $p$-Laplacian and 2-resistance-based methods. 

Our contributions are as follows: 
i) We give a guarantee for the approximated representation of $p$-resistance using the $q$-seminorm. 
ii) We show that the $p$-resistance characterizes the solution of semi-supervised learning of $p$-seminorm regularization of a graph. 
iii) We provide graph $p$-seminorm-based multi-class clustering. 
iv) We numerically show that our method outperforms the existing and standard methods. 
\textit{All proofs are in Appendix.}

\section{Preliminaries}
\label{sec:preliminaries}

We define a graph $G = (V,E)$, where $V$ is a set of \textit{vertices} and $E$ is a set of \textit{edges}. 
Throughout this paper, we use $n:= |V|$ and $m:=|E|$.
An edge connects two vertices, and we do not consider the direction of the edge (\textit{undirected}). 
A graph is \textit{connected} if there is a path for every pair of vertices. 
In the following we assume that the graph is connected.
We associate an $\ell$-th edge ($\ell$$=$$1,\ldots,m$) with a positive value $w_{\ell}$, which we refer to as a~\textit{weight}. 
We define a \textit{weight vector} $\bfw \in \R^{m}$, whose $\ell$-th element is $w_{\ell}$.
We define a \textit{weight matrix} as a diagonal matrix $W\in \R^{m\times m}$ whose $\ell$-th diagonal element is $w_{\ell}$.
We define an \textit{incidence matrix} $C \in \R^{m \times n}$, where $c_{\ell i} := 1$ and $c_{\ell j} := -1$ when $\ell$-th edge connects vertices $i$ to $j$ ($i>j$) otherwise 0.
We represent a graph by an \textit{adjacency matrix} $A \in \R^{n \times n}$; the $ij$-th element and $ji$-th element of $A$ are $w_{\ell}$ if $\ell$-th edge connects vertices $i$ to $j$ ($i>j$), i.e., $a_{ij} = a_{ji} := w_{\ell}$, and we define $a_{ij} = a_{ji} := 0$ if there is no edge between vertices $i$ and $j$.
By construction, the adjacency matrix $A$ is symmetric. 
A \textit{degree} $d_{i}$ for a vertex $i$ is defined as $d_{i} := \sum_{j} a_{ij}$.
We define a degree matrix $D$, a diagonal matrix whose diagonal elements are $D_{ii} := d_{i}$.
We define a matrix called \textit{graph Laplacian}, as $L:=D-A$.
The graph Laplacian can also be written as $L=C^{\top}WC$.
For more details, see~\cite{bapat2010graphs}.
A graph Laplacian induces a seminorm from a inner product $\langle \bfu,\bfx \rangle_{L}:= \bfu^{\top}L\bfx$.
We now define $\mathcal{V}(L)$, a set of $L^{+}\mathbf{e}_{i}$, a \textit{coordinate spanning set} of $L$, and we compute as
\begin{align}
\notag
    \mathcal{V}(L) &:= \{\mathbf{v}_{i} = L^{+}\mathbf{e}_{i}:i=1,\ldots,n\}, \\  
\label{eq:coordinate}
    \langle \bfu,\mathbf{v}_{i}\rangle_{L} &= \bfu^{\top}LL^{+}\mathbf{e}_{i} = u_{i}, \forall \mathbf{v}_{i} \in \mathcal{V}(L),\ \bfu \in \mathcal{H}(L),
\end{align}
where $\mathcal{H}(L) := \mathrm{span}(\mathcal{V}(L))$.
This is a reproducing kernel property for a kernel whose gram matrix is $L^{+}$~\cite{aronszajn1950theory,shawe2004kernel}. Thus, $\mathbf{v}_{i} = L^{+}\mathbf{e}_{i}$ works as coordinate for the space $\mathcal{H}(L)$.
Note that although we call $\mathcal{V}(L)$ as ``coordinate'', the vectors in $\mathcal{V}(L)$ are not necessarily orthonormal to each other.

An analog is established between graph and electric circuit~\cite{doyle1984random}. 
In this analog, the energy $S_{2}(\bfx)$ for a vector over vertices $\bfx$$\in$$ \R^{n}$ and effective 2-resistance $r_{G,2}(i,j)$ between two vertices $i,j$$\in$$V$ are defined as
\begin{align}
\label{eq:defenergyandresistance}
    S_{G,2}(\bfx) &:= \sum_{i,j \in V}a_{ij}(x_{i} - x_{j})^2 = \bfx^{\top} L \bfx, \\
    r_{G,2}(i,j) &:= (\min_{\bfx}\{S_{G,2}(\bfx)\ \mathrm{s.t.}\ x_{i} - x_{j} = 1\})^{-1}
\end{align}
Using the inner product $\langle \cdot,\cdot\rangle_{L}$ and its induced seminorm $\|$$\cdot$$\|_{L}$, we can rewrite 2-resistance as
\begin{align}
\label{eq:2resistancecharacteristics}
    r_{G,2}(i,j) &= \|L^{+}\mathbf{e}_{i} - L^{+}\mathbf{e}_{j}\|_{L}^{2} 
    = \|\bfv_{i} - \bfv_{j}\|_{L}^{2},
\end{align}
where $\bfv_{i}$$,\bfv_{j}$$\in$$\mathcal{V}(L)$.
From the definition, $r_{G,2}(i,i)$$=$$0$ and $r_{G,2}(i,j)$$=$$r_{G,2}(j,i)$.
These energy and effective resistance are extended to $p$-energy $S_{G,p}$ and $p$-resistance $r_{G,p}$ for $p>1$ as
\begin{align}
\label{eq:penergyandpresistance}
    S_{G,p}(\bfx) &:= \sum_{i,j \in V}a_{ij}|x_{i} - x_{j}|^{p}, \\
\label{eq:presistance}
    r_{G,p}(i,j) &:= (\min_{\bfx}\{ S_{G,p}(\bfx)\ \mathrm{s.t.}\ x_{i} - x_{j} = 1\})^{-1}.
\end{align}
This $p$-resistance shows the triangle inequality~\cite{herbster2010triangle}, that is for $a,b,c \in V$
\begin{align}
    \label{eq:herbstertriangle}
    r_{G,p}^{1/(p-1)}(a,b) \leq r_{G,p}^{1/(p-1)}(a,c) + r_{G,p}^{1/(p-1)}(c,b).
\end{align}
With $r_{G,p}^{1/(p-1)}$, the graph $G$ is a metric space.
Particularly, when $p$$=$$2$, 2-resistance defines a metric between $\mathbf{v}_{i}$$,\mathbf{v}_{j}$$\in$$\mathcal{V}(L)$.
More properties on $p$-energy and $p$-resistance, see~\cite{alamgir2011phase,herbster2009predicting}.

Lastly, we review several notions from linear algebra. 
We refer to~\cite{horn2012matrix} for the details.
First, we recall the weighted $p$-norm. 
Given positive weights $\bfr \in \R^{n_{1}}$ where $r_{i} > 0$, for a vector $\bfx \in \R^{n_{1}}$ we define the weighted $p$-norm $\|\bfx\|_{\bfr,p}$, and its inner product $\langle \bfx, \bfy \rangle_{\bfr}$ as
\begin{align}
\label{eq:weightedpnorm}
    \|\bfx\|_{\bfr,p} := \left(\sum_{i=1}^{n_{1}} r_{i} |x_{i}|^{p}\right)^{1/p}, \langle \bfx, \bfy\rangle_{\bfr} := \sum_{i=1}^{n_{1}} r_{i} x_{i} y_{i}.
\end{align}
For this weighted $p$-norm and inner product, we have \holders inequality as follows;
\begin{lemma}[\holders inequality]
\label{lemma:holders}
For $p,q>1$ such that $1/p + 1/q=1$, $\langle \bfx,\bfy \rangle_{\bfr} \leq \|\bfx\|_{\bfr,p}\|\bfy\|_{\bfr,q}.$
\end{lemma}
For a matrix $M$$\in$$\R^{n_{1} \times n_{2}}$, we define an \textit{image} of $M$ as $\Image(M)$$:=$$\{\bfy | \bfy$$=$$M\bfx, \bfx$$\in$$\R^{n_{2}}\}$$\subseteq$$\R^{n_{1}}$, that is a space spanned by the matrix $M$. 
Note that $MM^{+}\bfy$ is an orthogonal projection of $\bfy$ onto $\Image(M)$, where $M^{+}$ is a pseudoinverse of $M$.
We introduce a \textit{matrix operator $p$-norm} $\vvvert M \vvvert_{p}$ for a matrix $M$ as 
\begin{align}
    \vvvert M \vvvert_{p} := \sup_{\bfx \in \R^{n}}\|M\bfx\|_{p}/\|\bfx\|_{p}.
\end{align}

\section{Graph $p$-Seminorm and Approximating $p$-Resistance}
\label{sec:graphpnormand}

This section defines a graph $p$-seminorm, which is a foundation of our discussion. 
We then discuss several properties of the graph $p$-seminorm. 
Using these properties, we provide the approximation of $p$-resistance.

\subsection{Graph $p$-Seminorm}
\label{sec:graphpnorm}

In this section, we define a graph $p$-seminorm and discuss its characteristics.
For a vector over vertices $\bfx \in \R^{n}$, we define a graph $p$-seminorm over a graph using a weighted $p$-norm for a graph weight vector $\bfw \in \R^{m}$.
We define a graph $p$-seminorm $\|\bfx\|_{G,p}$ for $\bfx \in \R^{n}$ as
\begin{align}
\notag
    &\|\bfx\|_{G,p} := \|C\bfx\|_{\bfw,p} \\ 
    \label{eq:defgraphpnorm}
    &= \left(\sum_{i \in E} w_{i} |(C\bfx)_{i}|^{p}\right)^{1/p}= \left(\sum_{i,j \in V} a_{ij} |x_{i} - x_{j}|^{p}\right)^{1/p}.
\end{align}
From the definition of $p$-energy Eq.~\eqref{eq:penergyandpresistance}, $S_{G,p}(\bfx) = \|\bfx\|_{G,p}^{p}$.
Also, we immediately know that this norm is induced by the inner product $\langle C \bfx, C \bfy \rangle_{\bfw}$ from the definition of the graph $p$-seminorm. 
We now see that this graph seminorm can also be induced from the inner product $\langle \bfx, \bfy \rangle_{L}$, because
\begin{align}
    \langle C \bfx, C \bfy \rangle_{\bfw} = \bfx^{\top}C^{\top}W C\bfy = \bfx^{\top} L \bfy = \langle \bfx, \bfy \rangle_{L}.
\end{align}
From this observation, we see that $\|\bfx\|_{G,2} = \|\bfx\|_{L}$.
Also, we can restrict graph $p$-seminorm to a norm if we consider $\bfx \in \Image(L)$.
Note that this graph $p$-seminorm is same as the graph $p$-seminorm defined in~\cite{herbster2009predicting}.
For this graph $p$-seminorm, using Lemma~\ref{lemma:holders}, the \holders inequality holds;
\begin{align}
\label{eq:holders}
     \langle \bfx, \bfy \rangle_{L} \leq \|\bfx\|_{G,p} \|\bfy\|_{G,q}, 1/p+1/q=1.
\end{align}

When $p=2$ \holders inequality plays a fundamental role to show the representation of 2-resistance by Eq.~\eqref{eq:2resistancecharacteristics} in the following way.
Using the equality condition of the \holders inequality Eq.~\eqref{eq:holders} for $p=2$, we have a lemma.
\begin{lemma}[Classical, e.g.,~\citet{herbster2006prediction}]
\label{lemma:p=2holderseqcond}
For $\bfy \in \R^{n}$, $\|\bfy\|_{G,2}^{-2} = \min_{\bfx} \{ \|\bfx\|_{G,2}^{2}\  \mathrm{s.t.}\ \langle \bfx,\bfy \rangle_{L} = 1\}$.
\end{lemma}

This lemma is a classical result rewritten with our notation of graph $p$-seminorm.
By substituting $\bfy$$:=$$L^{+}\bfe_{i} - L^{+}\bfe_{j}$, the right hand side of Lemma~\ref{lemma:p=2holderseqcond} becomes the inverse of 2-resistance (see Appendix~\ref{sec:detailsoflemmap=2}).
Thus, we obtain $r_{G,2}(i,j)$$=$$\|L^{+}\bfe_{i} - L^{+}\bfe_{j} \|_{L}^{2}$$ =$$\|L^{+}\bfe_{i} - L^{+}\bfe_{j} \|_{G,2}^{2}$.
For $p$-resistance, the question is how is the coordinate spanning set $\mcalV(L^{+})$ related to the $p$-resistance? 
Can we derive such relation using \holders inequality Eq.~\eqref{eq:holders}, similarly to the $p$$=$$2$ case?
Next section will show such connection between $p$-resistance and the coordinate spanning set using Eq.~\eqref{eq:holders}.% 

\subsection{Approximating $p$-Resistance via Coordinate Spanning Set}
\label{sec:approximating}

This section discusses approximation of $p$-resistance via the coordinate spanning set $\mcalV(L^{+})$. 
Looking at the $p$$=$$2$ case, we see that $L^{+}\bfe_{i}$$\in$$\mcalV(L^{+})$ can be regarded as coordinate, and $r_{G,2}(i,j)$$=$$\|L^{+}\bfe_{i} - L^{+}\bfe_{j}\|_{G,2}^{2}$.
This expression aids us to compute all the pairs of 2-resistance much faster than naively obtaining 2-resistance. 
For $p$-resistance, a natural question to ask is that does there exist some norm $\|\cdot\|^{\ddagger}$ such that $r_{G,p}(i,j) = \|L^{+}\bfe_{i} - L^{+}\bfe_{j}\|^{\ddagger} $?
If not, how can we approximate as $r_{G,p}(i,j) \approx \|L^{+}\bfe_{i} - L^{+}\bfe_{j}\|^{\ddagger}$?
If we can write $p$-resistance by such expression, we expect to obtain all the pairs of approximated $p$-resistance much faster than naively computing all the pairs of $p$-resistance.
This section addresses this problem.

As we see in Sec.~\ref{sec:graphpnorm}, Lemma~\ref{lemma:p=2holderseqcond} is a key to show that $r_{G,2}(i,j) = \|L^{+}\bfe_{i} - L^{+}\bfe_{j}\|_{G,2}^{2}$.
In the following, we now extend Lemma~\ref{lemma:p=2holderseqcond} from the case of $p=2$ to the general $p$.

\begin{proposition}
\label{prop:newkroverc}
For a graph $G$ and $p,q>1$ such that $1/p + 1/q = 1$, we have
\begin{align}
\label{eq:holdersp}
     \|\bfy\|_{G,q}^{-p} \leq \min_{\bfx}\{ \|\bfx\|_{G,p}^{p} \ \mathrm{s.t.} \ \langle \bfy, \bfx \rangle_{L} = 1\} \leq \|\bfz\|_{G,p}^{p}
\end{align}
where
\begin{align}
    \bfz := C^{+} \frac{f_{q/p}(C\bfy)}{\|\bfy\|_{G,q}^{q}}, \quad (f_{\theta}(\bfx))_{i} := \sgn (x_{i}) |x_{i}|^{\theta}.
\end{align}
When $f_{q/p} (C \bfy) \in \Image(C)$, we have
\begin{align}
     \|\bfy\|_{G,q}^{-p}= \min_{\bfx}\{ \|\bfx\|_{G,p}^{p} \ \mathrm{s.t.} \ \langle \bfy, \bfx \rangle_{L} = 1\} = \|\bfz\|_{G,p}^{p}
\end{align}
\end{proposition}
We first note that the minimization problem of Eq.~\eqref{eq:holdersp} is the inverse of $p$-resistance Eq.~\eqref{eq:penergyandpresistance}.
The left hand side of inequality Eq.~\eqref{eq:holdersp} immediately follows from \holders inequality (Eq.~\eqref{eq:holders}) with $\langle \bfy, \bfx \rangle_{L} = 1$.
We now turn our attention to the right hand side.
Recall that when $p=2$ we always have $f_{q/p}(C\bfx) \in \Image(C)$ and $\|\bfy\|_{G,2}^{-1} = \|\bfz\|_{G,2}$, which matches  Lemma~\ref{lemma:p=2holderseqcond}.
In the general $p$ case, $f_{q/p}(C\bfx) \notin \Image(C)$ and $\|\bfy\|_{G,q}^{-1}$$\neq$$\|\bfz\|_{G,p}$.
Thus, neither $\|\bfy\|_{G,q}^{-p}$ nor $\|\bfz\|_{G,p}^{p}$ gives the solution to the minimization problem.
However, this theorem tells us that we can upper bound the solution to the minimization problem by $\|\bfz\|_{G,p}^{p}$.

Applying Prop.~\ref{prop:newkroverc}, we obtain the bound for $p$-resistance as follows;
\begin{theorem}
\label{thm:presistance}
For a graph $G$ and $p, q > 1$ such that $1/p + 1/q = 1$, the $p$-resistance can be bounded as
\begin{align}
\notag
    &\frac{1}{\alpha_{G,p}^p}\|L^{+}\mathbf{e}_{i} - L^{+}\mathbf{e}_{j}\|_{G,q}^{p} \leq r_{G,p}(i,j) \leq \|L^{+}\mathbf{e}_{i} - L^{+}\mathbf{e}_{j}\|_{G,q}^{p}, \\
\notag
    &\mathrm{\ \ \ \ where\ }\alpha_{G,p} := \vvvert W^{1/p}CC^{+}W^{-1/p} \vvvert_{p},
\end{align}
\end{theorem}

\begin{theorem}
\label{thm:tree}
For a tree $G$ and $p, q > 1$ such that $1/p + 1/q = 1$, the $p$-resistance can be written as
\begin{align}
r_{G,p}(i,j) = \|L^{+}\mathbf{e}_{i} - L^{+}\mathbf{e}_{j}\|_{G,q}^{p}
\end{align}
\end{theorem}

Thm.~\ref{thm:presistance} and Thm.~\ref{thm:tree} show the relationship between $p$-resistance and $\|L^{+}\bfe_{i} - L^{+}\bfe_{j}\|_{G,q}^{p}$.
For general graphs, we do not obtain the exact representation of $p$-resistance. 
However, Thm.~\ref{thm:presistance} guarantees the quality of approximation as 
\begin{align}
\label{eq:approxrgp}
    r_{G,p}(i,j)
    \approx \|L^{+}\bfe_{i} - L^{+}\bfe_{j}\|_{G,q}^{p}
    = \| \bfv_{i} - \bfv_{j}\|_{G,q}^{p} ,
\end{align}
where $\bfv_{i}$,$\bfv_{j}$$\in$$\mathcal{V}(L)$. 
By this approximation, we obtain the similar representation of $p$-resistance to the $p=2$ case Eq.~\eqref{eq:2resistancecharacteristics}. 
The term $\alpha_{G,p}$ is a $p$-norm of the orthogonal projector to $\Image(W^{1/p}C)$. 
Note that we always have $\alpha_{G,p}$$\geq$$1$.
For a tree graph, Thm.~\ref{thm:tree} shows that $\|L^{+}\bfe_{i} - L^{+}\bfe_{j}\|_{G,q}^{p}$ becomes the exact representation of $p$-resistance.

The next question is what is $\alpha_{G,p}$. 
We bound $\alpha_{G,p}$ as follows;
\begin{proposition}
\label{prop:generalbound}
For a general graph $G$ and $p>1$, we have $\alpha_{G,p} \leq m^{|1/2-1/p|}$.
\end{proposition}
This proposition gives the guarantee for the approximation in Thm.~\ref{thm:presistance}.
Although Prop.~\ref{prop:generalbound} gives the quality guarantee, 
we expect this upper bound to be loose, i.e., we expect that the actual approximation value is closer to the exact value than this bound.
The reason why we expect in this way is that to prove the bound we only use the general technique that holds for any matrix and we do not use any graph structural information.
In the real dataset, we observe that the approximation of $p$-resistance and $\alpha_{G,p}$ is far better than this guarantee, see Appendix~\ref{sec:comparsionofapproximated}.
We give more discussion on this $\alpha_{G,p}$ in Appendix~\ref{sec:morediscussionalpha}.

Finally, we discuss computational times of the $p$-resistance. 
To compute Eq.~\eqref{eq:approxrgp}, it takes $O(m)$, given $L^{+}$. 
Also, in general it takes $O(n^{3})$ to compute $L^{+}$.
Note that we can reuse $L^{+}$ to compute $p$-resistance for different pairs.
We now consider to obtain the $p$-resistance by naively solving the optimization problem.
We can rewrite the constrained problem Eq.~\eqref{eq:presistance} as unconstrained problem, which is solvable by gradient descent.
In each step of the gradient descent, we compute $\nabla_{\bfx} \|\bfx\|_{G,p}^{p}$, which takes almost same time as Eq.~\eqref{eq:approxrgp}. 
Moreover, we cannot reuse the result of a single pair to compute for other pairs, while we can reuse $L^{+}$. 
Thus, to compute $p$-resistance for a single pair, our approximation is expected to be faster than naively solving the optimization problem. 
Moreover, if we compute for many pairs, our approximation is much faster by reusing $L^{+}$.

%\vspace{-0.1in}
\section{Clustering via $p$-Resistance}

This section considers using the $p$-resistance for the clustering algorithm.
Firstly, we propose a clustering algorithm using the approximated $p$-resistance. 
We next characterize our clustering algorithm from the semi-supervised problem point of view.
From this characterization, we can see that our clustering algorithm inherits properties from semi-supervised learning.

\subsection{Proposed Clustering Algorithm via $p$-Resistance}
\label{sec:proposed}

This section proposes an algorithm using $p$-resistance. 
The triangle inequality Eq.~\eqref{eq:herbstertriangle} gives a metric property to  $r_{G,p}^{1/(p-1)}(i,j)$. 
We call this $1/(p-1)$-th power of $p$-resistance as \textit{$p$-resistance metric}.
This metric property motivates us to use $p$-resistance for clustering algorithms.

Furthermore, the parameter $p$ serves as a tuning parameter of the clustering result. 
The general $p$ of $p$-resistance captures the graph structure somewhere between the cut and shortest path.
Using this characteristic, we expect varying $p$ tunes the clustering result somewhere suitable between cut-based and path-based.
When $p$ is small, the clustering result biases towards clusters with high internal connectivity, like a min-cut.
When $p$ is large, the clustering result focus more on path-based topology, that is a preference for smaller shortest-path distances between vertices in the cluster.
We illustrate this with examples of the two-class clustering in Fig.~\ref{fig:example}. 
In these examples, we conduct clustering with $k$-center algorithm using $p$-resistance. The left example is intuitively ``symmetric''; for this kind, $p\to\infty$, which looks at the path-based topology, gives more natural result.
The more natural clustering of the right example is ``cut''; for this kind, $p\to 1$, where we focus on the graph cut, gives the more natural result.
More details are in Appendix~\ref{sec:illustrative}.

While the discussion above motivates us to use the $p$-resistance metric for clustering, computing the $p$-resistance metric for all pairs is costly.
Thus, we approximate this metric by Thm.~\ref{thm:presistance}, and we obtain
\begin{align}
\notag
    r_{G,p}^{1/p-1}(i,j) &\approx \|L^{+}\bfe_{i} - L^{+}\bfe_{j}\|_{G,q}^{p/(p-1)}\\
\label{eq:approximation1/p-1}
    &= \|L^{+}\bfe_{i} - L^{+}\bfe_{j}\|_{G,q}^{q}=\|\bfv_{i}-\bfv_{j}\|_{G,q}^{q},
\end{align}
where $\bfv_{i},\bfv_{j}\in\mathcal{V}(L)$.
We then apply $k$-medoids to the distance matrix obtained by Eq.~\eqref{eq:approximation1/p-1}.
The overall proposed algorithm is summarized in Alg.~\ref{algo:clusteringkmedoids}

We discuss the choice of $k$-medoids over the other distance based method, such as $k$-means~\cite{bishop2006pattern} and $k$-center~\cite{gonzalez1985clustering}.
Although the main emphasis of our algorithm does not comes from the choice of $k$-median but from the approximation of $p$-resistance metric, $k$-median has some advantages.
Since $p$-resistance metric cannot define a distance between other than the data points defined as $\mcalV(L^{+})$, we cannot define distance for the some ``mean'', which is outside of the data points.
Therefore, the mean-based method such as $k$-means is not appropriate for this setting.
Instead, $k$-medoids is similar to the $k$-means~\cite{kaufman1990partitioning} but more appropriate since $k$-medoids assigns the centers to the actual data points.
The other potential choice is $k$-center algorithm. 
The $k$-center algorithm also assigns the center to the actual data point, and is known to be faster than $k$-medoids.
Also, $k$-center algorithm is approximated by the fast greedy farthest first algorithm~\cite{gonzalez1985clustering,herbster2010triangle}.
However, the $k$-medoids is more robust to the outliers than $k$-center. 
Thus, we propose to use $k$-medoids.
The overall computational time for Alg.~\ref{algo:clusteringkmedoids} is dominated by the computation of the all the pairs of the approximated $p$-resistance, $O(mn^{2})$. 
If we use the farthest first algorithm instead of $k$-medoids the algorithm is dominated either by the computation of $L^{+}$, $O(n^3)$, or farthest first $O(kmn)$.
Thus, farthest first is faster since in general $m \gg n$ but less robust than $k$-medoids.

\begin{algorithm}[!t]
\begin{algorithmic}[1]
\REQUIRE{Graph $G=(V,E)$ and $p$}
\STATE Compute pseudoinverse of the graph Laplacian $L^{+}$.
\STATE Compute all the pairs of the $p$-resistance metrics $r_{G,p}^{1/(p-1)}$ using Eq.~\eqref{eq:approximation1/p-1} and obtain a distance matrix.
\STATE Apply $k$-medoids to the distance matrix.
\ENSURE{The clustering result.}
\end{algorithmic}
 \caption{{Clustering Algorithm via $p$-Resistance}}
 \label{algo:clusteringkmedoids}
\end{algorithm}

\subsection{Connection between Semi-supervised Learning and $p$-Resistance}

\begin{figure*}[t]
\begin{center}
\small
\includegraphics[width=0.75\hsize,clip]{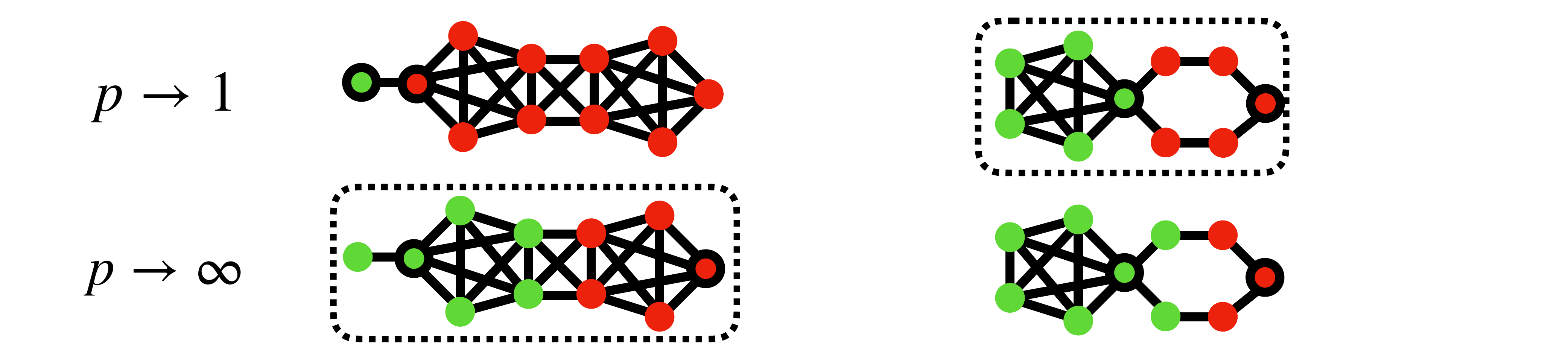}
\caption{
The illustrative examples where $p$ changes the clustering results. 
These examples conduct clustering with $k$-center algorithm using $p$-resistance as a metric.
The red and green colors show the clustering result.
Also, the vertices with borders show the obtained centers.
The dotted boxes exhibit natural clustering results.
These examples show varying $p$ tunes the clustering result; the left example gives a more natural clustering result when $p\to\infty$ whereas for right $p\to 1$ gives more natural result.
Details are in Sec.~\ref{sec:proposed} and Appendix~\ref{sec:illustrative}.
}
\label{fig:example}
\end{center}
\end{figure*}
This section explores connection from the $p$-resistance to the semi-supervised learning (SSL) via graph $p$-seminorm.
As we saw in Sec.~\ref{sec:preliminaries},~\citet{herbster2010triangle} shows the metric property of $p$-resistance.
While the metric property itself can motivates us to use $p$-resistance for our clustering, we do not know how much $p$-resistance shows connectivity of a graph.
This section shows that $p$-resistance can be seen from as an SSL perspective.
This connection assures us to use $p$-resistance for the clustering problem.
In the following we explain the connection by taking the following steps; i) SSL problem in the clustering context ii) the connection between the SSL and $p$-resistance.

We first consider an SSL problem for two known labels as
\begin{align}
\notag
    &\min_{\bfx} \{S_{G,p} (\bfx) \text{\ s.t.\ } x_{i} - x_{j} = 1\}\\ 
    \label{eq:sslp}
    &= \min_{\bfx}\{ S_{G,p} (\bfx) \text{\ s.t.\ } x_{i}=1, x_{j} = 0\}.
\end{align}
The equality holds since $S_{G,p}(\bfx) $$=$$ S_{G,p}(\bfx$$+$$c\mathbf{1})$, $ \forall$$c$$\in$$\R$.
We first note that Eq.~\eqref{eq:sslp} is an inverse of the $p$-resistance and we use the optimal value of this problem to $p$-resistance.
This learning problem for $p=2$ case has been considered in many literature, such as~\cite{zhu2003semi}, and extended to the $p$-seminorm setting~\cite{herbster2009predicting,alamgir2011phase,slepcev2019analysis}.

We now put Eq.~\eqref{eq:sslp} into clustering context; the solution of Eq.~\eqref{eq:sslp} tells us the graph structural information on clustering.
We recognize that Eq.~\eqref{eq:sslp} is two fixed-label problem.
Let $\bfx^{\mathrm{*}_{ij}}$ be a solution of the problem Eq.~\eqref{eq:sslp}. 
It is straightforward to interpret Eq.~\eqref{eq:sslp} if $i$ and $j$ is in different binary classes; we see which clusters the third point $\ell$ belongs to, the cluster which $i$ or $j$ is in.
More specifically, by comparing $x^{\mathrm{*}_{ij}}_{\ell} - x^{\mathrm{*}_{ij}}_{j}$ and $x^{\mathrm{*}_{ij}}_{i} - x^{\mathrm{*}_{ij}}_{\ell}$ we know which cluster the third point $\ell$ belongs to.
If we take the a pair of vertices $(i,j)$ arbitrarily, the assumption that ``$i$ and $j$ in different binary classes'' is not always appropriate. 
In this case, rather than assuming $i$ and $j$ in different binary classes, it is more natural to interpret in the following way; the two-pole binary SSL problem tells us that which of $i$ and $j$ the third point $\ell$ is close to in a graph. 
From this observation, if we look at $\bfx^{\mathrm{*}_{ij}}$ for all pairs, 
we know ``graph structural information'' from the SSL point of view.

We next show the connection between $p$-resistance and the solution of Eq.~\eqref{eq:sslp}, $\bfx^{\mathrm{*}_{ij}}$.
\begin{theorem}
\label{thm:alamgirp}
Let $\bfx^{\mathrm{*}_{ij}}$ be the solution of the problem Eq.~\eqref{eq:sslp}, and $\ell \in V$ be the third unlabeled point. Then we have
\begin{align}
\notag
    x^{\mathrm{*}_{ij}}_{\ell} - x^{\mathrm{*}_{ij}}_{j} \geq x^{\mathrm{*}_{ij}}_{i} - x^{\mathrm{*}_{ij}}_{\ell} \iff r_{G,p}(j,\ell) \geq r_{G,p}(\ell,i).
\end{align}
\end{theorem}

First note that~\citet{alamgir2011phase} proved Thm.~\ref{thm:alamgirp} only for the $p=2$ case in a different context than clustering (See Appendix~\ref{sec:original}), and posed the case of general $p$ as an open problem.
We resolve this open problem.

Thm.~\ref{thm:alamgirp} means that the $p$-resistance has a good property inherited from the SSL problem Eq.~\eqref{eq:sslp} in a following sense.
Thm.~\ref{thm:alamgirp} tells us that the ``graph structural information'', which can be obtained by comparing $x^{\mathrm{*}_{ij}}_{\ell} -x^{\mathrm{*}_{ij}}_{i}$ and $x^{\mathrm{*}_{ij}}_{j}-x^{\mathrm{*}_{ij}}_{\ell}$, is equivalent to comparing $p$-resistances $r_{G,p}(i, \ell)$ and $r_{G,p}(\ell,j)$.
Henceforth, Thm.~\ref{thm:alamgirp} further translates the intuition about $\bfx^{\mathrm{*}_{ij}}$ into $p$-resistance.

Thus, combining the two observations above, looking at the distance matrix computed from $p$-resistances can be interpreted as follows. Each distance shows how close the pair is in terms of two-pole binary SSL problem. Doing clustering with this distance matrix assigns a cluster by looking at all the graph structural information of two-pole binary SSL problems, which tells us that ``which the third point $\ell$ is close to, $i$ or $j$?''

From the observations above, we see that Thm.~\ref{thm:alamgirp} motivates us to use $p$-resistance metrics for multi-class problem.
Without Thm.~\ref{thm:alamgirp}, our algorithm is somewhat naive; even though $p$-resistance has a metric property,  we do not know how much $p$-resistance contains the structural information.

\section{Related Work}
\label{sec:related}

This section reviews the related work to the clustering via graph $p$-seminorm.
Since our work uses graph $p$-seminorm for the clustering purpose, spectral clustering using graph $p$-Laplacian is relevant. 
The graph $p$-Laplacian is induced from graph $p$-seminorm and used for the clustering purpose~\cite{pgraph}.
\citet{tudisco2018nodal} showed a theoretical guarantee for the use of the first $k$ variational eigenvectors (i.e., eigenvectors obtained by variational principle) of $p$-Laplacian for $k$-class clustering.
While we know the exact identification for the second eigenvectors of $p$-Laplacian, we do not know how to obtain the third or higher eigenvectors~\cite{lindqvist2008nonlinear}.
Thus, it is practically difficult to use spectra of $p$-Laplacian for multi-class clustering.
To bypass this limitation,~\citet{pgraph} applied two-class clustering method to multi-class by recursively bisectioning a subgraph into two subgraphs.
However, even in the $p$$=$$2$ case, recursive bisectioning produces suboptimal results~\cite{simon1997good}. 
The earlier works~\cite{luo2010eigenvectors,ding2019multiway,pasadakis2022multiway} used approximated orthogonality between eigenvectors of $p$-Laplacian for multi-class clustering.
However, we do not have theoretical supports that this approximated $k$ eigenvectors are the approximation of \textit{the first} $k$ variational eigenvectors.
Thus, we need to say that these methods rely on the ``ad-hoc bypasses'' and do not fully exploit the graph $p$-seminorm. 
For more details, see Appendix~\ref{sec:onlimitation}.

Another relevant approach is resistance-based clustering. 
In~\cite{yen2005clustering}, $k$-medoids algorithm is applied to the square of 2-resistance.
For clustering purpose, similar distances to the 2-resistance is proposed~\cite{fouss2007random,nguyen2016new,yen2008family}
The most relevant approach in this category is the $k$-center algorithm for the ``distance'' matrix obtained from the \textit{exact} $p$-resistance in~\cite{herbster2010triangle}. 
~\citet{herbster2010triangle} did not numerically verify the algorithm.
Our work uses $p$-resistance metric instead of $p$-resistance since without the $1/(p-1)$-th power operation $p$-resistance does not satisfy the metric property (Eq.~\eqref{eq:herbstertriangle}).
However, if we use $k$-center algorithm to the \textit{exact} $p$-resistance metric, we obtain the same result as~\cite{herbster2010triangle}.
The reason is that the $k$-center algorithm only matters the order of the distance, and the $1/(p-1)$-th power operation does not change the order of the $p$-resistance.
On the other hand, we emphasize that the most significant difference between our work and~\cite{herbster2010triangle} is that while we use the \textit{approximated} $p$-resistance~\cite{herbster2010triangle} uses exact $p$-resistance.
We also mention that the work~\cite{nguyen2016new} proposed a distance from the $p$-seminorm flow point of view. 
However, this distance does not have characterization from the learning problem (Thm.~\ref{thm:alamgirp}). 

Also, the graph $p$-seminorm is actively used in semi-supervised learning (SSL). 
The SSL problem using graph $p$-seminorm is relevant to $p$-resistance since the $p$-resistance can be seen as SSL for two known labels.
Earlier, the SSL using graph 2-seminorm is considered~\cite{zhou2003learning,zhu2003semi,calder2020poisson}.
The SSL via graph 2-seminorm and effective resistance is known to be ``ill-posed'' when the size of the unlabeled data points is asymptotically large~\cite{nadler2009semi}.
To overcome this problem, graph $p$-seminorm based SSL and $p$-resistance are considered~\cite{alamgir2011phase,bridle2013p,el2016asymptotic,slepcev2019analysis}, where the $p$-resistance is shown to be meaningful when $p$ is large.
Finally, the graph $p$-seminorm is widely used in the machine learning community, such as online learning~\cite{herbster2009predicting} and the local graph clustering task, where we find a cluster which the given vertices belong to~\cite{veldt2019flow,fountoulakis2020p,liu2020strongly}.

\section{Preliminary Experiments}
\label{sec:experiment}

\begin{table*}[!t]
\centering
\small
\caption{
Experimental Results. 
The ``type'' shows the type of methods; (ER) for effective resistance based methods
and (SC) for spectral clustering methods.
The ``Hop'' stands for Hopkins 155 dataset.
In method of ER, ``(a)'' shows that the method uses the approximation by (Eq.~\eqref{eq:approximation1/p-1}) and ``(ex)'' computes the exact $p$-resistance by gradient descent.
Also, ``$k$-med'' is $k$-medoids, and ``FF'' is the farthest first.  
Thus, the method ``$k$-med (a) $p$'' is our proposed algorithm, and ``FF (ex) $p$'' and ``FF $p=2$'' is a method proposed by~\cite{herbster2010triangle}.
The ``$p$-Flow'' is~\cite{nguyen2016new}, ``ECT'' is~\cite{yen2005clustering}, ``Rec-bi $p$'' is ~\cite{pgraph}, and ``$p$-orth'' is~\cite{luo2010eigenvectors}.
Since ``Rec-bi $p$'' is a deterministic method, we only report error. 
Also, since Hop contains multiple datasets, we only show the average. 
Due to the significant computational time, we were unable to finish some of the experiments, which are shown as ``--''. 
}
\label{tab:res}
\begin{tabular}{c|c|cc|ccc}
\toprule
        &                                         & \multicolumn{2}{c|}{2 clsss}                 & \multicolumn{3}{c}{multi-class}                                          \\
Type    & Method                                  & ionosphere                  & Hop 2~cls      & iris                       & wine                       & Hop 3~cls      \\
\midrule
ER & $k$-med (a) $p$                         & \textbf{0.196 $\pm$ 0.000} & \textbf{0.056} & 0.078 $\pm$ 0.013 & \textbf{0.287 $\pm$ 0.000} & \textbf{0.144} \\
ER & $k$-med (ex) $p$                        & --                         & --             & \textbf{0.075 $\pm$ 0.000} &        0.427 $\pm$ 0.000                 & --             \\
ER & $k$-med $p=2$                           & 0.305 $\pm$ 0.000          & 0.236          & 0.331 $\pm$ 0.000          & 0.534 $\pm$ 0.000           & 0.306          \\
ER & FF (a) $p$                              & 0.330 $\pm$ 0.023          & 0.109          & 0.108 $\pm$ 0.045          & 0.339 $\pm$ 0.054          & 0.313          \\
ER & FF (ex) $p$~\cite{herbster2010triangle} & 0.344 $\pm$ 0.020          & --             & 0.109 $\pm$ 0.019          &   0.524 $\pm$ 0.046           & --             \\
ER & FF $p=2$~\cite{herbster2010triangle}    & 0.355 $\pm$ 0.035          & 0.274          & 0.320 $\pm$ 0.000          & 0.530 $\pm$ 0.000          & 0.357          \\
ER & $p$-Flow~\cite{nguyen2016new} &   0.291 $\pm$ 0.000 &    0.231       & 0.247 $\pm$ 0.000          & 0.543 $\pm$ 0.043   &  0.243    \\
ER & ECT~\cite{yen2005clustering}            & 0.376 $\pm$  0.000        & 0.155          & 0.247 $\pm$ 0.000          & 0.534 $\pm$ 0.000          & 0.310          \\
SC      & Rec-bi $p$~\cite{pgraph}                & 0.225                      & 0.200          & 0.089                       & 0.354                      & 0.237          \\
SC      & SC $p$-orth~\cite{luo2010eigenvectors}  & 0.215 $\pm$ 0.123          & 0.237          & 0.087  $\pm$ 0.089          & 0.327 $\pm$ 0.116          & 0.221          \\
SC      & SC $p=2$                                & 0.308 $\pm$ 0.000          & 0.216          & 0.093  $\pm$ 0.000          & 0.438 $\pm$ 0.000          & 0.251         
\end{tabular}
\end{table*}

\begin{figure*}[!t]
\begin{center}
\subfigure[\small{Ion}]{%
\includegraphics[width=.25\hsize,clip]{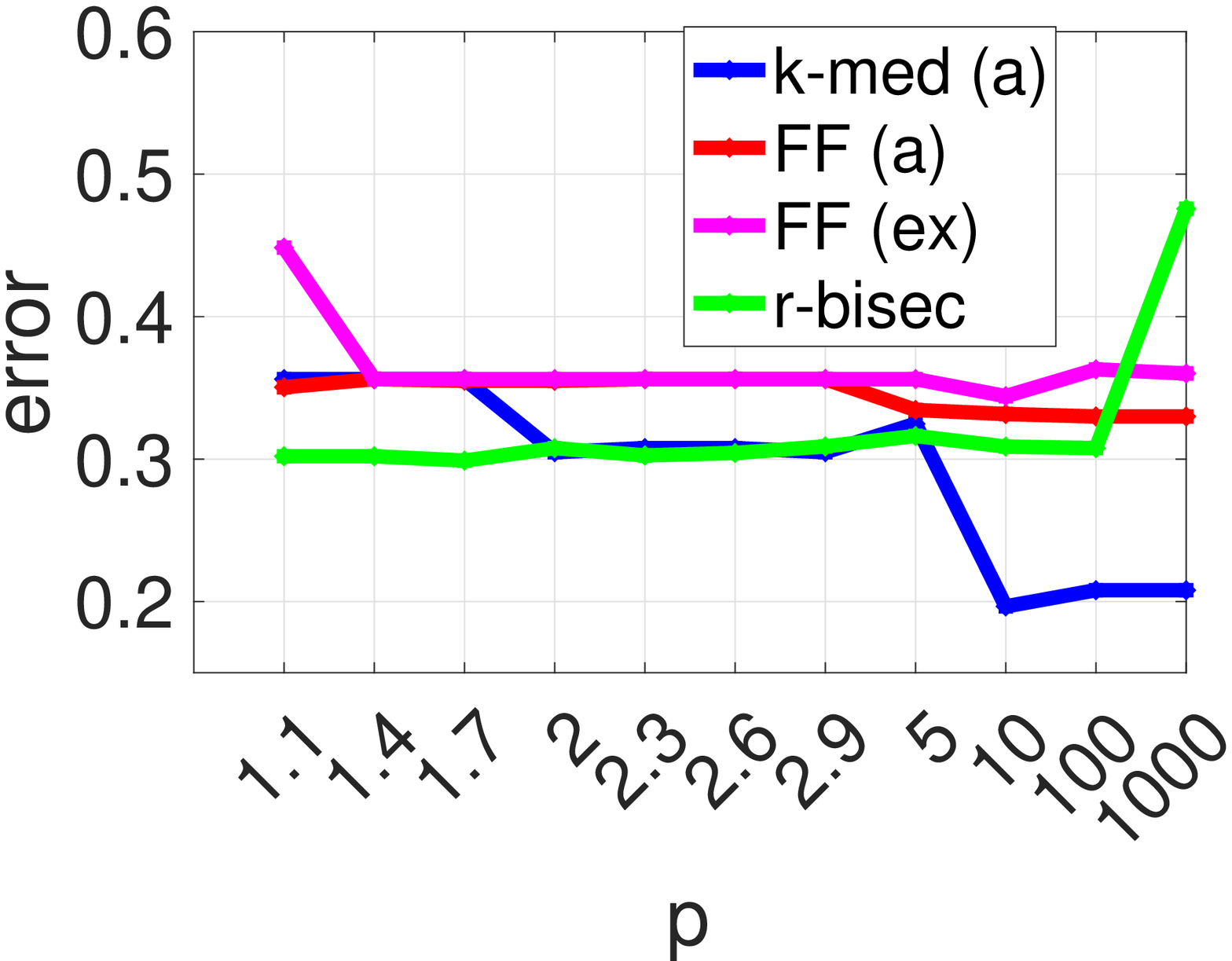}\label{fig:ion}}
\subfigure[\small{Wine}]{%
\includegraphics[width=.25\hsize,clip]{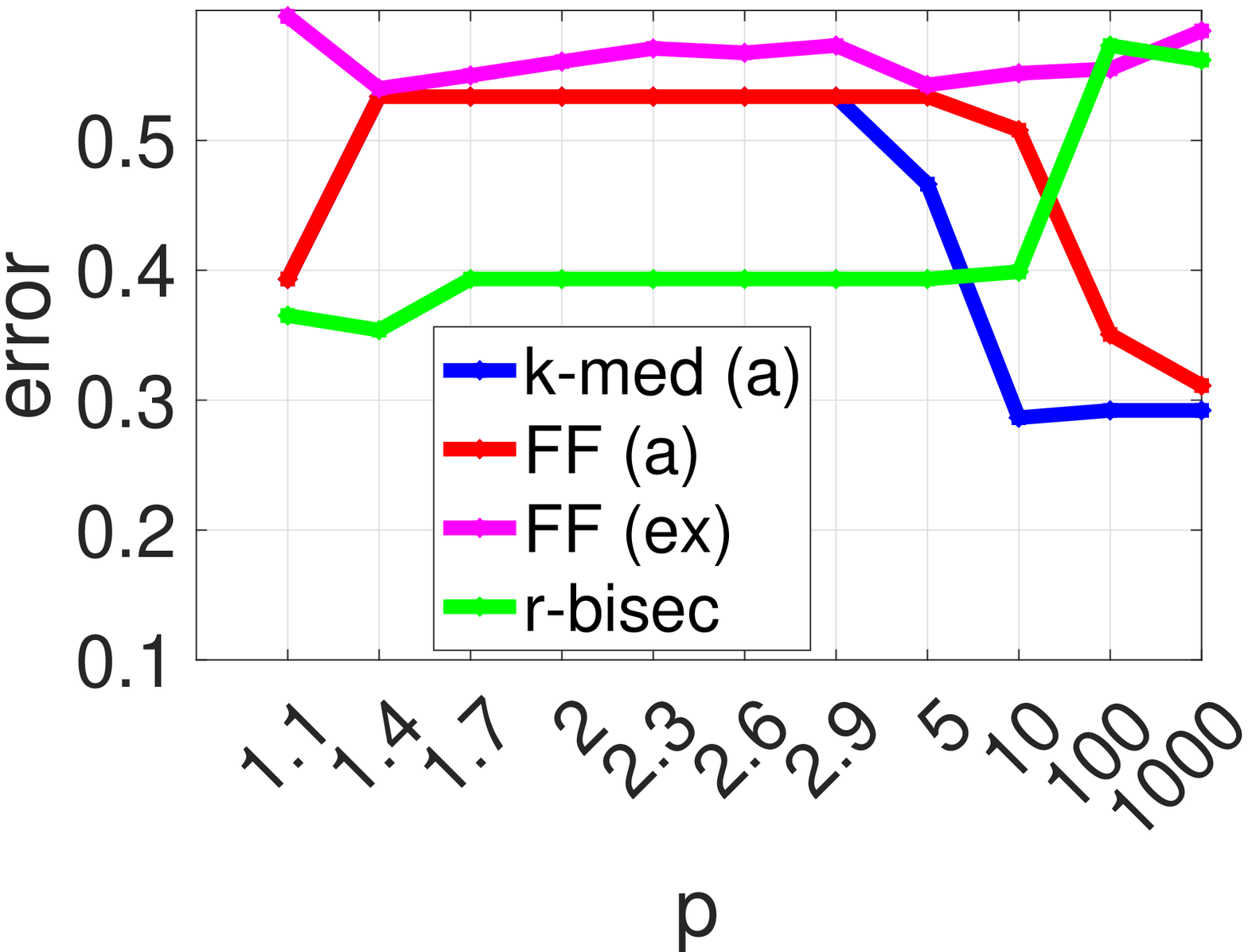}\label{fig:wine}}
\subfigure[\small{Iris}]{%
\includegraphics[width=.25\hsize,clip]{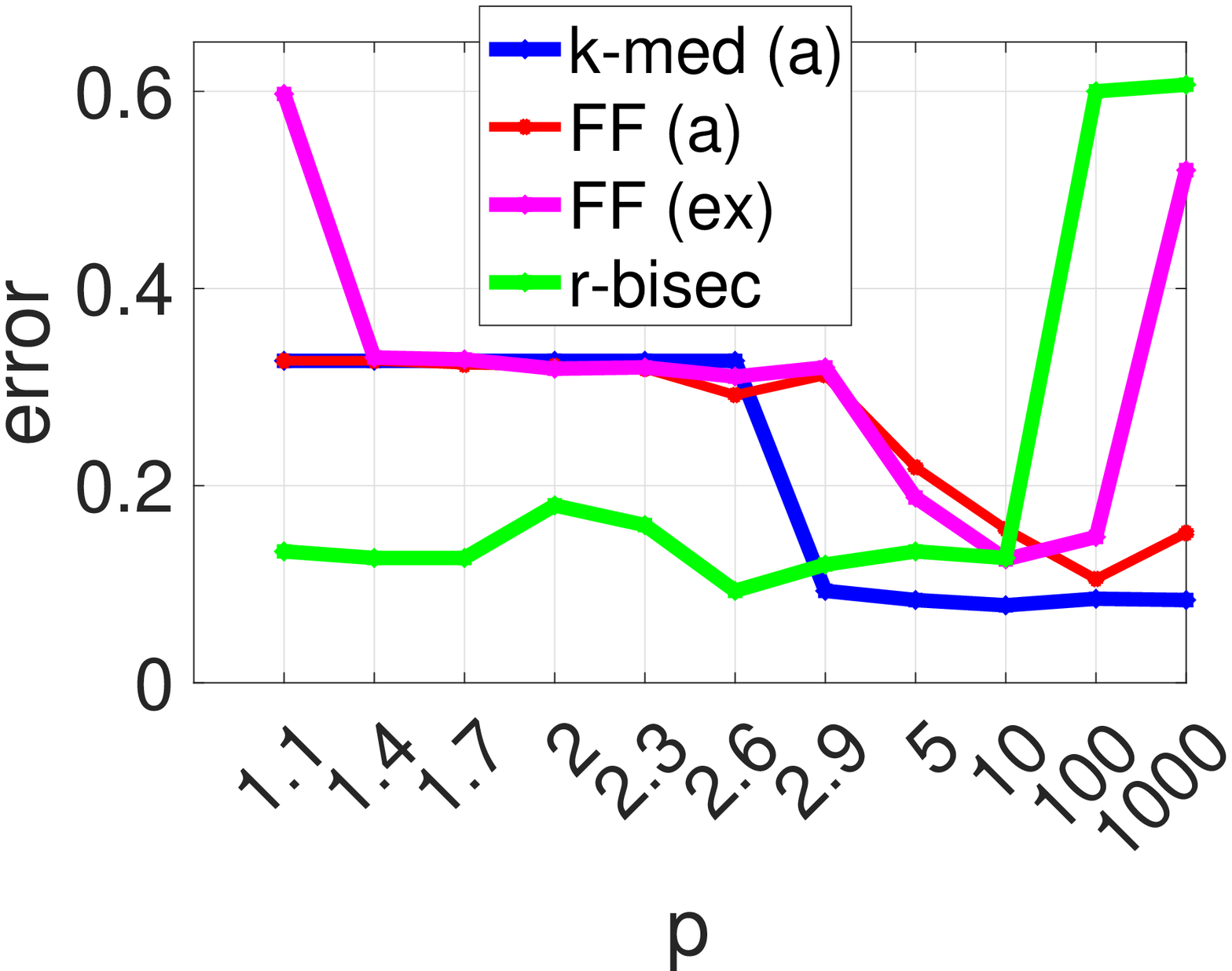}\label{fig:iris}}
\caption{
\small
Plots of the error verse $p$. 
}
\label{fig:exp}
\end{center}
\end{figure*}

\begin{table}[!t]
\small
\caption{Computational time for approximated vs exact $p$-resistance. (a) denotes approximation and (e) denotes exact. In ``r'' we reuse $L^{+}$. In ``et'' we compute $L^{+}$ each time. All time is in second.}
\label{tab:comptimeforapproxvsexact}
\begin{center}
\begin{tabular}{l|ccc}
\toprule
            & ionosphere         & iris               & wine               \\
        \midrule
(a) + r & 0.08 $\pm$ 0.04 & 0.07 $\pm$  0.03 & 0.01 $\pm$ 0.00 \\
(a) + et  & 0.39 $\pm$ 0.00 & 0.32 $\pm$ 0.00  & 0.05 $\pm$ 0.00 \\
(e)         & 1.11 $\pm$ 0.00 & 1.03 $\pm$ 0.04  & 0.36 $\pm$ 0.00 
\end{tabular}
\hfill
\end{center}
\end{table}

This section numerically demonstrates the performance of our Alg.~\ref{algo:clusteringkmedoids} using approximated $p$-resistance.
The purpose of this preliminary experiments is to evaluate if our algorithm on two-class and multi-class clustering problem improves the existing $p$-seminorm based graph clustering algorithm.
Thus, we compared with existing resistance based algorithms and spectral clustering algorithms using graph $p$-seminorm and its $p$$=$$2$ setting. 
For the resistance based method, we compared with the farthest first algorithm on our approximated $p$-resistance.
Additionally, we compared with existing methods; the farthest first using \textit{exact} $p$-resistance~\cite{herbster2010triangle}, a $p$-seminorm flow based method~\cite{nguyen2016new}, and 2-resistance based method~\cite{yen2005clustering}.
We especially note that for the farthest first~\cite{herbster2010triangle}, we computed the exact $p$-resistance by the gradient descent as discussed in Sec.~\ref{sec:approximating}.
We also apply this exact $p$-resistance to $k$-medoids.
Note that $k$-medoids is more costly than $k$-center as discussed in Sec.~\ref{sec:proposed}.
For spectral clustering methods, we compared with a recursively bisection method~\cite{pgraph} and a method using the approximated orthogonality~\cite{luo2010eigenvectors}.
Since the $p$-resistance is related to the unnormalized graph Laplacian, we use unnormalized graph Laplacian for the spectral methods.
Our experiments were performed on classification datasets, ionosphere, iris and wine from UCI repository.
We also used Hopkins155 dataset~\cite{tron2007benchmark}, which contains 120 two-class motion segmentation datasets and 35 three-class ones. 
We created a graph with the following procedure. 
First, we built a $k$-NN graph, where we choose $k$$=$$\mu n$ ($0$$<$$\mu$$\leq$$1$). Then, we computed the edge weight with a Gaussian kernel ($\kappa(\bfx_{i},\bfx_{j})$$=$$\exp (- \sigma \| \bfx_{i}-\bfx_{j} \|^2)$) for two real-valued vectors $\bfx_{i},\bfx_{j}$. 
We used free parameters $\mu$$\in$$\{0.04,0.06,0.08,0.1,1\}$, $\sigma$$\in$$\{10^{-3},\ldots,10^{2}\}$ and $p$$\in$$\{1.1,1.4,\ldots,2.9,5,10,100,1000\}$. 
For comparisons, we followed the original parameters other than above $\mu, \sigma$, and $p$.
We evaluated the performance by error rate, similarly to the previous study~\cite{pgraph}. 
Since the Hopkins155 dataset contains multiple two-class and three-class tasks, we take an average of error rates among a set of two-class tasks and three-class tasks and report both.
The implementation of our method is available at~\url{https://github.com/ShotaSAITO/approximated-presistance}.

The results are summarized as follows. 
We see that ours outperforms the others except for iris.  
As we expected, seeing the deviation, $k$-medoids offers more robust performance than the farthest-first algorithms.
Moreover, our approximation provides faster computation than the exact method since we could not finish some of the experiments using the exact $p$-resistance even for the farthest first.
Seeing Fig.~\ref{fig:exp}, in $k$-medoids large $p$ offers better performance.
Also, if we look at $p=2$, the $k$-medoids with 2-resistance is not always better than spectral clustering. 
However, for general $p$ the $k$-medoids with $p$-resistance performs better. 
Thus, the $k$-medoids with $p$-resistance can be said to be more benefited by $p$ than spectral clustering.
These correspond to the existing theoretical indication; $p$-resistance with large $p$ becomes meaningful function while 2-resistance is not~\cite{alamgir2011phase}.
Comparing exact and approximation $p$-resistance in Fig.~\ref{fig:exp}, while we observe similar performance in the middle range of $p$, we observe the better performance for approximation at the very small $p$ or very large $p$s. 
This might come from the numerical computation of the gradient $\nabla_{\bfx}\|\bfx\|_{G,p}^{p}$ as follows. 
For the exact solution of the very small $p$, the gradient of each step tends to be very small.
For the very large $p$, there is a risk of amplifying round-off numerical error at each step of optimization by taking the power of large $p$.
On the other hand, the approximation offers a robust computation, especially for important large $p$s, because instead we compute by taking the power of small $q$ in Eq.~\eqref{eq:approximation1/p-1}, by which we can avoid the risk discussed.

We next compare times to compute the exact and approximated $p$-resistance for randomly chosen 100 pairs of vertices.
We made a graph for ionosphere, iris, and wine by choosing the best-performing parameters in Table.~\ref{tab:res}.
We measure time for exact $p$-resistance by naively optimizing Eq.~\eqref{eq:penergyandpresistance} by the gradient descent.
For approximation, since we can reuse $L^{+}$ for our approximation, we report the time in two ways; i) we compute $L^{+}$ each time and compute the approximation ii) we reuse $L^{+}$. 
The ``each time'' scenario reports the computational time of $p$-resistance for a single pair.
In the reuse scenario, we measure time $t$ to compute $L^{+}$ and $p$-resistance 100 pairs using $L^{+}$. Then, we report the time $t/100$.
By this, the reuse scenario is much faster than the each time scenario.
Table~\ref{tab:comptimeforapproxvsexact} summarizes the computing time. 
We observe that the approximation method for a single pair is faster than the exact method by comparing the ``each time'' and exact.
As expected, the reuse scenario provides much faster computation than the exact $p$-resistance
For more details and results, including the computing time of Table~\ref{tab:res} and the comparison between values of the exact and approximated $p$-resistance, see Appendix~\ref{sec:detailsof}.

\section{Conclusion} 
\label{sec:conclusion}

We have proposed the multi-class clustering algorithm using the approximated $p$-resistance. 
For this purpose, we have shown the guarantee for the approximation of $p$-resistance.
We also have shown that $p$-resistance characterizes the solution of the semi-supervised learning problem.
Our algorithm has outperformed the existing clustering methods using the graph $p$-seminorm.

The limitation of this work is that we cannot exploit the sparse structures of graphs.
It is because we use $L^{+}$, which becomes dense even if the graph is sparse.
For future work, there remains an ample opportunity to further speed up the procedures involving pseudoinverse of graph Laplacian, such as sparsification techniques~\cite{spielman2008graph,spielman2011graph,spielman2014nearly} or numerical linear-algebraic methods~\cite{saad2003iterative}.
Another direction to consider is connection from $p$-resistance to $p$-Laplacian and nonlinear modularity such as~\cite{tudisco2018community}.
However, unfortunately, the current state of knowledge in the field offers very limited theoretical understanding, even for the standard case; we do not know the connections between 2-resistance and the standard existing clustering methods using the standard modularity and also the standard Laplacian.
Nevertheless, this potential connection highlights interesting open challenges.
Additionally, since we can compute 2-resistance by the simple size $n$ norm $\|(L^{+})^{1/2}\bfe_{i} - (L^{+})^{1/2}\bfe_{j}\|$, it is interesting to see if we can further approximate $\|L^{+}\bfe_{i} - L^{+}\bfe_{j}\|_{G,p}$ by the size $n$ norm instead of the size $m$ graph $p$-seminorm. 
This further speeds up the approximation from $O(m)$ to $O(n)$.
Moreover, instead of our approximated representation of $p$-resistance approach, the exciting approach is to obtain exact representation of $p$-resistance. 
We leave some discussion on the difficulty of this approach in Appendix.~\ref{sec:difficulties}.
Furthermore, the experiments using network datasets would be a nice to confirm.
Finally, it would be also interesting if we apply this $p$-resistance framework to the recently growing space of hypergraph clustering using hypergraph $p$-Laplacians~\cite{TotalVariation,saito2018hypergraph,li2018submodular, saito2023generalizing}.

\section*{Acknowledgments}
This research has been made possible thanks to Huawei, who are
supporting Shota Saito's PhD study at UCL.

{
\bibliography{reference}
}
\bibliographystyle{icml2023}

\clearpage
\appendix
\onecolumn

\section{On Limitation to Multi-class Spectral Clustering using Graph $p$-Laplacian}
\label{sec:onlimitation}

In Sec.~\ref{sec:introduction} and Sec.~\ref{sec:related}, we claim that it is difficult to obtain the third or higher eigenpairs of graph $p$-Laplacian.
This claim is a key motivation on why we take an alternate approach to the graph $p$-Laplacian.
Thus, for completeness, this section briefly reviews graph $p$-Laplacian, and its difficulties to apply for multi-class clustering.

\subsection{Eigenpairs of $p$-Laplacian}
\label{sec:eigenapirsofplaplacian}

To start, following~\cite{pgraph} we introduce the graph $p$-Laplacian and its eigenpairs.
The graph $p$-Laplacian $\Delta_{p}$ is defined as
\begin{align}
\label{eq:defplaplacian}
    (\Delta_{p}\bfx)_{i}:= \sum_{i,j\in V} a_{ij}|x_{i} - x_{j}|^{p-1}\sgn(x_{i} - x_{j}).
\end{align}
By this construction, we have $S_{G,p}(\bfx) = \langle \bfx, \Delta_{p}\bfx \rangle$.
The $k$-th eigenpair ($\lambda_{k}, \bfx_{k}$) of the $p$-Laplacian $\Delta_{p}$ satisfies
\begin{align}
    (\Delta_{p} \bfx_{k})_{i} = \lambda_{k} |(\bfx_{k})_{i}|^{p-1} \mathrm{sgn}((\bfx_{k})_{i}), \forall i \in V.
\end{align}
Note that the first eigenpair is $(0,\mathbf{1})$.
The eigenpairs are characterized by the critical values of the Rayleigh quotient.
\begin{proposition}[\citet{tudisco2018nodal}]
\label{prop:critical}
The eigenpairs of graph $p$-Laplacian is a critical value and point of the Rayleigh quotient defined as
\begin{align}
    \label{eq:rayleigh}
    R_{G,p}(\bfx) := \frac{S_{G,p}(\bfx)}{\|\bfx\|_{p}^{p}}.
\end{align}
\end{proposition}
From this proposition, we can know that $R_{G,p}(a\bfx) = R_{G,p}(\bfx)$, for $a\in \R$.
Therefore, to consider the eigenpairs of the graph $p$-Laplacian, we can limit our interest to $\mcalS_{p} := \{\bfx \mid \| \bfx \|^{p}_{p} = 1\}$, seeing the Rayleigh quotient in Eq.~\eqref{eq:rayleigh}.

In the sequel, we briefly explain why we can obtain the second eigenpair and why it is difficult to obtain the third or higher eigenpairs of the graph $p$-Laplacian.
We now define the following quotient,
\begin{align}
\label{eq:rayleigh_2}
R^{(2)}_{G,p}(\bfx) := \frac{S_{G,p}(\bfx)}{\min_{\eta}\|\bfx - \eta \mathbf{1}\|_{p}^{p}}.
\end{align}
This quotient gives the second eigenpair of $p$-Laplacian.
\begin{proposition}[\citet{pgraph}]
\label{prop:secondeigenglobal}
The global solution to Eq.~\eqref{eq:rayleigh_2} is given by $\bfx^{*} = \bfx_2 + \eta^{*}\mathbf{1}$, where $\eta^{*}=\argmin_{\eta} \| \bfx_2 - \eta \mathbf{1} \|_{p}^{p}$, and $\bfx_2$ is the second $p$-eigenvector.
\end{proposition}
This proposition shows that we have an exact identification for the second $p$-eigenpair; minimizing Eq.~\eqref{eq:rayleigh_2} gives the second $p$-eigenpair of $\Delta_{p}$.
However, we have not known yet the exact identification for third or higher eigenpair of $p$-Laplacian~\citep{lindqvist2008nonlinear}.

While we do not know the identification such as Eq.~\eqref{eq:rayleigh_2} for the higher order eigenpairs, the next question is if there is characterization of eigenpairs of the graph $p$-Laplacian, such as ``orthognality'' for the $p=2$ case.
When $p=2$, the eigenvectors of the graph Laplacian are further characterized from Prop.~\ref{prop:critical}; the eigenvectors of graph Laplacian are orthogonal to each other.
By using orthogonality and Rayleigh-Ritz theorem, we can obtain the full eigenvectors of the graph 2-Laplacian as

\begin{proposition}[Rayleigh-Ritz]
\label{prop:rayleighritzminmaxp=2}
Let $\bfx_{1},\ldots\bfx_{k-1}$ be eigenvectors of the graph Laplacian $L$. Then the $k$-th eigenvector $\bfx_{k}$ is given as
\begin{align}
\label{eq:seqeigenp=2}
    \bfx_{k} = \argmin_{\bfx} R_{G,2}(\bfx) \quad \mathrm{s.t.}\  \bfx_{k} \bot \bfx_{1},\ldots,\bfx_{k-1}.
\end{align}
\end{proposition}

By using this proposition and the sequence Eq.~\eqref{eq:seqeigenp=2}, we can easily obtain the higher order eigenvectors.
This orthogonality constraint of eigenvectors comes from the nature of $L^{2}$ space induced from the graph 2-seminorm.
However, we lose this sense of orthogonality if we expand to $p$-seminorm, since we lose the inner product structure in the $L^{p}$ space.

In this context, the following further generalizes the ``orthogonality'' to graph $p$-Laplacian.
To consider how we generalize the ``orthogonality'', we look back to the $p=2$ case. 
Prop.~\ref{prop:rayleighritzminmaxp=2} can be rewritten as in a more abstract fashion as follows.
\begin{proposition}[Courant's min-max]
   \label{prop:rayleighcourant}
The $k$-th eigenvalue $\lambda_{k}$ of the graph Laplacian $L$ is given as
\begin{align}
\label{eq:seqeigencourantp=2}
\lambda_{k} = \min_{B; \mathrm{dim}(B)=k}\max_{\mathbf{x} \in B}\frac{\|\mathbf{x}\|_{G,2}^{2}}{\|\mathbf{x}\|_{2}^{2}}.
\end{align} 
\end{proposition}
This proposition is also called as \textit{variational principle}.
Since the graph Laplacian is simpler due to its symmetric matrix nature, we have a simpler claim of Prop.~\ref{prop:rayleighcourant} as Prop.~\ref{prop:rayleighritzminmaxp=2}.
Prop.~\ref{prop:rayleighritzminmaxp=2} shows that the dimension plays a role to characterize the eigenpairs for the $p=2$ case.
The idea to characterize the eigenpairs of the graph $p$-Laplacian is that we consider some generalized dimension that is suitable for the $p$-seminorm.

Aiming for this ``generalized'' dimension, we use \textit{Krasnoselskii genus} $\gamma$ for a set $B$; 
\begin{align}
\label{eq:defgenus}
    \gamma(B) =
    \left\{
    \begin{array}{l}
     0  \mathrm{\ if}\ B=\emptyset \\
      \inf \{k \in \mathbf{Z}^{+} \mid \exists \mathrm{odd\ continuous}\ h:B \rightarrow \mathbf{R}^{k}\backslash \{0\}  \}   \\
      \infty  \mathrm{\ when\ no\ such\ }h \mathrm{\ exists\ } \forall j \in \mathbf{Z}^{+}
    \end{array}
    \right.
\end{align}
This genus is a generalized concept of dimension of $B$. 
Using this genus, we can characterize the eigenpairs;
\begin{proposition}[\citet{tudisco2018nodal}]
\label{prop:criticalvalue}
Consider the set of subsets $\mcalF_k(\mathcal{S}_{p}) := \{ B \subset \mcalS_{p} \mid B = -B, \mathrm{\ closed\ }, \gamma(B) \geq k \}$. 
The sequence defined as
    \begin{align}
        \label{eq:minmax}
        \lambda_{k}  = \min_{B \subset \mcalF_k(\mathcal{S}_{p})} \max_{\bfx \in B} R_{p} (\bfx) 
    \end{align}
gives a critical point of $R_{p} (\bfx)$, whose corresponding $\bfx$ and $\lambda_k$ constitute an eigenpair of $\Delta_p$. 
\end{proposition}

This proposition is the generalized Courant's min-max theorem Prop.~\ref{prop:rayleighcourant}; when $p=2$, this proposition corresponds to Prop.~\ref{prop:rayleighritzminmaxp=2}.
This proposition is again also called as \textit{variational principle}.
The eigenvectors obtained by the sequence Eq.~\eqref{eq:minmax} is called as \textit{variational eigenvectors}. 

In this proposition, the space $\mcalF_k(\mathcal{S}_{p})$ serves as a generalized orthogonal $k$-dimensional space.
Moreover, the sequence Eq.~\eqref{eq:minmax} may serve as a method to obtain eigenpairs in the sequential way.
However, we have two issues for the practical use of this sequence.
First problem is that due to the abstract characterization of Krasnoselskii genus, we do not know how we can \textit{numerically} apply this genus to obtain the higher eigenvectors. 
When $p=2$, this abstract characterization can be translated into the concrete and ``numerically computable'' characterization, ``orthogonality''. 
However, in the current form of Krasnoselskii genus given as Eq.~\eqref{eq:defgenus}, at this point we do not know how to numerically obtain this genus.
Secondly, similarly to the continuous $p$-Laplacian theory~\citep{lindqvist2008nonlinear}, we do not know in which condition this sequence yields exhaustive eigenpairs. 
For the tree (and the disconnected forest) case the sequence Eq.~\eqref{eq:minmax} exhausts all the spectra~\cite{deidda2022nodal,zhang2021homological}. 
On the other hand, for the complete graph case, it is shown that there are other eigenpairs than ones yielded by the sequence Eq.~\eqref{eq:minmax}~\cite{amghibech2003eigenvalues}.
Despite these extensive studies, we are yet to understand in which conditions this sequence exhausts all the spectra of $p$-Laplacian.
Thus, for a general graph we do not know if the variational eigenvalues are the same eigenvalues as the Rayleigh quotient would do in Prop.~\ref{prop:critical}.

To conclude the discussion above, while we know the identification for the second eigenpair of the graph $p$-Laplacian (Eq.~\eqref{eq:rayleigh_2}), we have three open problems related to the multi-class spectral clustering as follows;
\begin{enumerate}
    \item We do not know the identification of the third or higher eigenpairs. 
    \item We do not know if the sequence Eq.~\eqref{eq:minmax} exhausts the spectra of the graph $p$-Laplacian for a general graph.
    \item We do not know how to numerically obtain Krasnoselskii genus.
\end{enumerate}

\subsection{Cheeger Inequalities for graph $p$-Laplacian}
This section discusses Cheeger inequalities for the graph $p$-Laplacian. 
The Cheeger inequality theoretically supports the use of the variational eigenvectors of $p$-Laplacian from the Cheeger cut point of view. 

We start our discussion from a 2-way Cheeger cut.
Let $U \subset V$ be a set and $\overline{U}$ be a complement of $U$.
A Cheeger cut may be defined as
\begin{align}
\label{eq:cheegercut}
     &C(U) :=  \frac{ \cut (U, \overline{U}) }{\min(|U|,|\overline{U}|)},\quad \cut(U,\overline{U}) := \sum_{i,j \in V}  a_{ij}
\end{align}
We call the optimal cut $h_2:=\min_{U \subset V}C(U)$ as \textit{Cheeger constant}.
By recursively using this Cheeger cut, we define the multi-class Cheeger cut, which we call \textit{$k$-way Cheeger constant} as
\begin{align}
\label{eq:kwaycheeger}
     h_k := \min_{\{V_i\}_{i=1,\cdots,k}} \max_{j \in \{1,\ldots,k \}} C(V_j).
\end{align}
This $k$-way Cheeger constant can be seen as the smallest $k$-way Cheeger cut. 
To obtain the $k$-way Cheeger cut is known to be NP hard.
However, relaxing into the real-value would ease this problem; this Cheeger cut can be approximated by the variational eigenvalues of the graph $p$-Laplacian by Cheeger inequality.

Before we discuss the Cheeger inequality, we need a setup of the nodal domain. 
A nodal domain is a maximally connected subgraph $A$ of a graph $G$ such that for $\bfx \in \R^{n}$ where $A$ is either $\{i\mid x_{i}>0\}$ or $\{i \mid x_{i}<0\}$.
With this idea, a nodal domain can be seen as a ``partition'' of the graph by the sign function.
The number of the nodal domains of the variational eigenvectors is bounded, see~\cite{tudisco2018nodal}.

This nodal domain is used in the Cheeger inequality for the variational eigenvectors of graph $p$-Laplacian as follows.
\begin{proposition}[\cite{tudisco2018nodal}]
\label{prop:cheeger}
Let $(\lambda_k, \bfx_k)$ be a $k$-th eigenpair of $\Delta_p$, obtained by the sequence Eq.~\eqref{eq:minmax}. Let also $\omega_k$ be a number of nodal domains of $\bfx_k$. 
Then,
\begin{align}
\notag
\left(\max_{i}\frac{d_{i}}{2}\right)^{-(p-1)} \left(\frac{h_{\omega_k}}{p}\right)^{p} \leq \lambda_k \leq 2^{p-1} h_k
\end{align}
\end{proposition}
This proposition theoretically supports the use of the higher variational eigenvectors of the graph $p$-Laplacian for multi-way Cheeger Cut, since the eigenvalues can serve as an approximation of the $k$-way Cheeger constant.

For two-class spectral clustering, we are ready to use the second eigenvector of the graph $p$-Laplacian because we have theoretical supports (Prop.~\ref{prop:cheeger}) and we can numerically obtain by Prop.~\ref{prop:secondeigenglobal}. 
On the other hand, for multi-class spectral clustering, while we still have a theoretical supports (Prop.~\ref{prop:cheeger}), at this point we do not have a numerical way to obtain the higher eigenpairs due to open problems.

\subsection{Existing Work for Multi-class Spectral Clustering Using Graph $p$-Laplacian}

So far, we discuss limitations for spectral clustering using the graph $p$-Laplacian.
This section discusses how the existing works ``bypass'' this limitation.
In a rough sense, there are two ways to materialize the multi-class clustering using graph $p$-Laplacian; i) recursively bisectioning and ii) the use of the approximated orthogonality.

For i), the work~\cite{pgraph} proposed a multi-class clustering, which recursively bisections a graph by using Prop.~\ref{prop:secondeigenglobal}.
Thus, the work~\cite{pgraph} partitions a subgraph when we partition further than two, which does not exploit the full structure of a graph. 
In fact, this way is known to lead to the suboptimal cut even when $p=2$~\cite{simon1997good}.

The methods in line with ii), such as~\cite{ding2019multiway, luo2010eigenvectors,pasadakis2022multiway}, assume that the $k$ eigenvectors of the graph $p$-Laplacian are close to ones of the graph 2-Laplacian, since the Rayleigh quotients $R_{G,p}$ and $R_{G,2}$ are similar.
Using this assumption, these methods use optimization methods for the Rayleigh quotients $R_{G,p}$ with the initial conditions as the first $k$-eigenvectors of the graph 2-Laplacian.
By these initial conditions, we expect that the obtained $k$-eigenvectors are ``close'' to the first $k$-eigenvectors of the graph 2-Laplacian, and thus we expect that the obtained $k$-eigenvectors are the first $k$-eigenvectors from the Rayleigh quotient $R_{G,p}$.
These methods exploit approximated orthogonality proven in~\cite{luo2010eigenvectors} of eigenvectors of graph $p$-Laplacian in order to achieve better algorithms. 
However, this assumption might be too strong, especially for the very large $p$ or very small $p$, i.e., $p$ close to 1.
Moreover, even if this assumption may be reasonable, we do not know if the first $k$ eigenvectors from Rayleigh quotient are the same set of vectors with the first $k$-eigenvectors obtained by the sequence Eq.~\eqref{eq:minmax} as we discussed in Sec.~\ref{sec:eigenapirsofplaplacian}.
Now, recall that the Cheeger inequality guarantees the quality of the cut for the latter, the eigenvectors obtained by Eq.~\eqref{eq:minmax}. 
Thus, at this point, we do not know if this assumption is suitable for multi-class clustering.

As a side, we remark that we have a different way to define graph $p$-energy and corresponding graph $p$-seminorm in many literature~\cite{Bougleux07,BougleuxEM09,calder2018game,Elmoataz08,singaraju2009p,Zhou06}.
In this line, the $p$-energy is defined in ``vertex-wise'' way, which is written as
\begin{align}
    S_{G,p}^{VW}(\bfx) = \sum_{i \in V} \left(\sum_{j \in V} a_{ij} |x_{i} - x_{j}|^{2}\right)^{p/2},
\end{align}
This definition sums the vertex-wise energy.
For more details of the difference, see~\cite{saito2018hypergraph}.
For the corresponding $p$-Laplacian, we do not have an exact identification of higher eigenpairs either.
Moreover, for the corresponding seminorm, we have not theoretical characteristics yet, such as Cheeger inequality. 
Thus, for this graph $p$-Laplacian, we have less understanding than the graph $p$-Laplacian induced from the $p$-energy we used in the main text.
 
To conclude, there are three important open problems for the graph $p$-Laplacian. 
As we see in this section, these open problems are a key if we want to apply the graph $p$-Laplacian to multi-class clustering.
Hence, without solving these open problems, we cannot say that it is theoretically guaranteed to use the graph $p$-Laplacian for multi-class clustering.
However, these problems remain open not only in the graph domain, but also in the continuous domain, which has a longer history and wider research communities.
Thus, in this study, for the purpose of multi-class clustering, we take an alternative approach to spectral clustering using the graph $p$-Laplacian.
For more on the continuous $p$-Laplacian, see~\cite{lindqvist2008nonlinear}, and on the theory of the graph $p$-Laplacian, refer to~\cite{tudisco2018nodal}.

\section{Additional Preliminary Definitions}

First, we make an additional note for an intuition behind the analog between graph and electric circuit.
In this analog, a vertex is a point at a circuit, and an edge is a resistor with resistance $1/a_{ij}$. 
A flow over a graph mapped to a current, and a distribution over $V$ as $\bfx$ is seen as a potential at each vertex point. 
For the equations Eq.~\eqref{eq:defenergyandresistance}, the energy is defined as a sum of the inverse of resistance times square of the difference of the potential.
The effective resistance between $i$ and $j$ is computed as follows; we inverse the energy that is minimized with the constraint that the difference of potential between $i$ and $j$ is unit.
Given an electrical network the effective resistance between two vertices is the voltage difference needed to induce a unit ``current'' flow between the vertices i.e., it is resistance measured across the vertices.

Next, on top of the image for a matrix $M\in \R^{n_{1} \times n_{2}}$, $\Image(M)$, we also define a \textit{kernel}\footnote{This kernel is a linear algebraic kernel, not a kernel function which often appears in the machine learning context.} of $M$, which is a subclass of $\R^{n}$, as
\begin{align}
    \Kernel(M) := \{\bfx | M\bfx = 0, \bfx \in \R^{n}\}.
\end{align}
From the elementary result in the linear algebra area, we note that
\begin{align}
    \Image(M)^{\bot} = \Kernel(M^{\top}),
\end{align}
where $\Image(M)^{\bot}$ is an orthogonal space to $\Image(M)$.

The matrix norm is \textit{submultiplicative}, i.e., $\vvvert M_{1}M_{2}$$\vvvert_{p}$$ \leq$$\vvvert M_{1}\vvvert_{p} \vvvert M_{2}\vvvert_{p}$ whenever a product of matrices $M_{1}M_{2}$ can be defined.
A matrix norm is shown to be bounded as follows;
\begin{lemma}[\citet{higham1992estimating}]
\label{lemma:higham}
For a square matrix $M$$\in$$\R^{n_{1}\times n_{1}}$, $\|M\|_{p} $$ \leq $$ n_{1}^{|1/2-1/p|}\|M\|_{2}$.
\end{lemma}
\begin{lemma}[\citet{higham1992estimating}]
\label{lemma:higham2}
For a matrix $M \in \R^{n_{1}\times n_{2}}$, $\|M\|_{p} $$ \leq $$ \max (\|M\|_{1},\|M\|_{\infty})$. 
\end{lemma}
We elaborate more on Lemma~\ref{lemma:higham2}.
For a symmetric matrix, since we have
\begin{align}
\label{eq:1normandinfnorm}
    \vvvert M \vvvert_{1} = \vvvert M \vvvert_{\infty} = \max_{j}\sum_{i} |m_{ij}|.
\end{align}
From the Lemma~\ref{lemma:higham2} and Eq.\eqref{eq:1normandinfnorm}, for a symmetric matrix $M$, we have
\begin{align}
\label{eq:lemmahigham2symmetric}
    \vvvert M \vvvert _{p} \leq \vvvert M \vvvert_{1} = \vvvert M \vvvert_{\infty}.
\end{align}
By this we can bound $\vvvert M \vvvert_{p}$ by 1 or infinity norm of the matrix $M$.

An operator weighted matrix norm is defined for any matrix $M \in \R^{n_{1} \times n_{2}}$ and weights $\bfr$ as 
\begin{align}
\label{eq:defmatrixnorm}
     \vvvert M  \vvvert_{\bfr,p} := \sup_{\bfx \in \R^{n_{2}}} \frac{\|M\bfx\|_{\bfr,p}}{\|\bfx\|_{\bfr,p}}.
\end{align}
Recall that
\begin{align}
    \|\bfx\|_{\bfr,p} = \|R^{1/p}\bfx\|_{p},
\end{align}
where $R$ is a diagonal matrix whose diagonal element is a weight of the norm.

From this definition, we can rewrite $\vvvert M \vvvert_{\bfr,p}$ as
\begin{align}
\notag
    \vvvert M \vvvert_{\bfr,p} 
    &= \sup_{\bfx \in \R^{n_{2}}} \frac{\|M\bfx\|_{\bfr,p}}{\|\bfx\|_{\bfr,p}}\\
    &= \sup_{\bfx \in \R^{n_{2}}} \frac{\|R^{1/p}M\bfx\|_{p}}{\|R^{1/p}\bfx\|_{p}}\\
    &=\sup_{\bfx' := R^{1/p}\bfx, \bfx \in \R^{n_{2}}} \frac{\|R^{1/p}MR^{-1/p}\bfx'\|_{p}}{\|\bfx'\|_{p}}\\ 
    \label{eq:weightedtounweightednorm}
    & =  \sup_{\bfx'\in \R^{n_{2}}} \frac{\|R^{1/p}MR^{-1/p}\bfx'\|_{p}}{\|\bfx'\|_{p}} \\
    &= \sup_{\bfx \in \R^{n_{2}} } \frac{\|R^{1/p}MR^{-1/p}\bfx\|_{p}}{\|\bfx\|_{p}}\\ 
    &= \vvvert R^{1/p}MR^{-1/p} \vvvert_{p},
\end{align}

\section{Details of Lemma~\ref{lemma:p=2holderseqcond}}
\label{sec:detailsoflemmap=2}

This section provides detailed explanation on Lemma~\ref{lemma:p=2holderseqcond}.
We first note that the trick in this transformation is as same as done in~\cite{herbster2006prediction,klein1993resistance}.
The elaboration here follows these earlier works. 

Using the reproducing property Eq.~\eqref{eq:coordinate} as done in~\cite{herbster2006prediction,klein1993resistance}, the constraints of $2$-resistance (and also for $p$-resistance Eq.~\eqref{eq:penergyandpresistance}) can be rewritten as
\begin{align}
\label{eq:constraints}
    1= x_{i} - x_{j} = 
    \langle \bfx, L^{+}\bfe_{i} - L^{+}\bfe_{j} \rangle_{L}.
\end{align}

Now, since $L\mathbf{1} = 0$, there exists $c \in \R$ such that $\bfx - c\mathbf{1} \in \mathcal{H}(L)$.
We now define as $\bfx' := \bfx - c\mathbf{1}$.
We then compute
\begin{align}
    \langle \bfx',L^{+}\bfe_{i}\rangle_{L} 
    &= \bfx'^{\top}LL^{+}\bfe_{i}\\
    &= \bfx'^{\top}\bfe_{i}\\
    \label{eq:coordinatex'}
    &= x'_{i}.
\end{align}
The second line follows since we have $LL^{+}\bfx' = \bfx'$ for $\bfx' \in \mathcal{H}(L)$.
Note that this computation is same as the reproducing kernel property characteristics Eq.~\eqref{eq:coordinate}.
Also, the definition of $\bfx'$ immediately leads to
\begin{align}
    L\bfx' = L(\bfx - c\mathbf{1}) = L\bfx - cL\mathbf{1} =  L\bfx,
\end{align}
and thus for $\bfu \in \R^{n}$ we have
\begin{align}
    \langle \bfx', \bfu \rangle_{L} 
    &= \bfx'^{\top}L\bfu\\
    &= \bfx^{\top}L\bfu\\
    \label{eq:innerproductxuL}
    &= \langle \bfx,\bfu \rangle_{L}
\end{align}
From these discussions, we obtain
\begin{align}
    1
    &= x_{i} - x_{j}\\ 
    &= (x_{i} - c) - (x_{j} - c)\\
    &= x'_{i} - x'_{j}\\
    &= \langle \bfx',L^{+}\bfe_{i}\rangle_{L} - \langle \bfx',L^{+}\bfe_{j}\rangle_{L} \\
    &= \langle \bfx,L^{+}\bfe_{i}\rangle_{L} - \langle \bfx,L^{+}\bfe_{j}\rangle_{L} \\
    &= \langle \bfx, L^{+}\bfe_{i} - L^{+}\bfe_{j} \rangle_{L}.
\end{align}
The third line follows from the definition of $\bfx'$. The fourth line follows from Eq.~\eqref{eq:coordinatex'} and the fifth line follows from Eq.~\eqref{eq:innerproductxuL}. Thus, we obtain Eq.~\eqref{eq:constraints}

\begin{comment}
\section{Proof for Lemma.~\ref{lemma:holders}}

For this weighted $p$-norm and inner product, we have \holders inequality as follows;
\begin{lemma}[\holders inequality]
For $p,q$ s.t. $1/p + 1/q=1$, $\langle \bfx,\bfy \rangle_{\bfr} \leq \|\bfx\|_{\bfr,p}\|\bfy\|_{\bfr,q}.$
%\begin{align}
%    \langle \bfx,\bfy \rangle_{\bfr} \leq \|\bfx\|_{\bfr,p}\|\bfy\|_{\bfr,q}.
%\end{align}
\end{lemma}
By substituting $\bfx := C\bfx$, $\bfy := C\bfy$ and $\bfr := \bfw$ to this Lemma, we get the claim.

We leave a proof of Lemma~\ref{lemma:p=2holderseqcond} for the proof of Prop.~\ref{prop:newkr}.
\end{comment}

\section{Proof for Proposition ~\ref{prop:newkroverc}}

For the proof we divide the proof into lower bound, upper bound and equal condition.

\subsection{Lower Bound}
In the following we give a proof for this Proposition.
From \holders inequality, we have
\begin{align}
    \langle \bfx, \bfy \rangle_{L} \leq \|\bfx\|_{G,p} \|\bfy\|_{G,q}
\end{align}
Assuming $\langle \bfx, \bfy \rangle_{L} = 1$, we have
\begin{align}
    1 \leq \|\bfx\|_{G,p}\|\bfy\|_{G,q},
\end{align}
and hence
\begin{align}
\label{eq:holdersC}
    \|\bfy\|_{G,q}^{-p} \leq \|\bfx\|_{G,p}^{p},
\end{align}
which proves the lower bound.

\subsection{Upper Bound}

This section proves the upper bound of Prop.~\ref{prop:newkroverc}.

Recall the variable of the minimization problem in Eq.~\eqref{eq:holdersp} is $\bfx$.
If we prove that when $\bfx = \bfz$, this $\bfz$ satisfies the condition in the minimization problem $\langle \bfz,\bfy \rangle_{L} = 1$, from the minimization problem nature we can prove the upper bound.
For this strategy, we use the following lemma.

\begin{lemma}
\label{lemma:innerproductoverc}
For $\bfalpha \in \R^{n}$ and $\bfbeta \in \R^{m}$, we have
\begin{align}
    \langle C\bfalpha, \bfbeta \rangle_{\bfw} = \langle C\bfalpha, \bfbeta'\rangle_{\bfw}
\end{align}
where
\begin{align}
    \bfbeta := \bfbeta' + \bfbeta'', \quad \bfbeta' := CC^{+}\bfbeta, \bfbeta'' := (I-CC^{+})\bfbeta
\end{align}
\end{lemma}

For readability, we move the proof of Lemma~\ref{lemma:innerproductoverc} to Sec.~\ref{sec:prooflemmainnerproductoverc}.

By using Lemma~\ref{lemma:innerproductoverc}, we now prove the upper bound.

Note that
\begin{align}
    \frac{f_{q/p}(C\bfy)}{\|C\bfy\|_{\bfw,q}^{q}} = CC^{+} \frac{f_{q/p}(C\bfy)}{\|C\bfy\|_{\bfw,q}^{q}} + (I-CC^{+}) \frac{f_{q/p}(C\bfy)}{\|C\bfy\|_{\bfw,q}^{q}}.
\end{align}

Recall that we define as
\begin{align}
    \bfz := C^{+} \frac{f_{q/p}(C\bfy)}{\|\bfy\|_{G,q}^{q}}=C^{+} \frac{f_{q/p}(C\bfy)}{\|C\bfy\|_{\bfw,q}^{q}}, \quad (f_{\theta}(\bfx))_{i} := \sgn (x_{i}) |x_{i}|^{\theta}.
\end{align}

By using this relation and Lemma~\ref{lemma:innerproductoverc}, we have
\begin{align}
\langle \bfz, \bfy \rangle_{L} 
&=  \langle C\bfz, C\bfy \rangle_{\bfw} \\
&=  \langle C\bfy, C\bfz \rangle_{\bfw} \\
\label{eq:CC+fpqinner}
&= \langle C\bfy,  CC^{+} \frac{f_{q/p}(C\bfy)}{\|C\bfy\|_{\bfw,q}^{q}} \rangle_{\bfw} \\
\label{eq:fpqinner}
&= \langle C\bfy,\frac{f_{q/p}(C\bfy)}{\|C\bfy\|_{\bfw,q}^{q}} \rangle_{\bfw}\\
&= 1.
\end{align}

From Eq.~\eqref{eq:CC+fpqinner} to Eq.~\eqref{eq:fpqinner} we apply Lemma~\ref{lemma:innerproductoverc}.
The last equality comes from the same nature of Eq.~\eqref{eq:zetainnerproduct} in Prop.~\ref{prop:newkr}, which we will discuss in Sec.~\ref{sec:auxiliary}.

From the discussion above, $\bfz$ satisfies the condition of the minimization problem of Eq.~\eqref{eq:holdersp}.
Therefore, we obtain
\begin{align}
    \min_{\bfx}\{ \|\bfx\|_{G,p}^{p} \ \mathrm{s.t.} \ \langle \bfy, \bfx \rangle_{L} = 1\} \leq \|\bfz\|_{G,p}^{p}
\end{align}

\subsection{Proof for the Equal Condition}
Finally, we turn into the equality condition.
We obtain the following lemma.
\begin{lemma}
\label{lemma:pholderseqcond}
For any $p,q$ such that $1/p + 1/q =1$, we have
\begin{align}
    \min_{\bfx} \{ \|\bfx\|_{G,p}^{p} \subjectto \langle \bfx,\bfy \rangle_{L} =1 \} = \|\bfz\|_{G,q}^{p} = \|\bfy\|_{G,q}^{-p},
\end{align}
where 
\begin{align}
    \mathbf{z} := C^{+} \frac{f_{q/p}(C\bfy)}{\|\bfy\|_{G,q}^{q}} = C^{+} \frac{f_{q/p}(C\bfy)}{\|C\bfy\|_{\bfw,q}^{q}},
\end{align}
when $f_{q/p}(C\bfy) \in \Image(C)$
\end{lemma}

The proof of Lemma~\ref{lemma:pholderseqcond} is given in Sec.~\ref{sec:pholderseqcond}

\subsection{Proof for Lemma~\ref{lemma:innerproductoverc} and  Lemma~\ref{lemma:pholderseqcond}}

This section provides proofs for Lemma~\ref{lemma:innerproductoverc} and  Lemma~\ref{lemma:pholderseqcond}.
These lemmas are critical components for the proof for Prop.~\ref{prop:newkroverc}.
In order to enhance the readability, we gather proofs for these claims in this section.
We first give auxiliary lemmas, that hold for the general setting.
Then, using these auxiliary lemmas, we provide the proofs for Lemma~\ref{lemma:innerproductoverc} and Lemma~\ref{lemma:pholderseqcond}.

\subsubsection{Auxiliary Lemmas}
\label{sec:auxiliary}

This section provides auxiliary lemmas.
We start with the following claim.
\begin{proposition}
\label{prop:newkr}
For any $p,q>1$ such that $1/p + 1/q = 1$, we have
\begin{align}
     \min_{\bfx}\{ \|\bfx\|_{\bfr,p}^{p} \ \mathrm{s.t.} \ \langle \bfy, \bfx \rangle_{\bfr} = 1\} =\|\bfy\|_{\bfr,q}^{-p} 
\end{align}
\end{proposition}
\begin{proof}
Using the \holders inequality, we get
\begin{align}
     \|\bfy\|_{\bfr,q}\|\bfx\|_{\bfr,p} \geq \langle \bfx,\bfy \rangle_{\bfr}
\end{align}
Assuming $\langle \bfx,\bfy \rangle_{\bfr} = 1$, we can rearrange as
\begin{align}
\label{eq:holdersinequalitytransformed}
    \|\bfx\|_{\bfr,p} \geq \|\bfy\|_{\bfr,q}^{-1}.
\end{align}
Now we consider when the minimum of the right hand side of Eq.~\eqref{eq:holdersinequalitytransformed}.
The minimum with the assumption $\langle \bfy, \bfx \rangle_{\bfr} = 1$ is achieved when $\bfx = \bfzeta$ such that
\begin{align}
\label{eq:zetadef}
    \bfzeta := \frac{f_{q/p}(\bfy)}{\|\bfy\|_{\bfr,q}^{q}}, \quad (f_{\theta}(\bfy))_{i} := \mathrm{sgn} (y_{i}) |y_{i}|^{\theta}
\end{align}
which means
\begin{align}
    \zeta_{i} =\frac{\mathrm{sgn}(y_{i})|y_{i}|^{q/p}}{\|\bfy\|_{\bfr,q}^{q}}.
\end{align}
For this $\bfzeta$, we compute
\begin{align}
    \langle \bfy, \bfzeta \rangle_{\bfr} 
     &  =\sum_{i=1}^{n} r_{i}y_{i}\zeta_{i}\\
    &=\sum_{i=1}^{n} r_{i}y_{i} \frac{\mathrm{sgn}(y_{i})|y_{i}|^{q/p}}{\|\bfy\|_{\bfr,q}^{q}}\\
    \label{eq:thirdzetainnerproduct}
    &= \frac{\sum_{i=1}^{n} r_{i}|y_{i}|^{q/p + 1}}{\|\bfy\|_{\bfr,q}^{q}}\\
    \label{eq:fourthzetainnerproduct}
    &= \frac{\sum_{i=1}^{n} r_{i}|y_{i}|^{q}}{\|\bfy\|_{\bfr,q}^{q}}\\
    \label{eq:zetainnerproduct}
    &= \frac{\|\bfy\|_{\bfr,q}^{q}}{\|\bfy\|_{\bfr,q}^{q}} = 1.
\end{align}
The transition from Eq.~\eqref{eq:thirdzetainnerproduct} to Eq.~\eqref{eq:fourthzetainnerproduct} comes from $q/p+1 = q$.
Also, we have
\begin{align}
    \|\bfzeta\|_{\bfr,p} 
    &=\left\|\frac{f_{q/p}(\bfy)}{\|\bfy\|_{\bfr,q}^{q}} \right\|_{\bfr,p}\\
    &= \left(\sum_{i=1}^{n} r_{i}\left|\frac{\mathrm{sgn}(y_{i})|y_{i}|^{q/p}}{\|\bfy\|_{\bfr,q}^{q}}\right|^{p}\right)^{1/p}\\
    &= \frac{1}{\|\bfy\|_{\bfr,q}^{q}} \left(\sum_{i=1}^{n} r_{i}\left|\mathrm{sgn}(y_{i})|y_{i}|^{q/p}\right|^{p}\right)^{1/p}\\
    &=\frac{1}{\|\bfy\|_{\bfr,q}^{q}} \left(\sum_{i=1}^{n} r_{i}\left|y_{i}\right|^{q}\right)^{1/p}\\
    &= \frac{1}{\|\bfy\|_{\bfr,q}^{q}} \|\bfy\|_{\bfr,q}^{q/p}\\
    \label{eq:zetanorm}
    &= \|\bfy\|_{\bfr,q}^{q/p-q} = \|\bfy\|_{\bfr,q}^{-1}.
\end{align}
By substituting $\bfx = \bfzeta$ in Eq.~\eqref{eq:holdersinequalitytransformed}, the assumption $\langle \bfy, \bfx \rangle_{\bfr} = 1$ is satisfied and the equality holds.
Thus, we obtain
\begin{align}
     \|\bfx\|_{\bfr,p} \geq \|\bfy\|_{\bfr,q}^{-1} \iff \|\bfx\|_{\bfr,p}^{p} \geq \|\bfy\|_{\bfr,q}^{-p},
\end{align}
where the equality holds when $\bfx = \bfzeta$.
\end{proof}

We bring another lemma about spaces spanned by matrices.

\begin{lemma}[\cite{ben2003generalized} Ex.9, \S 1.3, p.43 \& \S 2.6, p.71]
\label{lemma:generalizedinverse}
For a matrix $M \in \R^{n_{1} \times n_{2}}$, we define a generalized inverse of matrix $M$ denoted by $M^{\dagger} \in \R^{n_{2} \times n_{1}}$, satisfying that
\begin{align}
    MM^{\dagger}M = M.
\end{align}
Then,
\begin{align}
    \Image(M) = \Image(MM^{\dagger}).
\end{align}
Also,
\begin{align}
    S = \{\bfy: \bfy = (I - MM^{\dagger})\bfx, \bfx \in \R^{n_{1}}\} \subseteq \Image(M)^{\bot}.
\end{align}
\end{lemma}
Note that the generalized inverse $M^{\dagger}$ is not unique.
However, the pseudoinverse $M^{+}$ is unique, and also be one of generalized inverses $M^{\dagger}$.
From this lemma, we can write as
\begin{align}
\label{eq:imageandorthogonal}
    \Image(M) = \{\bfa: \bfa = MM^{\dagger}\bfb, \bfb \in \R^{n_{1}}\}, \quad \Image(M)^{\bot} \supseteq \{\bfa: \mathbf{y} = (I-MM^{\dagger})\bfb, \bfb \in \R^{n_{1}}\}.
\end{align}

\subsubsection{Proof for Lemma~\ref{lemma:innerproductoverc}}
\label{sec:prooflemmainnerproductoverc}

This section provides a proof for Lemma~\ref{lemma:innerproductoverc}.

For the illustrative purpose, we start with the $\bfw = \mathbf{1}$ case.
If $\bfw = \mathbf{1}$, then
\begin{align}
    \langle C\bfalpha, \bfbeta \rangle_{\bfw} &= \langle C\bfalpha, \bfbeta' + \bfbeta'' \rangle_{\bfw}\\
    &= \langle C\bfalpha, \bfbeta' \rangle_{\bfw} + \langle C\bfalpha, \bfbeta''  \rangle_{\bfw}\\
    &= \langle C\bfalpha, \bfbeta' \rangle_{\bfw}.
\end{align}
The last equality follows for the following reason. 
By composition, $C\bfalpha \in \Image(C)$. Also, from Lemma~\ref{lemma:generalizedinverse}, $\bfbeta'' \in \Image(C)^{\bot}$. 
Hence, $C\bfalpha$ and $\bfbeta''$ are orthogonal to each other and we get $\langle C\bfalpha, \bfbeta''  \rangle_{\bfw} = 0$.

We now turn into the case where $\bfw$ is arbitrary.
Things are less trivial when we introduce the weight. 
Thus, we further analyze the weighted inner product.

Since the matrix $W$ is a full rank diagonal matrix, we obtain
\begin{align}
    W^{1/2}C (C^{+}W^{-1/2}) W^{1/2}C = W^{1/2}CC^{+}C = W^{1/2}C,
\end{align}
and thus $C^{+}W^{-1/2}$ is a generalized inverse of $W^{1/2}C$, i.e.,
\begin{align}
\label{eq:geninverse}
    (W^{1/2}C)^{\dagger} = C^{+}W^{-1/2}.
\end{align}

Also, since $W$ is a full rank diagonal matrix,
\begin{align}
\label{eq:spacew}
    \{\bfb: \bfb = W^{-1/2}\bfa, \bfa \in \R^{m}\} = \{\bfb' : \bfb' \in \R^{m}\} = \R^{m}.
\end{align}

Using these relations and Lemma~\ref{lemma:generalizedinverse}, we get
\begin{align}
    \Image(W^{1/2}C) 
    &= \{\bfb: \bfb = W^{1/2}C(W^{1/2}C)^{\dagger}\bfa, \bfa \in \R^{m}\}\\
    \label{eq:imusinggeninv}
    &= \{\bfb: \bfb = W^{1/2}CC^{+}W^{-1/2}\bfa, \bfa \in \R^{m}\}\\
    \label{eq:imspacew}
    &= \{\bfb: \bfb = W^{1/2}CC^{+}\bfa', \bfa' \in \R^{m}\},
\end{align}
where we use Eq.~\eqref{eq:geninverse} for Eq.~\eqref{eq:imusinggeninv} and we use Eq.~\eqref{eq:spacew} for Eq.~\eqref{eq:imspacew}.
Moreover, we have
\begin{align}
    \Image(W^{1/2}C)^{\bot} 
    &\supseteq \{\bfb: \bfb = (I - W^{1/2}C(W^{1/2}C)^{\dagger})\bfa, \bfa \in \R^{m}\}\\
    \label{eq:imbotusinggeninv}
    &= \{\bfb: \bfb = (I - W^{1/2}CC^{+}W^{-1/2})\bfa, \bfa \in \R^{m}\}\\
    &= \{\bfb: \bfb = W^{1/2} (I - CC^{+})W^{-1/2}\bfa, \bfa \in \R^{m}\}\\
    \label{eq:kernelcc}
    &= \{\bfb: \bfb = W^{1/2} (I - CC^{+})\bfa', \bfa' \in \R^{m}\},
\end{align}
where we use Eq.~\eqref{eq:geninverse} for Eq.~\eqref{eq:imbotusinggeninv} and we use Eq.~\eqref{eq:spacew} for Eq.~\eqref{eq:kernelcc}.
Therefore, 
\begin{align}
    \langle C\bfalpha,\bfbeta \rangle_{\bfw} 
    &= \langle W^{1/2}C\bfalpha, W^{1/2}\bfbeta \rangle\\
    &= \langle W^{1/2}C\bfalpha, W^{1/2}\bfbeta' \rangle + \langle W^{1/2}C\bfalpha, W^{1/2}\bfbeta'' \rangle\\
    &= \langle W^{1/2}C\bfalpha, W^{1/2}CC^{+}\bfbeta \rangle + \langle W^{1/2}C\bfalpha, W^{1/2} (I- CC^{+})\bfbeta \rangle\\
    \label{eq:thelineexplainedinnerproduct}
    &=  \langle W^{1/2}C\bfalpha, W^{1/2}CC^{+}\bfbeta \rangle.\\
    &= \langle W^{1/2}C\bfalpha, W^{1/2}\bfbeta' \rangle.\\
    \label{eq:thelastequalityforlemmainnerproduct}
    &=  \langle C\bfalpha, \bfbeta'\rangle_{\bfw}
\end{align}
The line Eq.~\eqref{eq:thelineexplainedinnerproduct} follows because from Eq.~\eqref{eq:kernelcc} the $W^{1/2} (I - CC^{+})\bfbeta \in \Image(W^{1/2}C)^{\bot}$ and therefore $W^{1/2} (I - CC^{+})\bfbeta$ is orthogonal to $W^{1/2}C\bfalpha$, which induces $\langle W^{1/2}C\bfalpha, W^{1/2} (I - CC^{+})\bfbeta \rangle = 0$.

Eq.~\eqref{eq:thelastequalityforlemmainnerproduct} concludes the proof.

\subsubsection{Proof for Lemma~\ref{lemma:pholderseqcond}}
\label{sec:pholderseqcond}
This section proves Lemma~\ref{lemma:pholderseqcond}.

 If $f_{q/p}(C\bfy) \in \Image(C)$, then 
 \begin{align}
    \label{eq:CC+fqp}
     C\bfz &= CC^{+} \frac{f_{q/p}(C\bfy)}{\|\bfy\|_{G,q}^{q}} \\
     \label{eq:plainfqp}
     &= \frac{f_{q/p}(C\bfy)}{\|\bfy\|_{G,q}^{q}}\\
     \label{eq:equalz}
     &= \frac{f_{q/p}(C\bfy)}{\|C\bfy\|_{\bfw,q}^{q}}
 \end{align}
 From Eq.~\eqref{eq:CC+fqp} to Eq.~\eqref{eq:plainfqp}, we use the following relation; for a vector $\bfa$$\in$$\Image(C)$ we have $\bfa = CC^{+}\bfa$ since $CC^{+}$ is an orthogonal projection onto the space $\Image(C)$.
Eq.~\eqref{eq:equalz} is a form of Eq.~\eqref{eq:zetadef}, and thus from Prop.~\ref{prop:newkr}, 
Eq.~\eqref{eq:equalz} satisfies the equality condition of the \holders inequality as
 \begin{align}
     \|C\bfz\|_{\bfw,p}^{p} = \|C\bfy\|_{\bfw,q}^{-p} \iff \|\bfz\|_{G,p}^{p} = \|\bfy\|_{G,q}^{-p}, 
 \end{align}
 where we use the definition of the graph $p$-seminorm.
 Thus, we obtain the claim.

\section{Proofs for Theorem~\ref{thm:tree} and Theorem~\ref{thm:presistance}}

In this section we prove Thm.~\ref{thm:tree} and Thm.~\ref{thm:presistance}.
The general strategy is applying Prop.~\ref{prop:newkroverc}.
We first prove the general case of Thm.~\ref{thm:presistance}.

\subsection{Proof for Theorem~\ref{thm:presistance}}

By definition,
\begin{align}
    r_{G,p}(i,j) = \frac{1}{\min_{\bfx} \|\bfx\|_{G,p}^{p} \ \subjectto \ x_{i} - x_{j} = 1}.
\end{align}
First, we rewrite the condition of the minimization problem.
Using Eq.~\eqref{eq:constraints}, we observe that the denominator of Eq.\eqref{eq:presistance} can be written as
\begin{align}
    \label{eq:minpcutrewritten}
    \min_{\bfx} \{\|\bfx\|_{G,p}^{p} \  \mathrm{s.t.} \  x_{i} - x_{j} = 1\} &= \min_{\bfx} \{ \|\bfx\|_{G,p}^{p} \  \mathrm{s.t.} \  \langle L^{+}\mathbf{e}_{i} - L^{+}\mathbf{e}_{j},\bfx \rangle_{L} = 1\}
\end{align}

From this rewrite, we see that Eq.~\eqref{eq:minpcutrewritten} is exactly same as the minimization problem of Prop.~\ref{prop:newkroverc} if we substitute $\bfy := L^{+}(\mathbf{e}_{i} - \mathbf{e}_{j})$. 
Thus, we apply Prop.~\ref{prop:newkroverc} to this problem in order to obtain lower and upper bounds of Eq.~\eqref{eq:minpcutrewritten}.

\paragraph{Lower Bound of Eq.~\eqref{eq:minpcutrewritten}.} 
Now, we come to the lower bound of this problem Eq.~\eqref{eq:minpcutrewritten}.
By applying the lower bound of Prop.~\ref{prop:newkroverc} with substituting $\bfy := L^{+}(\mathbf{e}_{i} - \mathbf{e}_{j})$, we obtain
\begin{align}
    \label{eq:furtherrewritedenominatorofeffectiveresistance}
    \|L^{+}(\mathbf{e}_{i} - \mathbf{e}_{j})\|_{G,q}^{-p} \leq \min_{\bfx} \{\|\bfx\|_{G,p}^{p} \ \mathrm{s.t.} \ \langle L^{+}(\mathbf{e}_{i} - \mathbf{e}_{j}), \bfx\rangle_{L} = 1 \}.
\end{align}
This conclude the lower bound.

\paragraph{Upper Bound of Eq.~\eqref{eq:minpcutrewritten}.}
Next, we turn to the upper bound of this problem Eq.~\eqref{eq:minpcutrewritten}.

We first compute
\begin{align}
    \|\bfz\|_{G,p} = \|C\mathbf{z}\|_{\bfw,p} &= \left\|\frac{CC^{+}f_{q/p}(C\bfy)}{\|C\bfy\|_{\bfw,q}^{q}}\right\|_{\bfw,p} \\
    \label{eq:CC+Cyinnorm}
    &= \frac{\left\|CC^{+}f_{q/p}(C\bfy)\right\|_{\bfw,p}}{\|C\bfy\|_{\bfw,q}^{q}}  \\
    \label{eq:CC+Cyoutnorm}
    &\leq \vvvert CC^{+}\vvvert_{\bfw,p} \frac{\|f_{q/p}(C\bfy)\|_{\bfw,p}}{\|C\bfy\|_{\bfw,q}^{q}}  \\
    \label{eq:CC+Cy}
    &= \vvvert CC^{+} \vvvert_{\bfw,p}  \|C\bfy\|_{\bfw,q}^{-1}\\
    \label{eq:WCC+WCy}
    &= \vvvert W^{1/p}CC^{+}W^{-1/p} \vvvert_{p} \|C\bfy\|_{\bfw,q}^{-1}\\
    &= \vvvert W^{1/p}CC^{+}W^{-1/p} \vvvert_{p} \|\bfy\|_{G,q}^{-1}
    \\
    \label{eq:inequalityzymatrixnorm}
    &= \alpha_{G,p}  \|\bfy\|_{G,q}^{-1},
\end{align}

where we recall that we defined as $\alpha_{G,p} := \vvvert W^{1/p}CC^{+}W^{-1/p} \vvvert_{p}$.
The transformation from Eq.~\eqref{eq:CC+Cyinnorm} to Eq.~\eqref{eq:CC+Cyoutnorm} follows from the submultiplicative characteristics of the matrix norm discussed in Sec.~\ref{sec:preliminaries}.
The equality from Eq.~\eqref{eq:CC+Cyoutnorm} to Eq.~\eqref{eq:CC+Cy} holds due to the same discussion as Eq.~\eqref{eq:zetanorm} in Prop.~\ref{prop:newkr}, which we discussed in Sec.~\ref{sec:auxiliary}.
The transformation from Eq.~\eqref{eq:CC+Cy} to Eq.~\eqref{eq:WCC+WCy} follows from a characteristics of the weighted matrix norm discussed in Eq.~\eqref{eq:weightedtounweightednorm}.
Hence, by taking the $p$-th power of the inequality Eq.~\eqref{eq:inequalityzymatrixnorm} and observing that we substitute $\bfy := L^{+}(\bfe_{i} - \bfe_{j})$, we obtain
\begin{align}
    \|\bfz\|_{G,p}^{p} \leq 
    &\alpha_{G,p}^{p}  \|L^{+}(\bfe_{i} - \bfe_{j})\|_{G,q}^{-p}
\end{align}
Thus, from Prop.~\ref{prop:newkroverc} and the inequality Eq.~\eqref{eq:inequalityzymatrixnorm} we get
\begin{align}
    \label{eq:furtherrewritedenominatorofeffectiveresistanceupperbound}
    \min_{\bfx} \{\|\bfx\|_{G,p}^{p} \ \mathrm{s.t.} \ \langle L^{+}(\mathbf{e}_{i} - \mathbf{e}_{j}), \bfx\rangle_{L} = 1 \}& \leq \|\bfz\|_{G,p}^{p} \leq \alpha_{G,p}^{p} \|L^{+}(\bfe_{i} - \bfe_{j})\|_{G,q}^{-p}.
\end{align}

\paragraph{Combining Lower and Upper Bounds of Eq.~\eqref{eq:minpcutrewritten}.}
We now combine the lower bound Eq.~\eqref{eq:furtherrewritedenominatorofeffectiveresistance} and the upper bound Eq.~\eqref{eq:furtherrewritedenominatorofeffectiveresistanceupperbound}.
By combining these two and using Eq.~\eqref{eq:minpcutrewritten}, we get
\begin{align}
\notag
    &\|L^{+}\mathbf{e}_{i}-L^{+}\mathbf{e}_{j}\|_{G,q}^{-p}\leq \min_{\bfx} \{ \|\bfx\|_{G,p}^{p}  \  \mathrm{s.t.} \ \langle L^{+}(\mathbf{e}_{i} - \mathbf{e}_{j}), \bfx\rangle_{L} = 1\} \leq \alpha_{G,p}^{p} \|L^{+}\mathbf{e}_{i}-L^{+}\mathbf{e}_{j}\|_{G,q}^{-p}
\\ &\iff    \|L^{+}\mathbf{e}_{i}-L^{+}\mathbf{e}_{j}\|_{G,q}^{-p}\leq \min_{\bfx} \{ \|\bfx\|_{G,p}^{p}  \  \mathrm{s.t.} \  x_{i} - x_{j} = 1\} \leq \alpha_{G,p}^{p} \|L^{+}\mathbf{e}_{i}-L^{+}\mathbf{e}_{j}\|_{G,q}^{-p}
\end{align}

For the $p$-effective resistance, taking the inverse we obtain
\begin{align}
     \frac{1}{\alpha_{G,p}^{p}} \|L^{+}\mathbf{e}_{i}-L^{+}\mathbf{e}_{j}\|_{G,q}^{p} \leq r_{G,p}(i,j) \leq \|L^{+}\mathbf{e}_{i}-L^{+}\mathbf{e}_{j}\|_{G,q}^{p}.
\end{align}

\subsection{Proof for Theorem~\ref{thm:tree}}
For the incidence matrix of tree, $\mathrm{rank}(C) = n-1$~\cite{bapat2010graphs}. 
Hence $\Image(C) = \R^{n-1}$.
Thus, $f_{q/p}(C\bfy) \in \R^{n-1} = \Image(C)$.
Using the Lemma~\ref{lemma:pholderseqcond} and substituting $\bfy = L^{+}\bfe_{i} - L^{+}\bfe_{j}$, 
\begin{align}
\label{eq:treeequation}
\min_{\bfx} \{ \|\bfx\|_{G,p}^{p} \subjectto \langle \bfx,L^{+}\bfe_{i} - L^{+}\bfe_{j} \rangle_{L} =1 \} = \|L^{+}\bfe_{i} - L^{+}\bfe_{j}\|_{G,q}^{-p}.
\end{align}
Recall that the minimization problem of Eq.~\eqref{eq:treeequation} is the inverse of the $p$-resistance. 
Therefore, Eq.~\eqref{eq:treeequation} leads to the claim.

\section{Proof for Proposition.~\ref{prop:generalbound}}
We recall that by definition of pseudoinverse, we have
\begin{align}
    \vvvert CC^{+} \vvvert_{2} = 1,
\end{align}
since the eigenvalues of $CC^{+}$ is either 0 or 1.
Also, for any matrix $M$ and any invertible matrix $P$, $PMP^{-1}$ and $M$ share the same eigenvalues.
By construction, $W$ is also an invertible matrix.
Thus, using Lemma~\ref{lemma:higham}, we obtain 
\begin{align}
    \alpha_{G,p} 
    &= \vvvert W^{1/p}CC^{+}W^{-1/p} \vvvert_{p} \\
    & \leq  m^{|1/2-1/p|} \vvvert W^{1/p}CC^{+}W^{-1/p} \vvvert_{2} \\
    &= m^{|1/2-1/p|}.
\end{align}

\section{Illustrative Examples of Clustering via $p$-resistance Fig.~\ref{fig:example}}
\label{sec:illustrative}

This section explains illustrative examples of clustering via $p$-resistance where $p$ plays a role.

\subsection{Preliminaries for Illustrative Examples}
\label{sec:preliminariesforillustrative}

Before we discuss the details of the clustering, we setup preliminaries.
We now setup the notions on the graph metrics.
First, a \textit{$st$-mincut} is defined as the minimum cut between the vertices $s$ and $t$, i.e.,
\begin{align}
\label{eq:defstmincut}
    \min_{V'} \mathrm{Cut}(s,t) := \min_{V'} \sum_{i\in V', j \in V'\backslash V | s \in V', t \in V'\backslash V} a_{ij}.
\end{align}
The act of the ``cut'' of the edges is defined to divide into two graphs so that the vertex $s$ belongs to one and the vertex $t$ belongs to the other.
The minimum cut is that we want such a cut so that the sum of the weight of the edges to be cut is minimized.

Now, we also define the shortest path between vertices $s$ and $t$ is defined as
\begin{align}
\label{eq:defshortestpath}
\min_{\bfi} \sum_{\ell \in E}w_{\ell}  i_{\ell} \ \subjectto \textnormal{$\bfi = (i_{\ell})_{\ell\in E}$ unit flow from $i$ to $j$},
\end{align}
where $\bfi \in \{0,1\}^{m}$.
The shortest path problem is to finding the path with smallest sum of the weights of edges between $s$ and $t$.

In the following, we show that $p$-resistance is connection with $st$-mincut and the shortest path.
We recall the theorem in~\cite{alamgir2011phase} as
\begin{proposition}[\citet{alamgir2011phase}]
Consider a $p$-flow problem as
\begin{align}
\label{eq:flowproblem}
    F_{G,p}(i,j) := \min_{\bfi}\sum_{\ell \in E} w_{\ell}^{1-p} i_{\ell}^{p} \ \subjectto \textnormal{$\bfi = (i_{\ell})_{\ell\in E}$ unit flow from $i$ to $j$},
\end{align}
where $\bfi \in \R^{+m}$ is a current at edges. 
Then, for $1/p+1/q=1$, we have
\begin{align}
\label{eq:resistanceandflow}
    r_{G,p}^{1/(p-1)}(i,j) = F_{G,q}(i,j).
\end{align}
\end{proposition}

We first remark that $\bfi$ in $q$-flow problem is non-negative real value whereas $\bfi$ for the shortest path is either 0 or 1.
We remark that when $p\to\infty$, $q$ goes to 1 and $q$-flow problem is a simple shortest path flow problem.

This proposition means that the $1/(p-1)$-th power of $p$-resistance is equivalent to the $q$-flow.
From this proposition, we now see the connection between $p$-resistance, and $st$-mincut and shortest path as follows.
\begin{itemize}
    \item When $p \to 1$, $p$-resistance between $s$ and $t$ is 1/$st$-mincut.
    \item When $p \to \infty$,  $1/(p-1)$-th power of the $p$-resistance is the discrete shortest path of the \textit{unweighted} graph.
\end{itemize}
Thus, we intuitively characterize the $p$-resistance as
\begin{itemize}
    \item When $p$ is small, $p$-resistance more focus on a minimum cut.
    \item When $p$ is large, $p$-resistance more focus on the ``path'', and also more focus on the ``unweighted topology''.
\end{itemize}

We next formulate the clustering problem as follows.
We use the $k$-center algorithm using $p$-resistance as a metric as
\begin{align}
\label{eq:kcenter1/p-1}
    C_{G,p}^{*} := \min_{v_{1}^{*},v_{2}^{*} \in V}\max_{v \in V} \min_{i\in\{1,2\}} r_{G,p}^{1/(p-1)}(v,v_{i}^{*}),
\end{align}
where $\{v_{1}^{*},v_{2}^{*}\}$ is a minimizer.
Since when $p \to 1$ and $r_{G,p}^{1/(p-1)}>0$, then $r_{G,p}^{1/(p-1)} \to \infty$ and therefore Eq.~\eqref{eq:kcenter1/p-1} cannot be used.
In this case, we note that the following relation that is 
\begin{align}
    x < y \iff x^{1/(p-1)} < y^{1/(p-1)}
\end{align}
we have
\begin{align}
\label{eq:kcenter}
    C_{G,p}^{*p-1} := \min_{v_{1}^{*},v_{2}^{*} \in V}\max_{v \in V} \min_{i\in\{1,2\}} r_{G,p}(v,v_{i}^{*}).
\end{align}
Thus, we simply use the comparison of $r_{G,p}$ instead of $r_{G,p}^{(1/(p-1))}$ when $p \to 1$.
We finally remark that ~\citet{herbster2010triangle} showed that when $p \to 1$ the triangle inequality still holds, i.e., 
\begin{align}
    r_{G,p\to1}(i,j) \leq r_{G,p\to1}(i,\ell) + r_{G,p\to1}(\ell,j).
\end{align}

\subsection{Illustrative Examples of Clustering via $p$-Resistance}

Now, we discuss the examples in Fig.~\ref{fig:example}.
We give notations as in Fig.~\ref{fig:examplegraphs}.
We denote by $(V_{ij}, E_{ij})$ the vertices and edges of the graph $G_{ij}$.
We also give the example where the weight matters and its notation in Fig.~\ref{fig:examplegraphsline}.

\begin{figure*}[t]
\begin{center}
\includegraphics[width=0.6\hsize,clip]{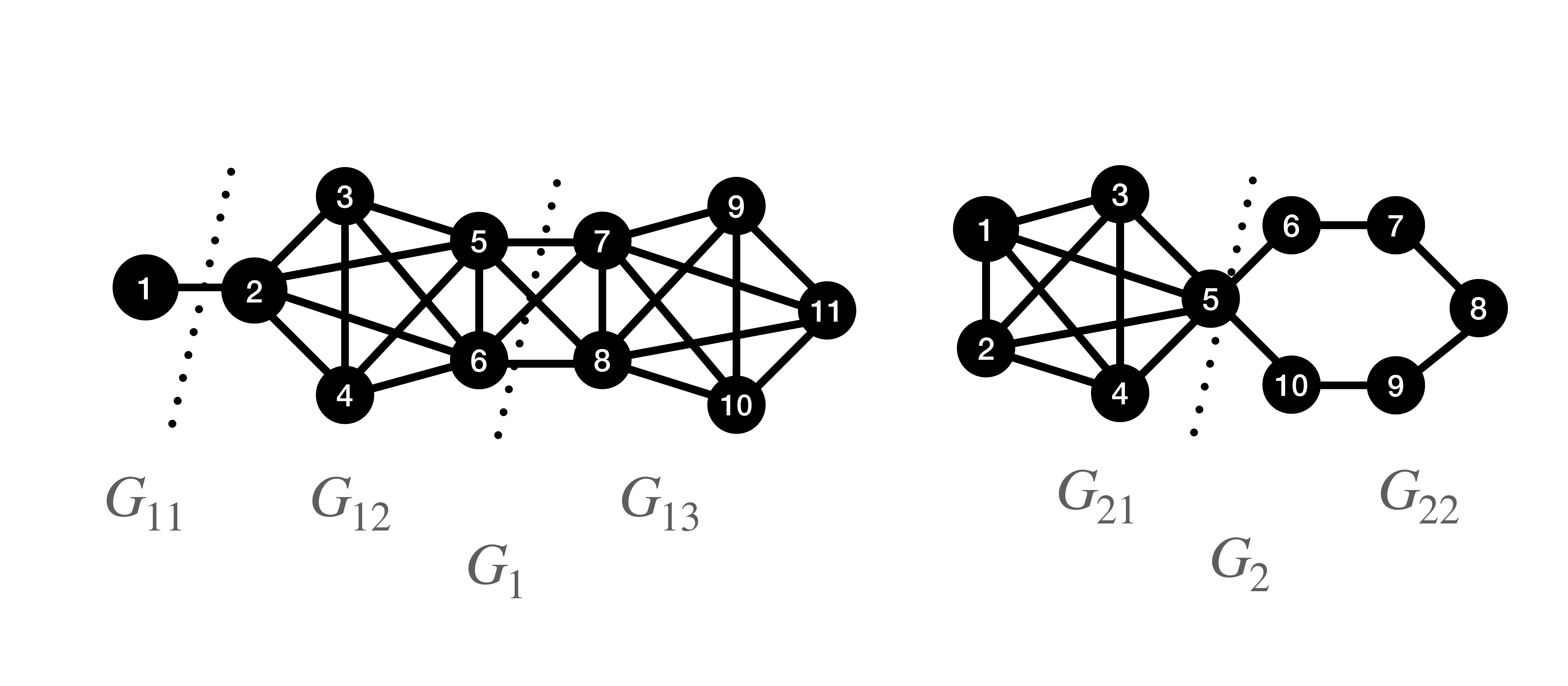}
\caption{
The notations of illustrative example graphs. In the graph $G_{2}$ the vertex 5 is in both $G_{21}$ and $G_{22}$.
}
\label{fig:examplegraphs}
\end{center} 
\vspace{-0.2in}
\end{figure*}

\begin{figure*}[t]
\begin{center}
\includegraphics[width=0.7\hsize,clip]{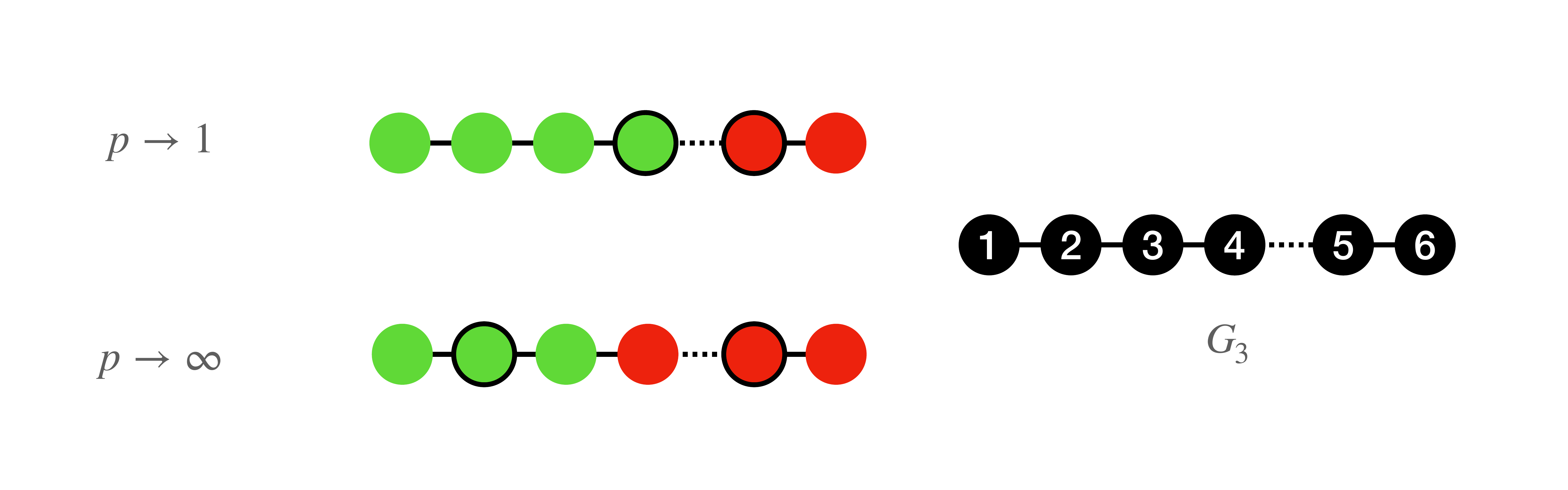}
\caption{
The illustrative example of a weighted graph and its notations. 
The weights of edge drawn in the line are 1, whereas weight of the dotted line is $\epsilon \ll 1$.
The other drawing rule follows Fig.~\ref{fig:example}.
In the example, we observe that we focus on the difference of the weight when $p\to 1$, while we ignore the weight when $p \to \infty$.
For this example, ``more natural result'' depends on the perspective.
If we look at the cut, the more natural result is obtained when $p \to 1$.
If we look at the path-based topology, we obtain the natural result when $p \to \infty$.
Details in Appendix~\ref{sec:illustrative}.
}
\label{fig:examplegraphsline}
\end{center}
\end{figure*}

\subsubsection{The Case of $G_{1}$}

For the case of $p \to 1$, since $p$-resistance is the 1 over min-cut, we have for $j>i$
\begin{align}
    r_{G,p}(i,j) = 
    \left\{
    \begin{array}{cc}
        1 & i=1 \text{ and } j \in V\backslash \{1\}\\
        1/5 & i,j \in V_{12} \text{ or } i,j \in V_{13} \\
        1/4 & i \in \{5,6\},j \in \{7,8\}
    \end{array}
    \right.
\end{align}
Note that $r_{G,p}(i,j) = r_{G,p}(j,i)$.
By using this $p$-resistance, the set satisfying Eq.~\eqref{eq:kcenter} is $v_{1}^{*} = 1$ and $v_{2}^{*} \in V_{12} \cup V_{13}$.
This is because if we do not take $v_{1}^{*} = 1$, 
\begin{align}
    \min_{v_{1}^{*},v_{2}^{*} \in V}\max_{v \in V} \min_{i\in\{1,2\}} r_{G,p}(v,v_{i}^{*}) = 1,
\end{align}
which is the maximum of the weight of edges of $G$.

For $p \to \infty$, since $p$-resistance is a shortest path, we have for $j>i$
\begin{align}
    r_{G,p}^{1/(p-1)}(i,j) = 
    \left\{
    \begin{array}{cl}
        1 & i=1,j=2\\
        2& i=1, j\in \{3,4,5,6\}\\
        3& i=1, j\in \{7,8\}\\
        4& i=1, j\in \{9,10,11\}\\
        1& i,j \in V_{12} \text{ or } i,j \in V_{13} \\
        2& i \in V_{12}, j\in V_{13}.
    \end{array}
    \right.
\end{align}
Then if we set $v_{1}^{*} = 2$ and $v_{2}^{*} \in V_{13}$, we have
\begin{align}
    \min_{v_{1}^{*},v_{2}^{*} \in V}\max_{v \in V} \min_{i\in\{1,2\}} r_{G,p}^{1/(p-1)}(v,v_{i}^{*}) = 1.
\end{align}
Since this is the minimum of the weight of the edge, it is clear that this set is optimal.

Coloring the vertices in the same color if the vertices are closer to the same center than the others, we obtain Fig.~\ref{fig:example}.

\subsubsection{The Case of $G_{2}$}
For the case of $p \to 1$, we have for $j>i$
\begin{align}
    r_{G,p}(i,j) = 
    \left\{
    \begin{array}{cl}
        1/5 & i,j \in V_{21}\\
        1/2 & i \in V_{21},j \in V_{22} \\
        1/2 & i,j \in V_{22}.
    \end{array}
    \right.
\end{align}
Then if we set $v_{1}^{*} \in V_{21}$ and $v_{2}^{*} \in V_{22}$, we have
\begin{align}
    \min_{v_{1}^{*},v_{2}^{*} \in V}\max_{v \in V} \min_{i\in\{1,2\}} r_{G,p}^{1/(p-1)}(v,v_{i}^{*}) = 1/2.
\end{align}
Since $\min_{i \in V_{22},}r_{G,p}(i,j)=1/2$, this is the best possible minimum. 

For the case of $p \to \infty$, we have for $j>i$
\begin{align}
    r_{G,p}(i,j)^{1/(p-1)} = 
    \left\{
    \begin{array}{cl}
        1 & i,j \in V_{21}\\
        2 & i \in V_{21}\backslash\{5\},j \in \{6,10\} \\
        3 & i \in V_{21}\backslash\{5\},j \in \{7,9\} \\
        4 & i \in V_{21}\backslash\{5\},j=8\\
        \min \{j-i, 6-(j-i) \} & i,j \in V_{22}
    \end{array}
    \right.
\end{align}
Then if we set $v_{1}^{*} = 5$ and $v_{2}^{*} = 8$, we have
\begin{align}
    \min_{v_{1}^{*},v_{2}^{*} \in V}\max_{v \in V} \min_{i\in\{1,2\}} r_{G,p}^{1/(p-1)}(v,v_{i}^{*}) = 1.
\end{align}
Since the minimum of $p$-resistance is 1, this is the best possible minimum.

Coloring the vertices in the same color if the vertices are closer to the same center than the others, we obtain Fig.~\ref{fig:example}.

\subsubsection{The Case of $G_{3}$}
For the case of $p \to 1$, we have for $j>i$
\begin{align}
    r_{G,p}(i,j) = 
    \left\{
    \begin{array}{cl}
        1 & i,j \in \{1,\ldots,4\}\\
        1 & i,j \in \{5,6\}\\
        1/\epsilon & i \in  \{1,\ldots,4\},j \in \{5,6\}.
    \end{array}
    \right.
\end{align}
Then if we set $v_{1}^{*} \in \{1,\ldots,4\}$ and $v_{2}^{*}\in \{5,6\}$, we have
\begin{align}
    \min_{v_{1}^{*},v_{2}^{*} \in V}\max_{v \in V} \min_{i\in\{1,2\}} r_{G,p}^{1/(p-1)}(v,v_{i}^{*}) = 1.
\end{align}
Since the minimum of $p$-resistance is 1, this is the best possible minimum. 

For the case of $p \to \infty$, we have for $j>i$
\begin{align}
    r_{G,p}^{1/(p-1)}(i,j) = j-i \text{ if $j>i$}
\end{align}
Then if we set $v_{1}^{*} =2$ and $v_{2}^{*}=5$, we have
\begin{align}
    \min_{v_{1}^{*},v_{2}^{*} \in V}\max_{v \in V} \min_{i\in\{1,2\}} r_{G,p}^{1/(p-1)}(v,v_{i}^{*}) = 1.
\end{align}
Since the minimum of $p$-resistance is 1, this is the best possible minimum. 

Coloring the vertices in the same color if the vertices are closer to the same center than the others, we obtain Fig.~\ref{fig:example}.

\section{On Theorem~\ref{thm:alamgirp}}

This section discusses Thm.~\ref{thm:alamgirp}, including proof and some existing claim on Thm.~\ref{thm:alamgirp}.

\subsection{Proof for Theorem~\ref{thm:alamgirp}}

We use the following characteristics of $p$-Laplacian, defined as Eq.~\eqref{eq:defplaplacian}.
\begin{proposition}[\cite{pgraph}]
\label{prop:laplacianproperties}
\begin{align} 
S_{G,p}(\bfx) & = \langle \bfx, \Delta_{p}\bfx \rangle_{\mathcal{H}(V)},   \label{penergy} \\
\left(\frac{\partial S_{p}(\bfx)}{\partial \bfx}\right)_{i}  & = p (\Delta_{p} \bfx)_{i}.\label{propertylaplacian} 
\end{align}
\end{proposition}

Before we prove the main argument, we now explore a matrix expression of the $p$-Laplacian $\Delta_{p}$.
We define a matrix $A_{p,\bfx}$ as
\begin{align}
    A_{p,\bfx}(i,j) := a_{ij}|x_{i} - x_{i}|^{p-2},
\end{align}
and its degree-like matrix $D_{p,\bfx}$ as
\begin{align}
    D_{p,\bfx}(i,j) = \left\{
    \begin{array}{ll}
        \sum_{j=1}^{n} A_{p,\bfx}(i,j) & \mathrm{if}\  l=i\\
        0 & \mathrm{if}\  l \neq i\\
    \end{array}
    \right.
\end{align}
Define the matrix $L_{p,\bfx}$ as
\begin{align}
    L_{p,\bfx} := D_{p,\bfx} - A_{p,\bfx}.
\end{align}
We note that these matrix expressions depend on $\bfx$ except the $p=2$ case.
Now,
\begin{align}
    (L_{p,\bfx} \bfx)_{i}
    &= D_{p,\bfx}(i,i) x_{i} - \sum_{j=1}^{n} A_{p,\bfx}(i,j)x_{j}\\
    &= \sum_{j=1}^{n} A_{p,\bfx}(i,j) x_{i} - \sum_{j=1}^{n} A_{p,\bfx}(i,j)x_{j}\\
    &= \sum_{j=1}^{n} A_{p,\bfx}(i,j) (x_{i} - x_{j})\\
    &= \sum_{j=1}^{n} a_{ij}|x_{i} - x_{i}|^{p-1} \sgn(x_{i} - x_{j})\\
    &= (\Delta_{p} \bfx)_{i}.
\end{align}
Thus, we can say that $L_{p,\bfx}$ is a matrix expression of the $p$-Laplacian, satisfying
\begin{align}
    L_{p,\bfx} \bfx = \Delta_{p}\bfx.
\end{align}
Then, by Prop.~\ref{prop:laplacianproperties}, we can prove that
\begin{align}
\bfx^{\top}L_{p,\bfx} \bfx = S_{G,p} (\bfx).
\end{align}
Now we turn to the optimization problem Eq.~\eqref{eq:sslp}.
By using the Lagrangian multiplier method, the optimal solution satisfies the following:
\begin{align}
    \label{eq:lagfunction}
    F(\bfx, \lambda) &:= (S_{G,p} (\bfx)) - \lambda (x_{i} - x_{j} -1)\\
    \label{eq:lagconditionpsi}
    \frac{\partial F}{\partial \bfx} &= pL_{p,\bfx} \bfx - \lambda (\bfe_{i} - \bfe_{j}) = 0\\
    \label{eq:lagconditionlambda}
    \frac{\partial F}{\partial \lambda} &= x_{i} - x_{j} -1 = 0.
\end{align}
From Eq.~\eqref{eq:lagconditionpsi}, we have
\begin{align}
    \label{eq:psiopt}
    \bfx^{\mathrm{*}_{ij}} = \frac{\lambda}{p} L_{p,\bfx^{\mathrm{*}_{ij}}}^{+} (\bfe_{i} - \bfe_{j}).
\end{align}
From Eq.~\eqref{eq:psiopt} and Eq.~\eqref{eq:lagconditionlambda}, we have
\begin{align}
    \label{eq:lambdawrittendown}
    \frac{\lambda}{p} \left((L_{p,\bfx^{\mathrm{*}_{ij}}}^{+}(i,i) - L_{p,\bfx^{\mathrm{*}_{ij}}}^{+}(i,j)) - (L_{p,\bfx^{\mathrm{*}_{ij}}}^{+}(j,i) - L_{p,\bfx^{\mathrm{*}_{ij}}}^{+}(j,j)) \right) = 1.
\end{align}
Following Eq.~\eqref{eq:psiopt}, we substitute $\lambda/p$ from Eq.~\eqref{eq:lambdawrittendown} into Eq.~\eqref{eq:psiopt}, and we have 
\begin{align}
\label{eq:optp}
    \bfx^{\mathrm{*}_{ij}} = \frac{L_{p,\bfx^\mathrm{*}_{ij}}^{+}}{L_{p,\bfx^{\mathrm{*}_{ij}}}^{+}(i,i) + L_{p,\bfx^{\mathrm{*}_{ij}}}^{+}(j,j) - 2L_{p,\bfx^{\mathrm{*}_{ij}}}^{+}(i,j)}(\bfe_{i} - \bfe_{j}).
\end{align}
Since $p$-resistance is an inverse of the energy, we obtain
\begin{align}
    r_{G,p} (i,j) &= (\bfx^{\mathrm{*}_{ij}\top} L_{p,\bfx^{\mathrm{*}_{ij}}}^{+} \bfx^{\mathrm{*}_{ij}})^{-1}\\
    &= L_{p,\bfx^{\mathrm{*}_{ij}}}^{+}(i,i) + L_{p,\bfx^{\mathrm{*}_{ij}}}^{+}(j,j) - 2L_{p,\bfx_{\mathrm{*}_{ij}}}^{+}(i,j) \\
    &= (\bfe_{i} - \bfe_{j})^{\top}L_{p,\bfx^{\mathrm{*}_{ij}}}^{+}(\bfe_{i} - \bfe_{j})
\end{align}

The rest of the proof is same as the original proof in Thm.~6 in~\citep{alamgir2011phase}.
The trick is that we do not have to the exact form of $L_{p,\bfx^{\mathrm{*}_{ij}}}$.
Only this expression is enough to prove Theorem~\ref{thm:alamgirp}.

\subsection{Original Context of Theorem~\ref{thm:alamgirp}}
\label{sec:original}
Originally in Sec.~5~\cite{alamgir2011phase}, Thm.~\ref{thm:alamgirp} when $p=2$ has a different interpretation.
~\citet{nadler2009semi} proves that the semi-supervised learning problem of $p=2$ case is meaningless if the number of vertices are infinite. 
Thm.~\ref{thm:alamgirp} for $p=2$ supports this claim in~\cite{nadler2009semi} for two-pole semi-supervised leaning problem for the following way.
Since the equivalent 2-resistance is known to converge to a meaningless function,
the solution of the semi-supervised problem is equivalently characterized by this meaningless function. 
Thus, the semi-supervised learning does not make sense, if the number of the vertices are large.
If the conjecture for $p>1$ case were proven, Thm.~\ref{thm:alamgirp} can be interpreted that for some range of $p>1$ two-pole semi-supervised learning problem is not meaningless, since the equivalent $p$-resistance is shown not to converge to a meaningless one. 
Later year, independent of $p$-resistance, the statement ``for some range of $p>1$ semi-supervised learning problem is not meaningles'' is proven by~\cite{slepcev2019analysis}. 
Thm.~\ref{thm:alamgirp} now supports~\cite{slepcev2019analysis} from a $p$-resistance view.

\subsection{Remark on the Existing Claims on Theorem~\ref{thm:alamgirp}}

Finally, we discuss several existing claims on this theorem.
First, we need to mention a small \textit{fixable} mistake in the original proof in~\cite{alamgir2011phase} for the $p=2$ case.
The original proof assumes that the solution to the semi-supervised learning Eq.~\eqref{eq:sslp} when $p=2$ is that
\begin{align}
    \bfx^{\mathrm{*}_{ij}} = L^{+} (\bfe_{i} - \bfe_{j}).
\end{align}
However, this is not true since this does not satisfy the constraint
\begin{align}
    x^{\mathrm{*}_{ij}}_{i} - x^{\mathrm{*}_{ij}}_{j} = (\bfe_{i} - \bfe_{j})^{\top} L^{+} (\bfe_{i} - \bfe_{j}) \neq 1.
\end{align}
Instead, the solution is given as
\begin{align}
    \bfx^{\mathrm{*}_{ij}} = \frac{L^{+} (\bfe_{i} - \bfe_{j})}{(\bfe_{i} - \bfe_{j})^{\top} L^{+} (\bfe_{i} - \bfe_{j})}.
\end{align}
Note that this corresponds to Eq.~\eqref{eq:optp}.
However, this does not affected the rest of the proof, since the proof exploits only $\bfx^{\mathrm{*}_{ij}} = \rho L^{+} (\bfe_{i} - \bfe_{j})$ for $\rho \in \R$, and $\rho$ does not matter.
Thus, the validity of the original claim still remains.

\begin{figure*}[t]
\begin{center}
\includegraphics[width=.2\hsize,clip]{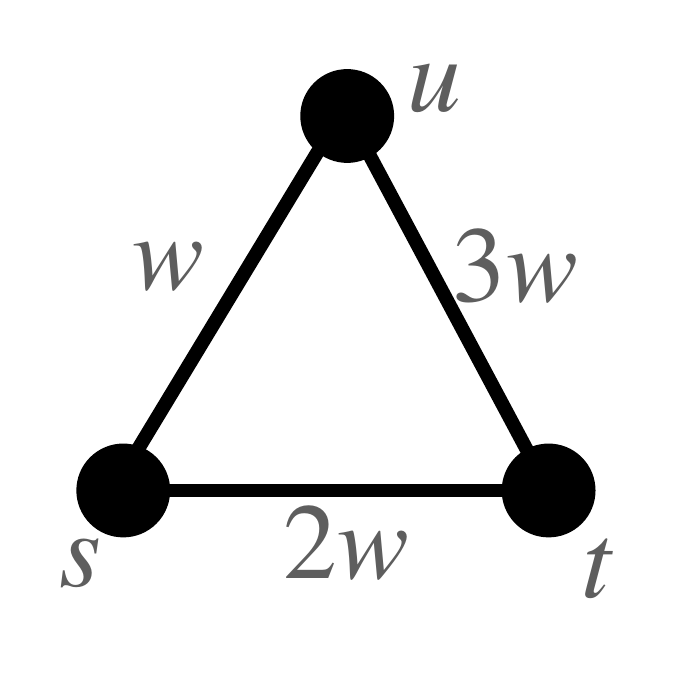}
\caption{The example discussed in~\cite{bridle2013p}.}
\label{fig:counterexample}
\end{center}
\end{figure*}

Next, since Thm.~\ref{thm:alamgirp} resolves the open problem in~\cite{alamgir2011phase}, there is an existing discussion on if this statement is true or not. 
The work~\cite{bridle2013p} claims that there is a counterexample to Thm.~\ref{thm:alamgirp} in the general $p$ case. 
In the following, we argue that the discussion on the example in~\cite{bridle2013p} does not work as a counterexample.

The ``counterexample'' given in~\cite{bridle2013p} is based on the example shown as Fig.~\ref{fig:counterexample}.
However, unfortunately, we believe that there is invalidity in the discussion on this example.
Firstly, we recall that
\begin{align}
\label{eq:notequal1}
    \min_{\bfx}\{S_{G,p}(\bfx)\ \subjectto\  x_{s} -x_{t} = 1\}  \neq \min_{\bfx} S_{G,p}(\bfx),
\end{align}
since $\min_{\bfx} S_{G,p}(\bfx) = 0$ when $\bfx = c\mathbf{1}$, $\forall c \in \R$, while $c\mathbf{1}$ does not satisfy the constraint of the left hand side.
However, the work~\cite{bridle2013p} assumes the equality of Eq.~\eqref{eq:notequal1}, see the the first equality at the top of the left column in p.3 of~\cite{bridle2013p}.
Moreover, we note that
\begin{align}
    \frac{\partial S_{G,p}(\bfx)}{\partial x_{u}} 
    &= \frac{\partial}{\partial x_{u}} \left( w|x_{s} - x_{u}|^{p} +  2w|x_{u} - x_{v}|^{p} + 3w|x_{s} - x_{t}|^{p}\right)\\
\label{eq:righteqaul}
    &= p \left( w|x_{s} - x_{u}|^{p-1} +  2w|x_{u} - x_{v}|^{p-1} \right),
\end{align}
and therefore
\begin{align}
\label{eq:notequal2}
    \frac{\partial S_{G,p}(\bfx)}{\partial x_{u}} \neq p \left( - w|x_{s} - x_{u}|^{p-1} +  2w|x_{u} - x_{v}|^{p-1} \right),
\end{align}
where the difference between Eq.~\eqref{eq:righteqaul} and Eq.~\eqref{eq:notequal2} is the sign of the term $w|x_{s} - x_{u}|^{p-1}$.
\footnote{
There is a slight difference between the definition of the $p$-resistance between ours and ~\cite{bridle2013p}. 
We follow the definition of~\cite{herbster2009predicting} and~\cite{bridle2013p} follows the definition of~\cite{alamgir2011phase}. 
However, these two have almost same properties. 
Moreover, while we write the equations in our form, this difference does not affect the discussion here. 
For more details of the difference, see \S 6.1 in~\cite{alamgir2011phase} or Sec.~\ref{sec:detailsof}.}
However, the work~\cite{bridle2013p} assumes the equality of Eq.~\eqref{eq:notequal2}, see the the third equality at the top of the left column in p.3 of~\cite{bridle2013p}.
In~\cite{bridle2013p}, these invalid equality assumptions of Eq.~\eqref{eq:notequal1} and Eq.~\eqref{eq:notequal2} derive the fundamental relationship in order to bring a counterexample.
The rest of the analysis in~\cite{bridle2013p} is carried with this relationship.
Due to this invalidity, we believe that there are serious flaws in the claim that the example Fig.~\ref{fig:counterexample} leads to a counterexample to Thm.~\ref{thm:alamgirp}.
Hence, we claim that Thm.~\ref{thm:alamgirp} holds with the proof in this section.

\section{Details of the Preliminary Experiments}
\label{sec:detailsof}

\begin{table}[t]
\centering
\caption{Dataset Summary. Since Hopkins 155 contains 155 different videos, we report the sum of the data points and sum of the dimensions of videos. Also, Hopkins 155 dataset contains 120 2-class datasets and 35 3-class datasets.}
\begin{tabular}{c|ccccc}
\toprule
            & ionosphere & hop 155 2cls & iris & wine & hop 155 3cls \\
\midrule
\# of class & 2         & 2            & 3    & 3    & 3            \\
size        & 351       & 31981        & 150  & 178  & 13983        \\
dimension   & 34        & 3542         & 4    & 13   & 999         
\end{tabular}
\label{tab:datasetsummary}
\end{table}

\begin{table}[t]
\centering
\caption{Computational Time for the Experiment (unit:sec).
Here we use E notation, e.g., E-6$=10^{-6}$ or E1 =$10^{1}$.
For methods we use the same abbriviation as Table~\ref{tab:res}.
Since ``Rec-bi $p$'' is a deterministic method, we only report time. 
Also, since Hop contains multiple datasets, we only show the average. }
\label{tab:comptime}
\begin{tabular}{c|c|cc|ccc}
\toprule
     &                                                              & \multicolumn{2}{c|}{2 clsss}            & \multicolumn{3}{c}{multi-class}                                \\
Type & Method                                                       & ionosphere              & Hop 2~cls & iris                  & wine                  & Hop 3~cls \\ \midrule
ER   & $k$-med (a) $p$                                              & 1.37E1 $\pm$ 0.01E1    & 1.21E1         & 4.93E0 $\pm$ 0.01E0   & 1.30E-1 $\pm$ 0.13E-1 & 1.78E1         \\
ER   & $k$-med $p=2$                                                & 3.76E0 $\pm$ 0.01E0    & 1.01E1         & 1.36E0 $\pm$ 0.00E0   & 9.07E-2 $\pm$ 1.05E-2 & 1.34E1         \\
ER   & FF (a) $p$                                                   & 9.71E-2 $\pm$ 0.32E-1  & 4.88E-1        & 5.45E-1 $\pm$ 0.33E-1 & 4.72E-2 $\pm$ 0.18E-2 & 8.01E0         \\
ER   & FF $p=2$  & 8.71E-2 $\pm$ 0.12E-1  & 3.56E-1        & 4.21E-1 $\pm$ 0.12E-1 & 3.82E-2 $\pm$ 0.06E-2 & 6.45E0         \\
ER   & $p$-Flow  & 1.46E1 $\pm$ 0.01E1    & 1.33E1         & 5.21 E0 $\pm$ 0.02E0  & 1.60E-1 $\pm$ 0.15E-1 & 1.81E1         \\
ER   & ECT        & 1.01E-1 $\pm$  0.10E-1 & 8.12E-1        & 2.03E-2 $\pm$ 0.42E-2 & 2.05E-2 $\pm$ 0.49E-2 & 9.26E-1        \\
SC   & Rec-bi $p$             & 1.18E-1                & 8.11E-1        & 5.35E-1               & 9.07E-2               & 1.01E0         \\
SC   & SC $p$-orth & 8.60E-2 $\pm$ 0.01E-2  & 6.31E-1        & 4.78E-1 $\pm$ 1.65E-1 & 3.02E-2 $\pm$ 0.57E-2 & 8.10E-1        \\
SC   & SC $p=2$                                                     & 1.60E-2 $\pm$ 0.01E-2  & 3.43E-2        & 1.28E-1 $\pm$ 0.00E-1 & 8.78E-3 $\pm$ 0.00E-3 & 6.12E-2       
\end{tabular}
\end{table}

\begin{figure*}[t]
\begin{center}
\subfigure[ion]{%
\includegraphics[width=.32\hsize,clip]{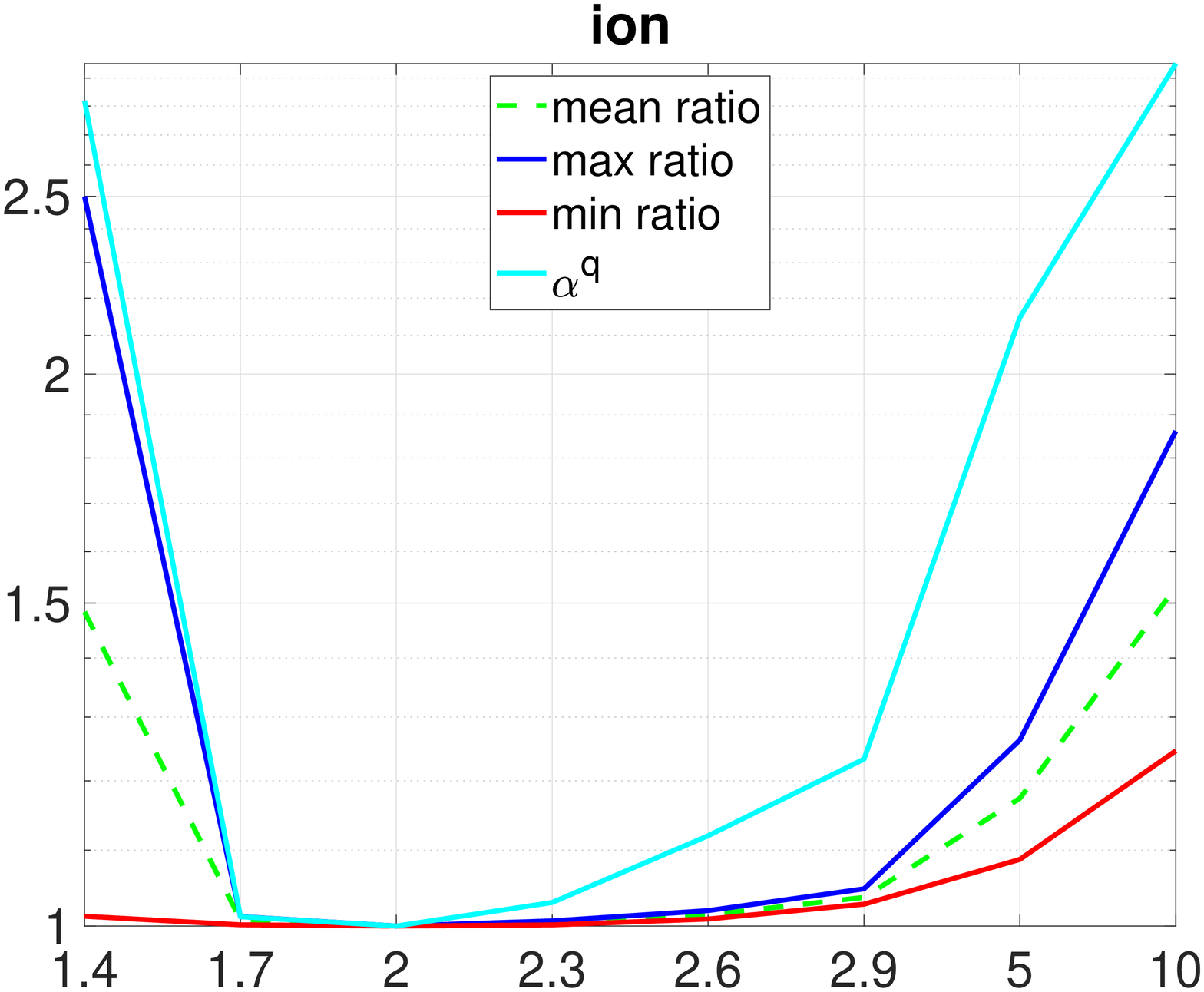}\label{fig:ioncomp}}
~\subfigure[iris]{%
\includegraphics[width=.32\hsize,clip]{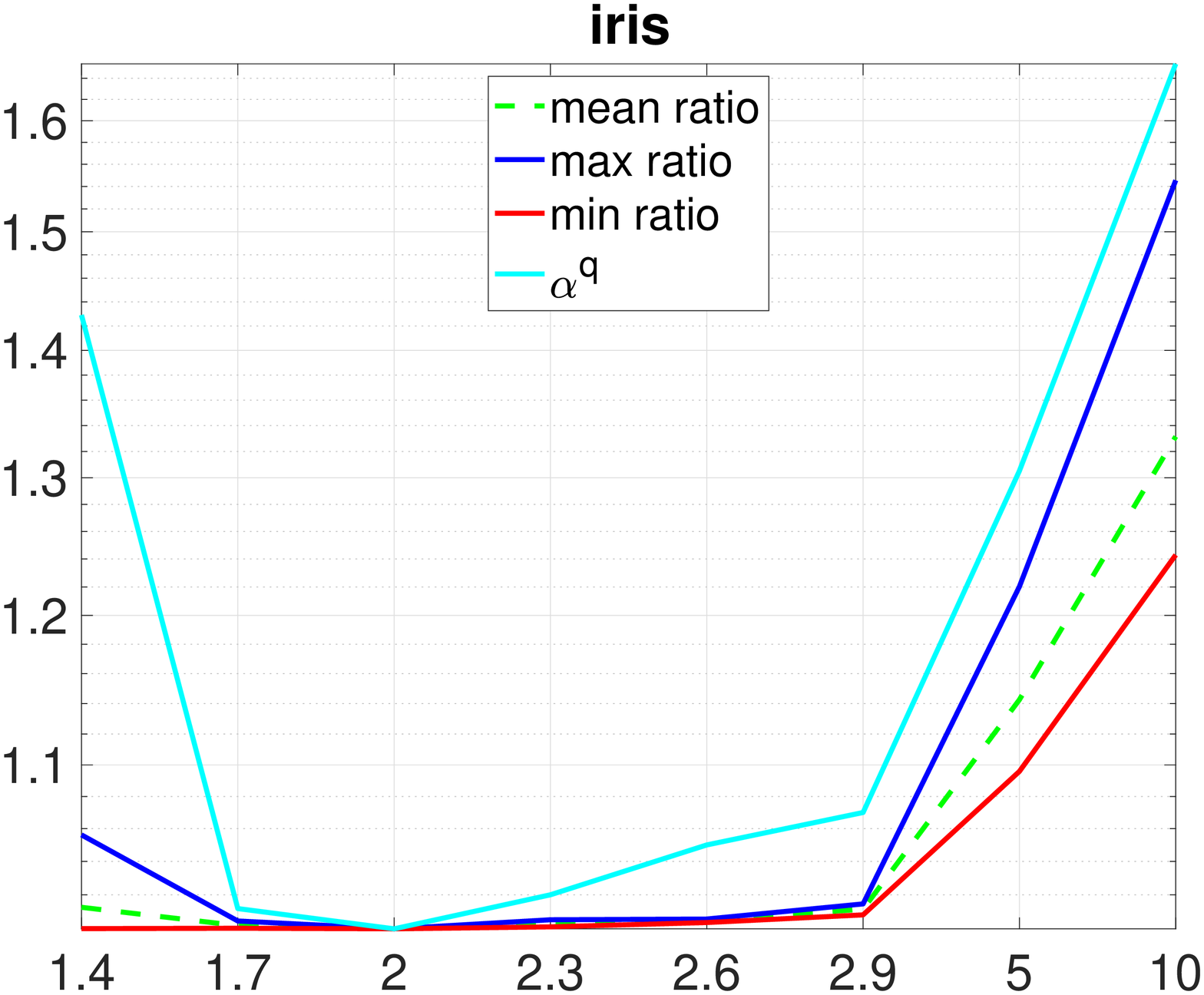}\label{fig:iriscomp}}
~\subfigure[wine]{%
\includegraphics[width=.32\hsize,clip]{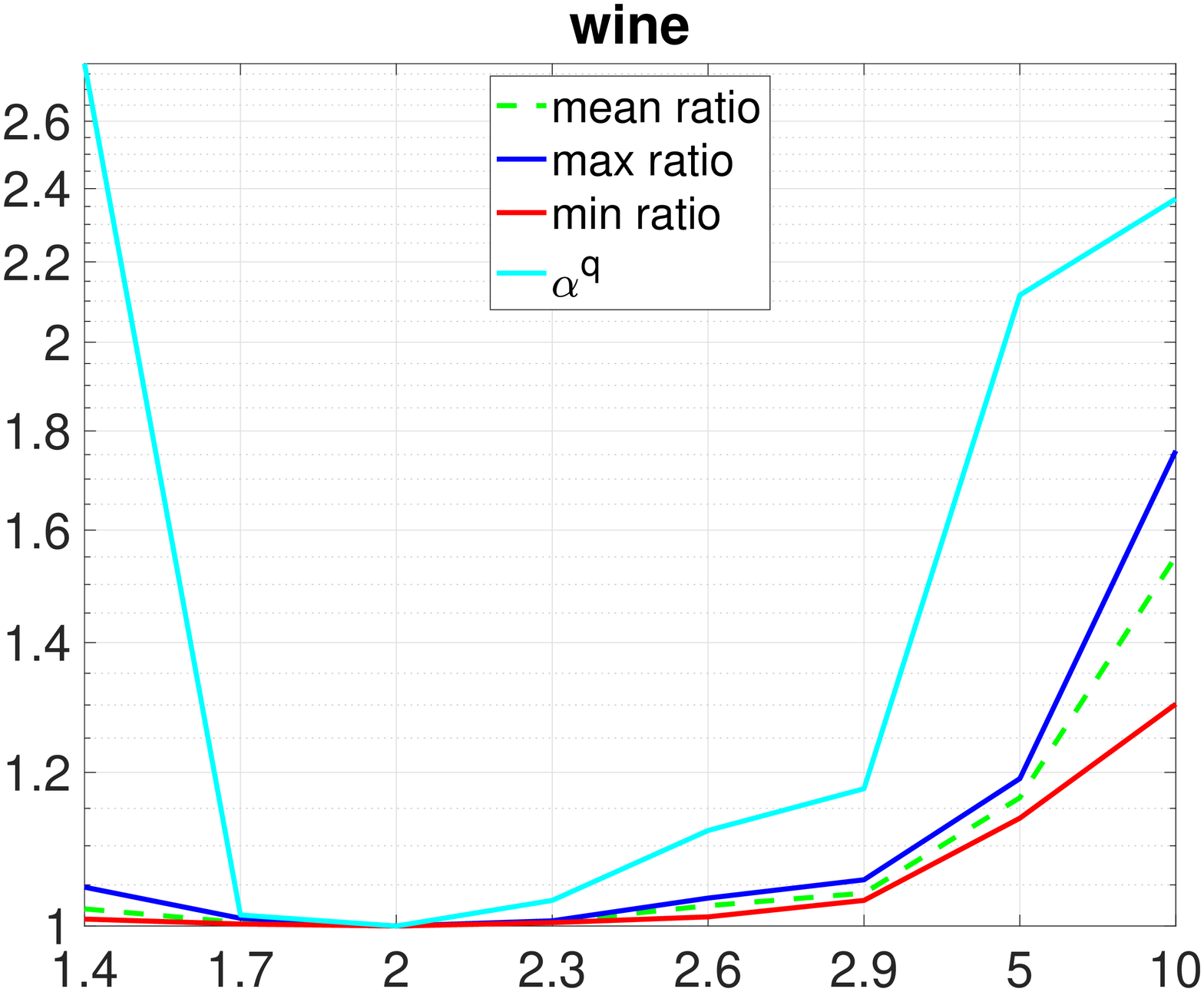}\label{fig:winecomp}}
\caption{
The ratio of the approximated value of $p$-resistance to the exact $p$-resistance, i.e., 
$\|L^{+}\bfe_{i}-L^{+}\bfe_{j}\|_{G,q}^{q}/r_{G,p}^{1/(p-1)}(i,j)$. 
Also, the factor of the bound $\alpha_{G,p}^q$.
}
\label{fig:compratio}
\end{center}
\end{figure*}

This section discusses missing details of preliminary experiments and additional experimental results.

\subsection{Details of Preliminary Experiments}

To begin with, we further explain the details of the experimental setting.
We report the size of the dataset in Table.~\ref{tab:datasetsummary}.
Our experiment was conducted on Mac Studio with M1 Max Processor and 32GiB RAM. 
Also, we use an Intel binary Matlab translated by Rosetta, which is a standard use in MacOS with Apple Silicon environment.
We need to note that for the comparison method Rec-bi~\cite{pgraph}, originally, the algorithm is defined for $p\leq 2$.
Thus, we apply the same technique for $p\leq3$.
In~\cite{pgraph}, in order to avoid the conversion to too close or too far local optimum, at the step $t$~\cite{pgraph} minimizes Eq.~\eqref{eq:rayleigh_2} via the gradient descent method using the initial condition as the obtained eigenvector of the previous step $p_{t} = 0.9p_{t-1}$.
If we increase the $p$, we use the same technique; $p_{t} = p_{t-1}/0.9$ until $p$ reaches 5. Beyond 5, we use $p_{t} = 2p_{t-1}$.
When $p=2$, we know that 2-resistance is further computed as
\begin{align}
    r_{G,2}(i,j) = \|L^{+}\bfe_{i} - L^{+}\bfe_{j}\|_{G,2}^{2} = \|(L^{+})^{1/2}\bfe_{i} - {(L^{+})^{1/2}}\bfe_{j}\|^{2}.
\end{align}
The graph $p$-seminorm is the size $m$ norm, while the latter is based on the size $n$ norm.
However, we use graph 2-seminorm even for the $p=2$ case.
The reason is that we needed to be consistent in the experiment since we observed numerical round-off errors in the other methods.
On the other hand, ECT~\cite{yen2005clustering} uses the square of 2-resistance. 
For this method, we use $\|(L^{+})^{1/2}\bfe_{i} - {(L^{+})^{1/2}}\bfe_{j}\|^{2}$ since the original paper~\cite{yen2005clustering} uses this.

Next, we discuss the computational time for the experiment. 
Due to its significant computational time, we parallelized the distance computation for the exact method, while we did not use such a technique for the others. 
Thus, we first compare the approximation of the $p$-reistance and the exact $p$-resistance.
Next, we compare the computational time among the methods except for the exact methods.

\subsection{Computational Times of Experiments}
\label{sec:computational}
In this section we discuss computational times of the experiment. 

In Table~\ref{tab:comptimeforapproxvsexact}, we compare the approximation by Eq.~\eqref{eq:approximation1/p-1} and the exact computation of $p$-resistance by naively optimizing Eq.~\eqref{eq:penergyandpresistance} by the gradient descent.
For ionosphere, iris, and wine, we made a graph for the best performing parameters in Table.~\ref{tab:res}.

In Table~\ref{tab:comptime}, we compare the computing time for the the best performing parameters in Table.~\ref{tab:res}.
We can see that ours are slower than spectral clustering methods. 
This slowness is because ours takes $O(mn^2)$ while spectral clustering methods using $p$-Laplacian are $O(n^3)$-based convergence methods.
Looking at the computational time for ECT, the time is similar to the spectral methods since ECT is also the $O(n^{3})$ method.
The computing time on ECT further motivates this future direction.

\subsection{Comparison of the Values of Approximated and Exact $p$-Resistance}
\label{sec:comparsionofapproximated}

This section discusses the comparison of approximated and exact $p$-resistance.
This experiment shows that the value is approximation is tighter than Prop.~\ref{prop:generalbound}.

This preliminary experiment aims to evaluate the quality of the approximation comparing to Thm.~\ref{thm:presistance}. 
To do so, we would like to compute the ratio of approximation to the exact $p$-resistance as $\|L^{+}\mathbf{e}_{i} - L^{+}\mathbf{e}_{j}\|_{G,q}^{q}/r_{G,p}^{1/(p-1)}(i,j)$.
Using Thm.~\ref{thm:presistance}, this ratio can be theoretically evaluated as
\begin{align}
\label{eq:boundtransform}
    1 \leq \frac{\|L^{+}\mathbf{e}_{i} - L^{+}\mathbf{e}_{j}\|_{G,q}^{q}}{r_{G,p}^{1/(p-1)}(i,j)} \leq \alpha_{G,p}^q,  \alpha_{G,p} := \vvvert W^{1/p}CC^{+}W^{-1/p} \vvvert_{p}.
\end{align}
This experiment also aims to evaluate this inequality.
We numerically computed the approximated $p$-resistance, exact $p$-resistance, and $\alpha_{G,p}$.
To compute $\alpha_{G,p}$, we use the same algorithm for the Table.~\ref{fig:cc}, which is discussed in Sec.~\ref{sec:additional}
To compute the approximated and exact $p$-resistance, we conducted with the following procedure.
To create a graph, we used $\mu=0.1$, in order to make $k$-nn graph. 
This means that we use $k = \lfloor 0.1n \rfloor$.
To make the comparison simple, we use an unweighted graph.  
This is because if we incorporate weights, it is not trivial how $\alpha_{G,p}$ behaves, and thus, the results might not be a consistent of a comparison among different $p$s.
The rest of the analysis was carried in the same procedure as Table~\ref{tab:comptime}.

The result is summarized in Fig.~\ref{fig:compratio}.
We can see that the all the ratios are in the bound of Thm.~\ref{thm:presistance}.
We also remark that using Prop.~\ref{prop:generalbound} we have the bound as
\begin{align}
    \alpha_{G,p}^{q} \leq m^{|q/2-1|},
\end{align}
all of the plots of $\|L^{+}\mathbf{e}_{i} - L^{+}\mathbf{e}_{j}\|_{G,q}^{q}/r_{G,p}^{1/(p-1)}(i,j)$ in Fig.~\ref{fig:compratio} is obviously far lower than this bound.
For this result we observe the looser bound for larger $p$, since $\vvvert CC^{+} \vvvert_{p} \leq \vvvert CC^{+} \vvvert_{\infty}$ assuming an unweighted graph.
By incorporating the weight, we might observe $\alpha_{G,p}$ differently, since we do not know which is larger  $\vvvert W^{1/p} CC^{+} W^{-1/p} \vvvert_{p} $ and $ \vvvert W^{1/\infty} CC^{+} W^{-1/\infty} \vvvert_{\infty}$.
Further, we have
\begin{align}
\label{eq:furthernorm}
    \alpha_{G,p}^{q} = \vvvert W^{1/p} CC^{+} W^{-1/p} \vvvert_{p}^{q} \leq \left(\frac{w_{\max}}{w_{\min}} \right)^{q/p=q-1} \vvvert CC^{+} \vvvert_{p}^{q}.
\end{align}
Seeing the current derived bound Eq.~\eqref{eq:furthernorm}, the bound may be looser if we involve the weights and $p$ is small and hence $q$ is large.
A tighter bound particularly for smaller $p$ is a possible future direction, but this might be a lower priority due to the low performance at the smaller $p$.
The reason is that, for small $p$, it is known that $r_{G,p}(i,j)$ converges to a meaningless function~\cite{alamgir2011phase,slepcev2019analysis} under certain graph building conditions. 
Also, possibly due to this, our method performs better for larger $p$.

We remark on the large $p$ observation. 
In the main text, we argue that ``these correspond to the existing theoretical insight; $p$-resistance with large $p$ becomes meaningful function while 2-resistance is not~\cite{alamgir2011phase,slepcev2019analysis}.'' 
At a first glance,~\cite{alamgir2011phase} seems to argue that ``smaller $p$ works''. 
However, this is due to the difference of the definition.
While we follow the $p$-resistance definition in~\cite{herbster2009predicting}, in~\cite{alamgir2011phase} their $p$-resistance $r_{G,p}^{A}$ is given as
\begin{align}
    r_{G,p}^{A}(i,j) &= \frac{1}{\min_{\bfx} \sum_{ij} a_{ij}^{1/(p-1)}|x_{i} - x_{j}|^{p/(p-1)} \ \subjectto \ x_{i} - x_{j} = 1}\\
    &=\frac{1}{\min_{\bfx} \sum_{ij} a_{ij}^{q-1}|x_{i} - x_{j}|^{q} \ \subjectto \ x_{i} - x_{j} = 1}.
\end{align}
Hence, in~\cite{alamgir2011phase} the parameter $p$ works in the opposite way; if we mean large, $p$~\cite{alamgir2011phase} means smaller $p$ and vice-versa. 
Despite this slight change of the definition, $p$-resistance in~\cite{herbster2009predicting} and~\cite{alamgir2011phase} shares almost the same properties.
More discussion can be seen in \S 6.1 in~\cite{alamgir2011phase}.

\subsection{Code for the Experiments}
We leave a remark on our code in the supplemental material.
Although we include an implementation of Alg.~\ref{algo:clusteringkmedoids}, at this stage we include a part of the whole codebase due to the copyright of the library reasons. 
In the final version, we plan to reduce the blockers to publish the code as much as possible. 
Moreover, we plan to publish our implementation at Github, an online codebase repository service.

\section{More Discussion on $\alpha_{G,p}$}
\label{sec:morediscussionalpha}

This section gives more observations on $\alpha_{G,p}$. 
In practice, we want to know how close to the exact value and how far from this upper bound the value of $ \vvvert W^{1/p}CC^{+}W^{-1/p} \vvvert_{p}$ is.
In the following, we argue that in the general case $\alpha_{G,p}$ is far less than the bound given in Prop.~\ref{prop:generalbound}.

Before we get into the detail, we give a brief overview of an interpretation of $\alpha_{G,p}$.
From the definition of $\bfz$, $\bfz$ is a mapping of $f_{q/p}(C\bfy)/\|\bfy\|_{G,q}^{q}$ from $\R^{m} \to \Image(C)$.
Comparing the equality condition Eq.~\eqref{eq:equalz}, we observe that if $f_{q/p}(C\bfy) \in \Image(C)$, we obtain the tightest bound since $\|\bfz\|_{G,p} = \|\bfy\|_{G,q}^{-1}$.
By looking at this, we observe that the $\alpha_{G,p}$ is the worst possible ``overflow'' of the mapping from $\Image(C)$ from $\R^{m}$, in a sense of the weighted $p$-norm.

\subsection{Condition Number Point of View}
\label{sec:boundingalpha}

To prove the bound of Prop.~\ref{prop:generalbound}, we only use $\vvvert MM^{+}\vvvert_{2} = 1$ and Lemma~\ref{lemma:higham}, which holds for \textit{any} matrix $M$.
Hence, we can say that this is the ``worst'' bound and we expect a far lower value of $\vvvert C C^{+} \vvvert_{p} $ for a general incidence matrix of graph.
To gain some qualitative observation on how close between the exact and approximation, we further decompose $\alpha_{G,p}$.
By using the submultiplicity and Lemma~\ref{lemma:higham2}, 
\begin{align}
\alpha_{G,p} = \vvvert W^{1/p}CC^{+}W^{-1/p} \vvvert_{p} 
\leq \vvvert W^{1/p}\vvvert_{p} \vvvert CC^{+} \vvvert_{p} \vvvert W^{-1/p} \vvvert_{p}
&\leq \vvvert C C^{+} \vvvert_{p}  w_{\max}^{1/p}/w_{\min}^{1/p}\\
\label{eq:condnumber}
&\leq \vvvert C \vvvert_{p}\vvvert C^{+} \vvvert_{p}  w_{\max}^{1/p}/w_{\min}^{1/p}
\end{align}
where $w_{\max}$$:=$$\max_{\ell} w_{\ell}$ and $w_{\min} := \min_{\ell} w_{\ell}$.
In numerical analysis, the term $\vvvert C \vvvert_{p} \vvvert C^{+} \vvvert_{p}$ is called as a \textit{condition number} of the matrix $C$~\cite{saad2003iterative}.
A condition number is related to the ``difficulty'' to numerically solve the linear equation $C\bfx = \bfy$.
The larger the condition number gets, the more difficult to solve the linear equation. 
The linear equation is difficult to solve if we can make one or more pairs of column or row of $C$ close to parallel by elementary operations.
However, by construction of incidence matrix, no pairs of column or row of the incidence matrix are close to parallel.
Thus, we expect that the condition number of $C$ will not be large, and hence we expect a smaller value of $\alpha_{G,p}$ than Prop.~\ref{prop:generalbound} in general.
For the specific graphs, we theoretically show this in the next section.

\subsection{Bound of $\alpha_{G,p}$ for Some Specific Graphs} 
\label{sec:specificgraphs}
In this section we give a constant bound of $\alpha_{G,p}$ for some specific graphs. 
By this we can independently bound $\alpha_{G,p}$ of $m$.

Now, for the specific cases, we have the following.
\begin{proposition}
\label{prop:specific}
If a graph is complete or cyclic, then $\vvvert CC^{+} \vvvert_{p} \leq 4$ and hence $\alpha_{G,p} \leq 4w_{\max}^{1/p}/w_{\min}^{1/p}$.
\end{proposition}
For these specific graphs, we can bound the $p$-resistance (Thm.~\ref{thm:presistance}) by a constant. 
We now divide the proof into the complete case and the cyclic case.
\subsubsection{Complete Case}
First, we obtain the pseudoinverse of $C$ of a complete graph.
\begin{lemma}
\label{lemma:pseudoinversecompletegraph}
For an incidence matrix $C'$ for a complete graph,
\begin{align}
    C^{'+} = \frac{1}{n}C^{'\top}
\end{align}
\end{lemma}
\begin{proof}

For a graph Laplacian $L$ of unweighted graph can be written as
\begin{align}
    L &= nI - \mathbf{1}^{\top}\mathbf{1},
      \end{align}
and thus
\begin{align}
     L_{ij} =  \left\{
        \begin{array}{ll}
            n-1&  \mathrm{if}\  i=j\\
            -1&   \mathrm{if}\ i \neq j
        \end{array}
     \right.
     .
\end{align}
Also, we know that
\begin{align}
    L = C^{'\top}C'.
\end{align}
Now we consider the the vector $\bfx_{ij} \in \R^{n}$ as
\begin{align}
    \bfx_{ij}^{\top} := \underbrace{(0,\ldots,0,\overbrace{1}^{i\mathrm{th\ element}},0,\ldots,0,\overbrace{-1}^{j\mathrm{th\ element}},0,\ldots,0)}_{\mathrm{size}\ n}.
\end{align}
Note that this $\bfx_{ij}$ is one row of the incidence matrix $C'$.
Now we get
\begin{align}
    (L\bfx_{ij})_{l} &= 
    \left\{
        \begin{array}{ll}
            (n-1)\times 1 + (-1)\times (-1) = n&  \mathrm{if}\  l=i\\
            1\times (-1) + (n-1)\times (-1) = -n&   \mathrm{if}\  l=j\\
            (-1) \times 1 + (-1)\times (-1) = 0  & \mathrm{otherwise} 
        \end{array}
     \right.\\
     &= n(L\bfx_{ij})_{l}.
\end{align}
Since $\bfx_{ij}$ is one column of the transpose of the incidence matrix $\mathbf{C}^{'\top}$,
\begin{align}
\label{eq:pseudoinversedef1}
    LC = C^{'\top}C'C^{'\top} = nC^{'\top} \iff \left(\frac{1}{n}C^{'\top}\right)C'\left(\frac{1}{n}C^{'\top}\right) = \frac{1}{n}C^{'\top}
\end{align}
Also,
\begin{align}
\label{eq:pseudoinversedef2}
    (LC)^{\top} = C^{'}C^{'\top}C^{'} = nC^{'} \iff C^{'}\left(\frac{1}{n}C^{'\top}\right)C^{'} = C^{'}
\end{align}
From Eq.~\eqref{eq:pseudoinversedef1} and Eq.~\eqref{eq:pseudoinversedef2}, the matrix $1/n C'$ satisfies the definition of $C^{+}$, which leads to the claim.
\end{proof}

Note that $\vvvert C C^{+}\vvvert_{1}=\vvvert C C^{+}\vvvert_{\infty}$ due to the symmetricity of $ C C^{+}$.
\begin{align}
    \vvvert CC^{+}\vvvert_{p} \leq \vvvert C C^{+}\vvvert_{\infty} \leq \vvvert C \vvvert_{\infty}\vvvert C^{+} \vvvert_{\infty} = 4\frac{n+1}{n}\leq 4.
\end{align}

\subsubsection{Cyclic Case}
In the cyclic graph, $m=n$, i.e., the number of vertices is equal to the number of edges. 
Thus, the incidence matrix $C$ is square.
However, in order to avoid confusion, in the following we use $m$ and $n$.
Now, we define the incidence matrix $C \in \R^{m \times n}$of the cyclic graph as
\begin{align}
    c_{i1} &= \left\{
    \begin{array}{cc}
        -1 &  \mathrm{when}\ i=1\\
        1 &   \mathrm{when}\ i=2 \\
        0 &   \mathrm{otherwise}
    \end{array}
    \right.\\
    c_{i2} &= \left\{
    \begin{array}{cc}
        -1 &  \mathrm{when}\ i=1\\
        1 &   \mathrm{when}\ i=n \\
        0 &   \mathrm{otherwise}
    \end{array}
    \right.\\
    c_{ij} &= \left\{
    \begin{array}{cc}
        -1 &  \mathrm{when}\ i=j-1\\
        1 &   \mathrm{when}\ i=j \\
        0 &   \mathrm{otherwise \ for\ } \ 
    \end{array}
    \right. \mathrm{for\ } j \geq 3.
\end{align}

Before we explore $C^{+}$, we introduce \textit{cyclic shift operator} of the vector.
Given the vector $\bfa$, the shift operator $(l)$ ``cyclic shifts'' the element, as
\begin{align}
    \bfa^{(l)} = (a_{n-l+1},a_{n-l+2},\ldots,a_{n},a_{1},\ldots,a_{n-l})^{\top}.
\end{align}
Thus, $\bfa^{(0)} = \bfa$.
Also, we define the reverse operator $\mathrm{rev}$ for a vector $\bfa$ as
\begin{align}
    \mathrm{rev}(\bfa) = (a_{n},a_{n-1},\ldots,a_{1}).
\end{align}
We also define the vector $\bfxi \in \R^{n}$ as
\begin{align}
    \bfxi = (1/2-1/2n,1/2 - 3/2n,\ldots,1/2 - (2i-1)/2n,\ldots,-1/2+1/n).
\end{align}

Now, we define a matrix $B$ as
\begin{align}
    B_{1\cdot} &= \bfxi^{(1)}\\
    B_{2\cdot} &= \mathrm{rev}(\bfxi^{(0)}) = \mathrm{rev}(\bfxi)\\
    B_{j\cdot} &= \bfxi^{(j-1)} \mathrm{for\ } j \geq 3,
\end{align}
where $B_{i\cdot}$ denotes $i$-th column of $B$.
We plot a heatmap of $C$ and $B$ for the illustrative purpose.

\begin{figure*}[t]
\begin{center}
\subfigure[$C$]{%
\includegraphics[width=.28\hsize,clip]{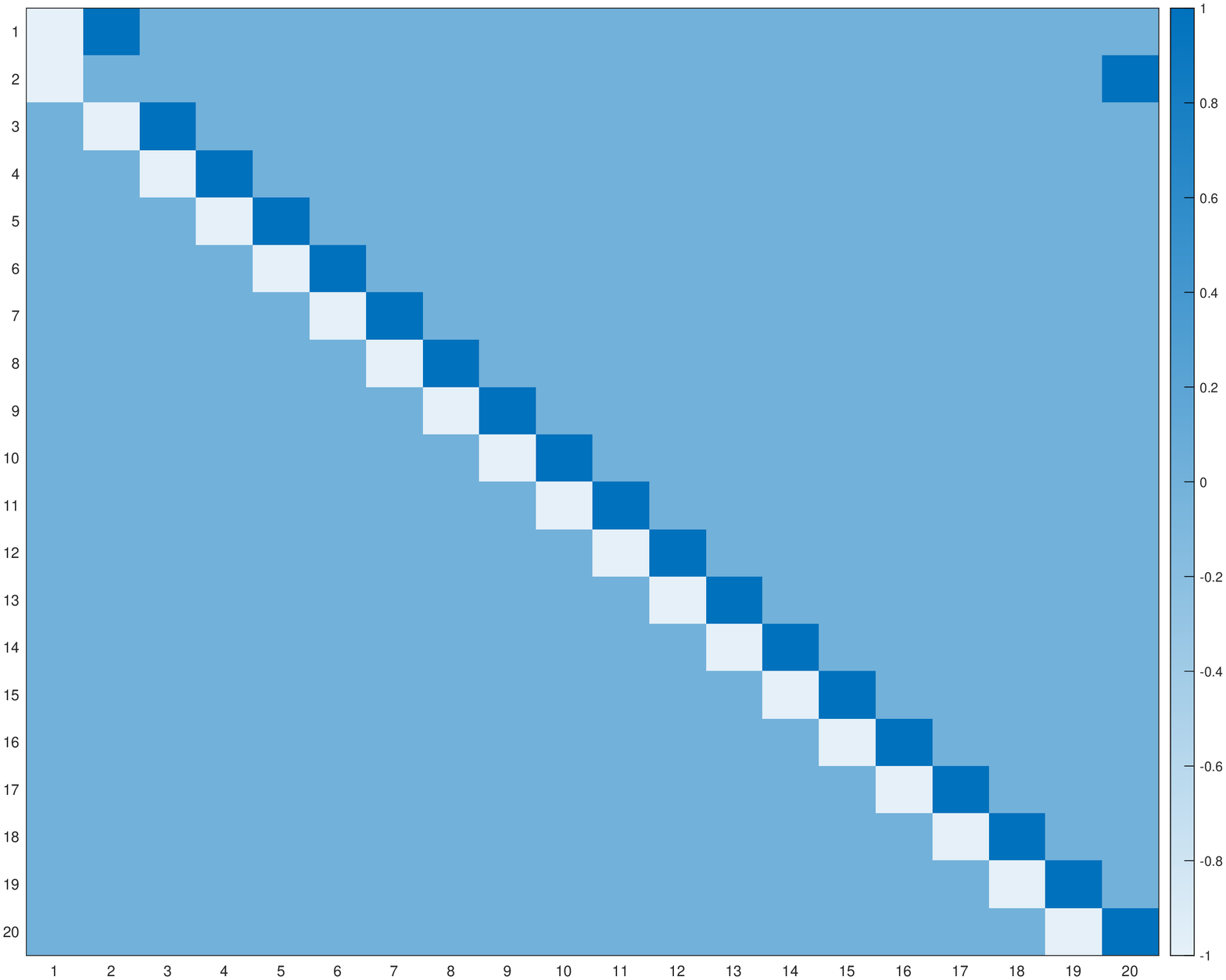}\label{fig:c}}
~\subfigure[$B=C^{+}$]{%
\includegraphics[width=.28\hsize,clip]{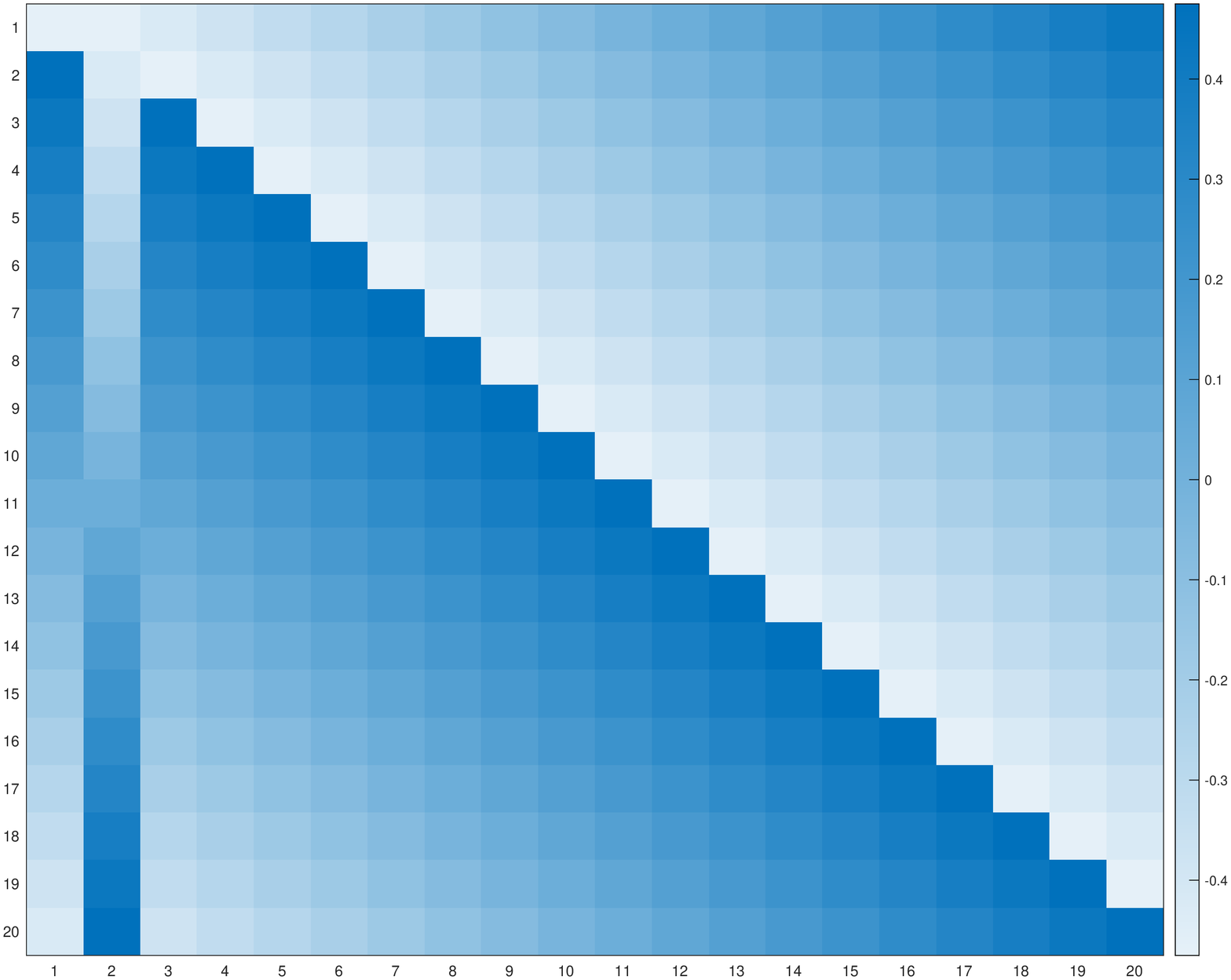}\label{fig:pc}}
~\subfigure[$CC^{+}$]{%
\includegraphics[width=.28\hsize,clip]{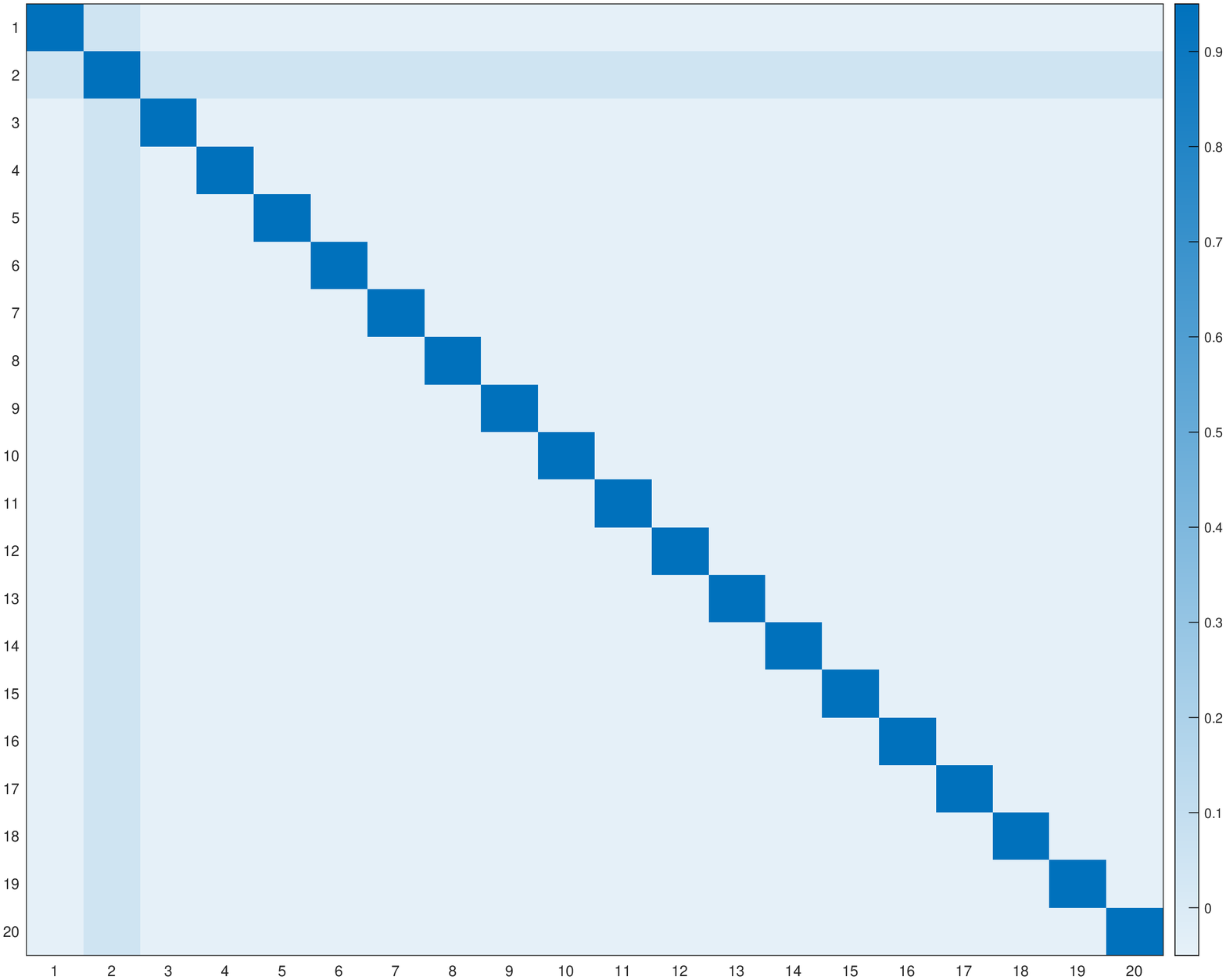}\label{fig:cpc}}
\caption{
Heatmap plot for the matrices $C$, $B=C^{+}$ and $CC^{+}$ of the cyclic graph for $n=20$.}
\label{fig:heatmapc}
\end{center}
\end{figure*}

Now we prove that $C^{+} = B$.
To claim that, it is enough to prove that $BC = I - \mathbf{1}^{\top}\mathbf{1}/n$~\cite{bapat2010graphs}.
From the construction,
\begin{align}
    (BC)_{ii} = \xi_{1} - \xi_{n} = 1/2-1/2n - (-1/2-1/2n) = 1-1/n. 
\end{align}
Also, when $i \neq j$,
\begin{align}
    (BC)_{ij} &= 
    \left\{
    \begin{array}{ll}
        - \xi_{i-1} - \xi_{n - i + 1} &   \mathrm{when}\  \mathrm{when}\ j = 1\\
        \xi^{(j)}_{i+1} - \xi^{(j)}_{i} &   \mathrm{when}\  \mathrm{when}\ 2\leq j < n,\\
        \xi_{n - i + 1} + \xi_{i + 1} &   \mathrm{when}\  \mathrm{when}\ j=n,\\
    \end{array}
    \right.\\
    & = -1/n.
\end{align}
Thus, we can say that $B = C^{+}$.
By doing a similar computation, we get
\begin{align}
    CC^{+} = 
    \left\{
    \begin{array}{cc}
        1 - 1/n &  \mathrm{when}\ i=j\\
        1/n &   \mathrm{when}\ i=2 \mathrm{\ or\ } j=2, i\neq j\\
        -1/n &   \mathrm{otherwise} \ 
    \end{array}
    \right.
\end{align}
We also plot a heatmap for $CC^{+}$ for the illustrative purpose in Fig.~\ref{fig:cpc}.
Thus, applying Lemma~\ref{lemma:higham2}, we get
\begin{align}
    \vvvert CC^{+} \vvvert_{p} \leq \vvvert CC^{+} \vvvert_{1} = \max_{i} \sum_{j=1}^{n} |(CC^{+})_{ij}| = 2 - 1/n \leq 4.
\end{align}

We leave a brief note for other concrete examples.
Several attempts are made to obtain the concrete form of $C^{+}$ for the specific graph~\cite{azimi2018moore,azimi2019moore}.
However, due to their abstract ways to characterize the graph such as distance or cut, we think that it is hard to immediately obtain a non-trivial bound from these results.
Also, $C^{+}$ for tree is studied~\cite{bapat1997moore}.
However, since we know the exact representation of $p$-resistance for tree in Thm.~\ref{thm:tree}, we do not have to discuss the tree case.

\subsection{Additional Experiments for Bound of Approximation}
\label{sec:additional}

This section discusses some additional experiments for bound of approximation in Prop.~\ref{prop:generalbound}.
For the real datasets we confirmed that the approximation over exact $\|L^{+}\mathbf{e}_{i} - L^{+}\mathbf{e}_{j}\|_{G,q}^{q}/r_{G,p}^{1/(p-1)}(i,j)$ is far below the bound and its worst case, in Fig.~\ref{fig:compratio}.
In this section, we further investigate the $\alpha_{G,p}$ of the real datasets. 
Moreover, we further investigate artificial dataset that the ratio is close to the worst dataset, and we discuss why it happens.

\subsubsection{Plot of $\vvvert CC^{+}\vvvert_{p}$.}

As we discussed in Sec.~\ref{sec:comparsionofapproximated}, since $\alpha_{G,p} = \vvvert W^{1/p}CC^{+}W^{-1/p}\vvvert_{p}$ involves $p$ in the matrix as well as norm, it is somewhat difficult to how $\alpha_{G,p}$ behaves by changing $p$. 
Thus, we focus on the unweighted graph; we numerically investigate the unweighted $\|CC^{+}\|_{p}$ that is a matrix norm evaluated in Eq.~\eqref{eq:furthernorm}.
We plot $\|CC^{+}\|_{p}$ for wine, ion, and iris with the $\mu$ when $k$-medoids performs the best in Fig.~\ref{fig:cc}. 
We use the matrix $p$-norm estimation algorithm proposed by~\cite{higham1992estimating}.
We plot $p\to 1$ and $p \to \infty$ as the exact value. 
We remark on the estimation of the matrix norm by~\cite{higham1992estimating}; let $\xi$ be the output by the estimation of the matrix norm of $CC^{+}$, then $\vvvert CC^{+}\vvvert_{p}/m^{|1/2-p|}\leq \xi \leq \vvvert CC^{+}\vvvert_{p}$. 
We note that using Lemma~\ref{lemma:higham2} and symmetricity of $CC^{+}$, we have
\begin{align}
\label{eq:CC+boundby1}
    \vvvert CC^{+}\vvvert_{p} \leq \vvvert CC^{+} \vvvert_{1} =  \vvvert CC^{+} \vvvert_{\infty},
\end{align}
whose exact value is computable.
Thus, although we need to use estimation algorithm for $\vvvert CC^{+}\vvvert_{p}$, we can exactly compute the bound of this norm.

Fig.~\ref{fig:cc} shows that $\|CC^{+}\|_{p}$ is far lower than the worst bound in Prop.~\ref{prop:generalbound}.
Moreover, Fig.~\ref{fig:cc} shows that the estimation algorithm proposed by~\citet{higham1992estimating} outputs is reliable results on this problem since this follows theory in terms of $\vvvert CC^{+} \vvvert_{p'} \leq \vvvert CC^{+} \vvvert_{p}$ if $2 < p' < p$ or $p<p'<2$, also Eq.~\eqref{eq:CC+boundby1}.
Also, even only looking at $\vvvert CC^{+}\vvvert_{1}=\vvvert CC^{+}\vvvert_{\infty}$, which is exactly computable the bound of $\vvvert CC^{+}\vvvert_{p}$, we have far lower value than the worst bound, especially in larger $p$ where our clustering algorithm works better.

\begin{figure}
	\begin{minipage}{.4\hsize}
		\centering
		\includegraphics[width=.8\hsize,clip]{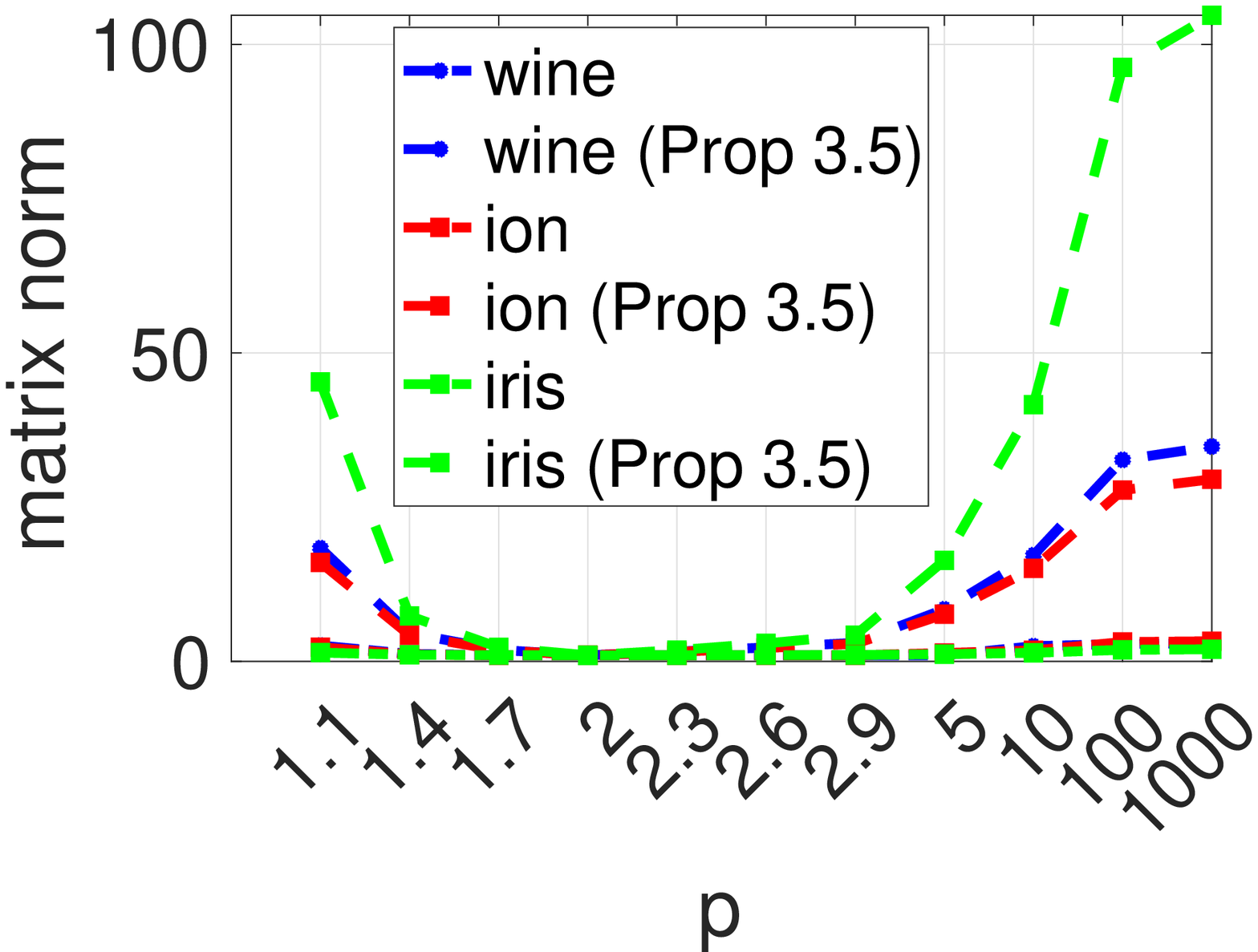}
		\captionof{figure}{Matrix norm $\|CC^{+}\|$ vs $p$.}
		\label{fig:cc}
	\end{minipage}
	\begin{minipage}{0.55\hsize}
    \centering
    \includegraphics[width=0.66\hsize,clip]{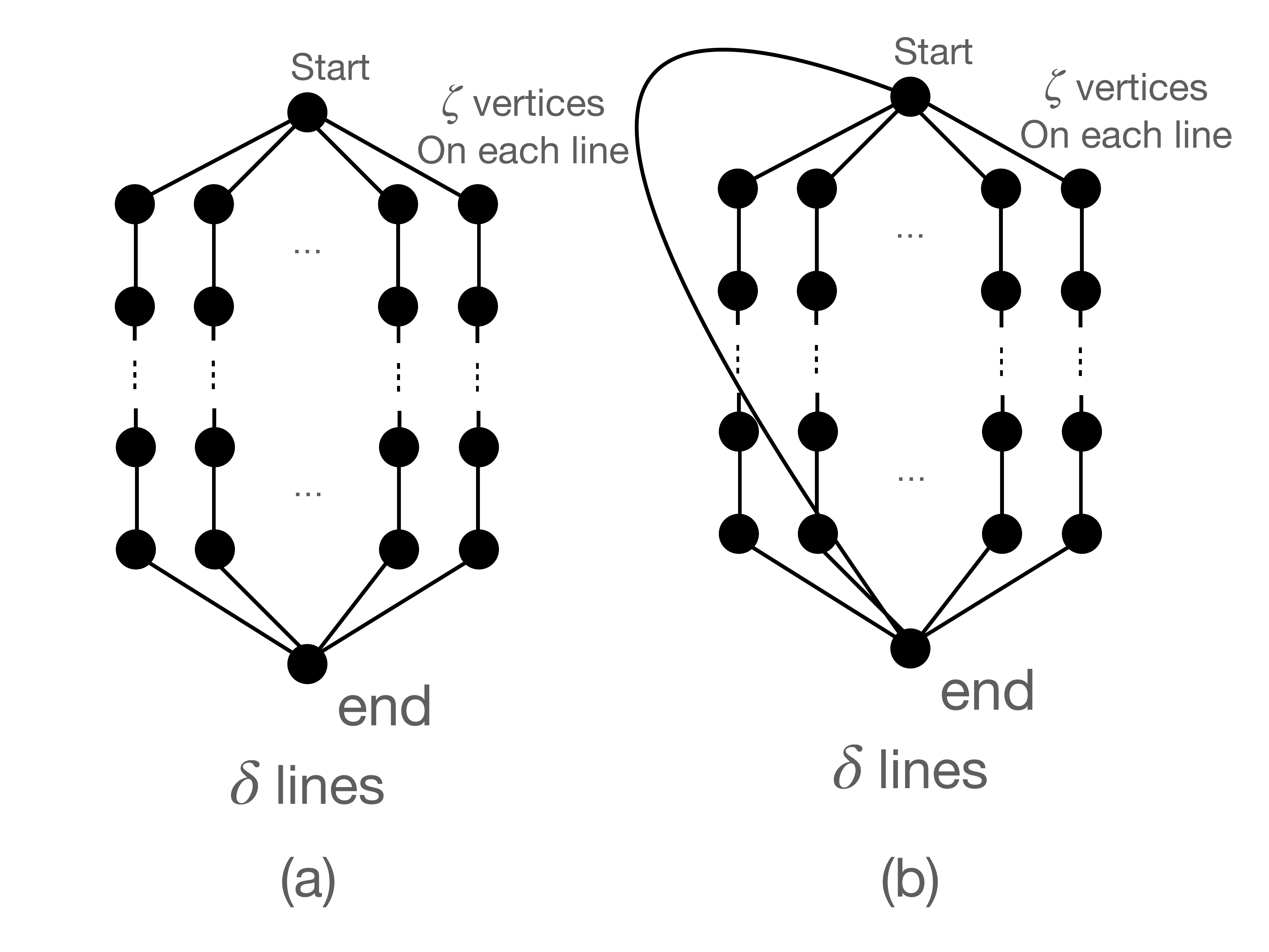}
    \caption{The example where the approximated value is far lower than the exact value. See~\ref{sec:additional} for details.}
    \label{fig:unbalk}
	\end{minipage}\hfill
\end{figure}

\begin{table}[]
\small
\centering
\begin{minipage}{.4\hsize}
\centering
\caption{
\small
The values of approximated $1$-Resistance for (a).
The exact 1-resistance for this graph is $1/\delta$.
}
\label{tab:approxa}
\begin{tabular}{cc|ccccc}
\toprule
                  &    & \multicolumn{5}{c}{$\zeta$}              \\
                  &    & 5   & 10  & 20   & 40    & 80     \\
\midrule
\multirow{5}{*}{$\delta$} & 5  & 0.2 & 0.1 & 0.05 & 0.025 & 0.0125 \\
                  & 10 & 0.2 & 0.1 & 0.05 & 0.025 & 0.0125 \\
                  & 20 & 0.2 & 0.1 & 0.05 & 0.025 & 0.0125 \\
                  & 40 & 0.2 & 0.1 & 0.05 & 0.025 & 0.0125 \\
                  & 80 & 0.2 & 0.1 & 0.05 & 0.025 & 0.0125
\end{tabular}    
\end{minipage}
\hspace{0.04\columnwidth}
\begin{minipage}{.4\hsize}
\centering
\caption{
\small
The values of approximated $p$-Resistance for (b).
The exact 1-resistance for this graph is $1/(\delta+1)$.
}
\label{tab:approxb}
\begin{tabular}{cc|ccccc}
\toprule
                  &    & \multicolumn{5}{c}{$\zeta$}
                  \\
                  &    & 5      & 10     & 20     & 40     & 80     \\
\midrule
\multirow{5}{*}{$\delta$} & 5  & 0.46 & 0.33 & 0.21 & 0.13 & 0.07 \\
                  & 10 & 0.61 & 0.48 & 0.33 & 0.21 & 0.12 \\
                  & 20 & 0.75 & 0.64 & 0.49 & 0.33 & 0.20\\
                  & 40 & 0.85 & 0.77 & 0.65 & 0.49 & 0.33 \\
                  & 80 & 0.92 & 0.87 & 0.79 & 0.66 & 0.50
\end{tabular}    
\end{minipage}
\end{table}

\begin{table}[]
\small 
\centering
\begin{minipage}{.4\hsize}
\centering
\caption{
\small
The values of $\vvvert CC^{+}\vvvert_{1}/m^{1/2}$ for the graph (a).
If this value is close to 1, we have a looser bound.
}
\label{tab:cca}
\begin{tabular}{cc|ccccc}
\toprule
                  &    & \multicolumn{5}{c}{$\zeta$}             \\
                  &    & 5    & 10   & 20   & 40   & 80   \\
\midrule
\multirow{5}{*}{$\delta$} & 5  & 0.51 & 0.33 & 0.23 & 0.16 & 0.11 \\
                  & 10 & 0.41 & 0.27 & 0.19 & 0.13 & 0.09 \\
                  & 20 & 0.31 & 0.2  & 0.14 & 0.1  & 0.07 \\
                  & 40 & 0.23 & 0.15 & 0.1  & 0.07 & 0.05 \\
                  & 80 & 0.16 & 0.11 & 0.07 & 0.05 & 0.04
\end{tabular}
\end{minipage}
\hspace{0.04\columnwidth}
\begin{minipage}{.4\hsize}
\centering
\caption{
\small
The values of $\vvvert CC^{+}\vvvert_{1}/m^{1/2}$ for the graph (b).
If this value is close to 1, we have a looser bound.
}
\label{tab:ccb}
\begin{tabular}{cc|ccccc}
\toprule
                  &    & \multicolumn{5}{c}{$\zeta$}                                                                                                      \\
                  &    & \multicolumn{1}{l}{5} & \multicolumn{1}{l}{10} & \multicolumn{1}{l}{20} & \multicolumn{1}{l}{40} & \multicolumn{1}{l}{80} \\
\midrule
\multirow{5}{*}{$\delta$} & 5  & 0.70                  & 0.55                   & 0.43                   & 0.33                   & 0.24                   \\
                  & 10 & 0.68                  & 0.59                   & 0.51                   & 0.41                   & 0.32                   \\
                  & 20 & 0.59                  & 0.56                   & 0.54                   & 0.49                   & 0.41                   \\
                  & 40 & 0.47                  & 0.48                   & 0.51                   & 0.52                   & 0.48                   \\
                  & 80 & 0.36                  & 0.38                   & 0.44                   & 0.49                   & 0.51                  
\end{tabular}
\end{minipage}
\end{table}

\subsubsection{Example where Approximation is Far Lower than the Exact Value}

Lastly, we discuss an example where the approximation is far lower than the exact value and how this happens.
We consider a graph depicted in Fig.~\ref{fig:unbalk}.
First, we see a graph in Fig.~\ref{fig:unbalk}~(a).
To build this graph, first consider the line graph, where the $\zeta$ vertices are in line.  
This graph is constructed with $\delta$ lines of diameter
$\zeta$ each lines start vertex is glued to each other lines ``start vertex'' and similar to the ``end vertices''.
For Fig.~\ref{fig:unbalk}~(b), we add one edge to the graph in Fig.~\ref{fig:unbalk} between the start vertex and end vertex.

We now compare with approximation and the exact value of $1$-resistance between the start vertex and end vertex.
As we discussed in Sec.~\ref{sec:illustrative}, we compute the exact $1$-resistance between $i$ and $j$ as the minimum cut's inverse.
Thus, for (a) $r_{G,1}(\textrm{start},\textrm{end})=1/\delta$ and for (b) we have  $r_{G,1}(\textrm{start},\textrm{end})=1/(\delta+1)$.
We then compute the approximated values and $\vvvert CC^{+} \vvvert_{1}$ for Fig.~\ref{fig:unbalk}.
We give a result in Tables~\ref{tab:approxa}--\ref{tab:ccb}.
From Tables~\ref{tab:approxa} and~\ref{tab:approxb}, we observe that we have a far less accurate approximation for graph (b) than that for graph (a).
In Tables~\ref{tab:cca} and~\ref{tab:ccb}, we observe that a far larger value of $\vvvert CC^{+}\vvvert_{1}$ for the graph (b) than that for the graph (a).
We also observe that comparing with $\vvvert CC^{+}\vvvert_{1}$ of the graph constructed from the real dataset in Fig.~\ref{fig:cc}, we see a far larger value of $\vvvert CC^{+}\vvvert_{1}$ for the graph (b).
The larger value of $\vvvert CC^{+}\vvvert_{1}$ might be the reason why the approximation of the $1$-resistance of graph (b) is far worse than the graph (a).

We now discuss why $\vvvert CC^{+}\vvvert_{1}$ for graph (b) is far larger than that for graph (a).
We now revisit the condition number argument in Sec.~\ref{sec:boundingalpha}.
The condition number is the stableness of the linear equation of the matrix.
The stable linear equation is even if we add small value $\mathbf{\epsilon}$ to the linear equation, i.e., $C\bfx=\bfy+\mathbf{\epsilon}$, the solution $\bfx$ is almost unchanged.
If we add perturbation on each edge in graph (a), the graph can absorb the perturbation since each line graph is almost independent.
However, on the graph (b), each line graph becomes dependent due to the additional edge. 
Moreover, the start and the end vertex are like ``pivots'' of the graph. 
The perturbations might be widely spread over the graph by connecting two pivots.
By this spread, graph (b) becomes unstable, while graph (a), where we do not connect the pivots is more stable.
In Tables~\ref{tab:conda} and~\ref{tab:condb} we see that graph (b) is far
more unstable than graph (a).

Finally, we argue that we do not observe this phenomenon in the real setting.
In the example of graph (b), the unstableness comes from the sparse connection over the graph and connection of the ``pivots'' over such a spares graph.
In a dense graph such as a complete graph, we saw far lower $\vvvert CC^{+}\vvvert_{1}$ as we observe in Sec~\ref{sec:specificgraphs}.
As we saw in the real dataset case, we can assume that there is a denser connection over the graph, even between the clusters.

\begin{table}[]
\small
\centering
\begin{minipage}{.4\hsize}
\centering
\caption{
\small
The values of condition number $\vvvert C \vvvert_{1}\vvvert C^{+} \vvvert_{1}$ for (a).
}
\label{tab:conda}
\begin{tabular}{cc|ccccc}
\toprule
                  &    & \multicolumn{5}{c}{$\zeta$}                   \\
                  &    & 5     & 10    & 20     & 40     & 80   \\
\midrule
\multirow{5}{*}{$\delta$} & 5  & 20.4  & 45.2  & 95.1   & 195.1  & 395  \\
                  & 10 & 40.4  & 90.2  & 190.1  & 390.1  & 790  \\
                  & 20 & 80.4  & 180.2 & 380.1  & 780.1  & 1580 \\
                  & 40 & 160.4 & 360.2 & 760.1  & 1560.1 & 3160 \\
                  & 80 & 320.4 & 720.2 & 1520.1 & 3120.1 & 6320
\end{tabular}  
\end{minipage}
\hspace{0.04\columnwidth}
\begin{minipage}{.4\hsize}
\centering
\caption{
\small
The values of condition number $\vvvert C \vvvert_{1}\vvvert C^{+} \vvvert_{1}$ for (b).
}
\label{tab:condb}
\begin{tabular}{cc|ccccc}
\toprule
                  &    & \multicolumn{5}{c}{$\zeta$}                      \\
                  &    & 5     & 10     & 20     & 40     & 80     \\
\midrule
\multirow{5}{*}{$\delta$} & 5  & 24.4  & 49.7   & 101.8  & 219.3  & 457.7  \\
                  & 10 & 51    & 136.7  & 346.2  & 812.8  & 1790   \\
                  & 20 & 120.2 & 362.3  & 1035.1 & 2702.6 & 6432.7 \\
                  & 40 & 266.2 & 871.6  & 2768.2 & 8119.5 & 21449  \\
                  & 80 & 563.8 & 1943.9 & 6670.3 & 21713  & 64438 
\end{tabular}   
\end{minipage}
\end{table}
\section{On Difficulties of The Exact Solution}
\label{sec:difficulties}
In this section, we briefly explain the difficulties to obtain the exact solution of the resistance.
Again, we consider the minimization problem Eq.~\eqref{eq:sslp}.
The Lagrangian multiplier method gives Eq.~\eqref{eq:lagconditionpsi} and Eq.~\eqref{eq:lagconditionpsi}.
From Eq.~\eqref{eq:lagconditionpsi}, the optimal solution $\bfx$ satisfies
\begin{align}
\label{eq:thelagrange}
    0 
    &=p\Delta_{p}\bfx - \lambda (\bfe_{i} - \bfe_{j}) 
\end{align}
To solve this problem, we want to consider $\Delta_{p}^{+}$, which is ``generalized inverse'' function $\Delta_{p}$, defined as
\begin{align}
    \Delta_p^{+}(\Delta_p(\Delta_p^{+}(\bfx))) = \Delta_p^{+} \bfx
\end{align}
Recall that we can write as
\begin{align}
    \Delta_{p} = C^{\top} W f_{p-1}(C\bfx).
\end{align}
For the convenience of notation, we write
\begin{align}
    f_{\bfw,p} = W f_{p}(C\bfx).
\end{align}
If there exists $\bfalpha \in \Kernel (C)$ s.t.
\begin{align}
\label{eq:imageCalpha}
    f^{-1}_{\bfw,p-1}(C^{+\top}\bfx - \bfalpha) \in \Image(C),
\end{align}
the $\Delta_{p}^{+}$ is given as
\begin{align}
    \Delta_{p}^{+} (\bfx) := C^{+} f_{\bfw,p-1}^{-1}(C^{+\top}\bfx - \bfalpha),
\end{align}
The reason is as follows. 
We get
\begin{align}
    \Delta_{p}(\Delta_{p}^{+}(\bfx)) 
    &= C^{\top}f_{\bfw,p-1}(CC^{+} f_{\bfw,p-1}^{-1}(C^{+\top}\bfx - \bfalpha)  )\\
    &= C^{\top}f_{\bfw,p-1}(f_{\bfw,p-1}^{-1}(C^{+\top}\bfx - \bfalpha))\\
    &= C^{\top}(C^{+\top}\bfx - \bfalpha))\\
    &= C^{\top}C^{+\top}\bfx.
\end{align}
The second line follows from the assumption that $f_{\bfw,p-1}^{-1}(C^{+\top}\bfx - \bfalpha) \in \Image(C)$.
Thus,
\begin{align}
    \Delta_{p}^{+}(\Delta_{p}(\Delta_{p}^{+}(\bfx))) 
    &= C^{+} f_{\bfw,p-1}^{-1}(C^{+\top}C^{\top}C^{+\top}\bfx - \bfalpha)\\
    &= C^{+} f_{\bfw,p-1}^{-1}(C^{+\top}\bfx - \bfalpha)\\
    &= \Delta_{p}^{+}\bfx.
\end{align}
From this property, if we substitute
\begin{align}
    \bfx = \Delta^{+}_p \left(\frac{\lambda}{p} (\bfe_{i} - \bfe_{j}) \right)
\end{align}
the Eq.~\eqref{eq:thelagrange} satisfied.
Therefore, the next question is what $\bfalpha$ is.
However, we do not know even if such $\bfalpha$ satisfying Eq.~\eqref{eq:imageCalpha} exists or not.

\end{document}